\documentclass[english]{article}
\usepackage{lmodern}

\usepackage[T1]{fontenc}
\usepackage[utf8]{inputenc}
\usepackage{geometry}
\usepackage{color}
\usepackage{babel}
\usepackage{units}
\usepackage{amsmath}
\usepackage{amsthm}
\usepackage{amssymb}
\usepackage{booktabs}

\usepackage{stmaryrd}
\usepackage{graphicx}
\usepackage{setspace}
\usepackage[authoryear]{natbib}
\usepackage{caption}
\usepackage{subcaption}
\usepackage{algorithm}
\usepackage{phaistos}
\usepackage{protosem}
\usepackage{hieroglf}

\usepackage[unicode=true,
 bookmarks=true,bookmarksnumbered=false,bookmarksopen=false,
 breaklinks=false,pdfborder={0 0 1},backref=page,colorlinks=true]
 {hyperref}
\hypersetup{pdftitle={PAC-Bayes},
 linkcolor=Blue, urlcolor=Blue, citecolor=Blue}

\makeatletter

\numberwithin{equation}{section}
\numberwithin{figure}{section}
\numberwithin{table}{section}
\theoremstyle{plain}
\newtheorem{thm}{\protect\theoremname}
\newtheorem{dfn}{Definition}
\newtheorem{prop}{Proposition}
\newtheorem{lem}{Lemma}
\newtheorem{cor}{Corollary}
\newtheorem{asm}{Assumption}
\theoremstyle{definition}

\newtheorem{rem}{Remark}

\usepackage{dsfont}
\usepackage[dvipsnames,svgnames,x11names]{xcolor}

\DeclareMathOperator*{\argmin}{argmin}

\makeatother

\providecommand{\examplename}{Example}
\providecommand{\theoremname}{Theorem}

\begin{document}
\title{Robust Bayesian Inference via Variational Approximations of Generalized Rho-Posteriors}
\author{El Mahdi Khribch, Pierre Alquier
\\
ESSEC Business School}
\maketitle
\begin{abstract}
We introduce the $\tilde{\rho}$-posterior, a modified version of the $\rho$-posterior of \citet{baraud2020robust}, obtained by replacing the supremum over competitor parameters with a softmax aggregation. This modification allows a PAC-Bayesian analysis of the $\tilde{\rho}$-posterior. This yields finite-sample oracle inequalities with explicit convergence rates that inherit the key robustness properties of the original framework, in particular graceful degradation under model misspecification and data contamination. Crucially, the PAC-Bayesian oracle inequalities extend to variational approximations of the $\tilde{\rho}$-posterior, providing theoretical guarantees for tractable inference. Numerical experiments on exponential families, regression, and real-world datasets confirm that the resulting variational procedures achieve robustness competitive with theoretical predictions at computational cost comparable to standard variational Bayes.
\end{abstract}
\tableofcontents

\section{Introduction and Motivation}
\subsection{Universal Estimation and the Robustness Problem}

A fundamental challenge in statistical inference is the construction of universal estimation procedures: given an observed sample $\mathcal{S} = (X_1, \ldots, X_n)$ from unknown distributions $P_\star^1, \ldots, P_\star^n$, one seeks methods that achieve minimax-optimal rates when the model is well-specified, yet degrade gracefully under misspecification \citep{bickel1976another,birge2006model}. Most classical procedures, including maximum likelihood estimation, method of moments, and standard Bayesian inference, fail to satisfy both requirements simultaneously.
The non-universality of maximum likelihood is well illustrated by an example due to \citet{birge2006model}. Consider $n$ independent observations from the mixture
\[P_\star = \left(1 - \frac{2}{n}\right)\,\mathcal{U}\left(\left[0, \frac{1}{10}\right]\right) + \frac{2}{n}\, \mathcal{U}\left(\left[\frac{1}{10}, \frac{9}{10}\right]\right),\]
modeled by $\{P_\theta = \mathcal{U}([0,\theta]) : \theta \in [0,1)\}$. Since $\mathcal{H}^2(P_\star, P_{1/10}) < 5/(4n)$ for $n \geq 4$, the parameter $\theta = 1/10$ provides an excellent approximation in squared Hellinger risk. Yet the MLE $\hat{\theta}_{\mathrm{MLE}} = X_{(n)}$ satisfies $\mathbb{E}[\mathcal{H}^2(P_\star, P_{\hat{\theta}_{\mathrm{MLE}}})] > 0.38$, failing even to be consistent. Although maximum likelihood achieves optimality under classical regularity conditions \citep{lecam1970assumptions,van2000asymptotic}, it can degrade catastrophically under misspecification.

Standard Bayesian inference suffers from analogous brittleness. The posterior
\[
\pi_n(\theta) \propto \pi(\theta) \prod_{i=1}^n p_\theta(X_i)
\]
concentrates around $\theta_0$ at the optimal rate in the well-specified setting $P_\star^i = P_{\theta_0}$, and credible sets are asymptotically valid confidence sets by the Bernstein--von Mises theorem \citep{van2000asymptotic,kleijn2012bernstein}. Under misspecification, however, a single corrupted observation can cause the posterior to concentrate arbitrarily far from the truth \citep{barron1999consistency,grunwald2017inconsistency,owhadi2015brittleness,baraud2017new}. The source of this fragility is the implicit commitment of Bayes' rule to minimizing $\mathrm{KL}(P_\star \| P_\theta)$, a divergence that is unbounded whenever the model assigns zero density to events of positive probability \citep{ronchetti2009robust}.

The connection between universality and robustness is made precise through the contamination model of \citet{huber1992robust}: a fraction $\varepsilon$ of observations is adversarially corrupted, so that $P_\star = (1-\varepsilon)P_{\theta_0} + \varepsilon Q$ for some $\theta_0 \in \Theta$ and arbitrary $Q$. Frequentist robustness requires continuity with respect to total variation \citep{ronchetti2009robust}, while Bayesian posterior consistency relies on Kullback--Leibler divergence \citep{kleijn2012bernstein}, a substantially stronger condition. As \citet{owhadi2015brittleness} observe, this mismatch creates fundamental obstacles for robust Bayesian inference. The Hellinger distance bridges the two perspectives: it satisfies $\mathcal{H}^2(P,Q) \leq \mathrm{TV}(P,Q) \leq \sqrt{2}\,\mathcal{H}(P,Q)$ \citep{tsybakov2008nonparametric} and, unlike KL divergence, remains bounded when the supports of $P$ and $Q$ differ.

\subsection{Robust Bayesian Inference via Generalized Posteriors}

Generalized Bayesian inference \citep{chernozhukov2003mcmc,bissiri2016general} provides a principled route to robustness by replacing the log-likelihood with an alternative loss $\ell_n(\theta)$:
\[
\pi_n^{\ell}(\theta) \propto \pi(\theta) \exp\!\left(-\beta n\, \ell_n(\theta)\right).
\]
Equivalently, the generalized posterior minimizes
\[
\pi_n^{\ell} = \argmin_{\rho\in\mathcal{P}(\Theta)} \left\{ \beta n \,\mathbb{E}_{\theta \sim \rho}[\ell_n(\theta)] + \mathrm{KL}(\rho \| \pi)\right\},
\]
an optimization-centric viewpoint that unifies Bayesian updating, variational inference, and generalized posteriors \citep{knoblauch2022optimization}, and reveals that standard Bayesian inference implicitly minimizes $\mathrm{KL}(P_\star \| P_\theta)$.

An important special case is that of \emph{tempered} or \emph{fractional} posteriors, obtained by raising the likelihood to a power $\beta \in (0,1)$. Non-asymptotic concentration inequalities for such posteriors and their variational approximations were established by \citet{bhattacharya2016bayesian,alquier2019concentration}; \citet{khribch2024convergence} subsequently sharpened these rates through an information-theoretic framework based on mutual information bounds.

More broadly, the robustness properties of a generalized posterior are inherited from the underlying loss function. Replacing KL divergence with Hellinger distance yields posteriors insensitive to tail behaviour \citep{hooker2014bayesian}; this principle extends to $\alpha$- and $\beta$-divergences \citep{ghosh2016robust}, $\gamma$-divergences \citep{nakagawa2020robust}, and was systematically formalized by \citet{jewson2018principles}, who showed that boundedness and total variation continuity of the divergence translate directly to posterior robustness. Related work includes kernel Stein discrepancy methods \citep{matsubara2022robust}, scoring-rule inference \citep{giummole2019objective,pacchiardi2021generalized}, power posteriors \citep{holmes2017assigning}, coarsened posteriors \citep{miller2019robust}, bagged posteriors \citep{huggins2019robust}, the safe Bayes framework \citep{grunwald2011safe,grunwald2017inconsistency,heide2020safe}, and MMD-Bayes \citep{cherief2020mmd,dellaporta2022robust}, which achieves minimax-optimal rates via kernel embeddings.

An alternative strategy constructs log-likelihood proxies through the median-of-means principle \citep{lecue2020robust,minsker2022generalized}. A fundamentally different approach, due to \citet{catoni2012challenging}, \citet{baraud2017new}, and \citet{baraud2020robust}, replaces the log-likelihood with bounded contrast functions, yielding the $\rho$-estimator (frequentist) and the $\rho$-posterior (Bayesian). The resulting competitor-based risk decomposition separates estimation error from model misspecification, achieving universal estimation in both frameworks.

\subsection{The Rho-Posterior Framework: Elegant Theory, Computational Challenge}

The theoretical foundations for robust testing-based estimation originate in the work of \citet{lecam1973convergence,le1975local} on asymptotic decision theory, subsequently developed by \citet{birge1983approximation,birge2006model} into $T$-estimation, a framework that constructs estimators through robust tests comparing the fitted model against reference distributions. While $T$-estimators achieve consistency and robustness, they require compactness of the parameter space.
To remove this restriction, \citet{baraud2018rho} and \citet{baraud2017new} introduced $\rho$-estimation. The method aggregates robust pairwise tests between models into an empirical supremum risk, and defines the $\rho$-estimator as the parameter minimizing the worst-case discrepancy against all competitors. Unlike $T$-estimation, $\rho$-estimation extends to noncompact models, including linear regression under various error distributions, and recovers the MLE asymptotically when the model is well-specified and regular.
\citet{baraud2020robust} introduced the $\rho$-posterior, a Bayesian counterpart that replaces the likelihood with the same robust testing-based criterion. The resulting posterior satisfies remarkable theoretical properties: explicit contamination rates degrading gracefully with the fraction of corrupted observations, minimax-optimal convergence, and versions of the Bernstein--von Mises theorem valid under outlier contamination. These guarantees make the $\rho$-posterior arguably the most theoretically complete robust Bayesian framework.
The $\rho$-posterior has, however, remained a theoretical construct. The computational obstruction is fundamental: its definition requires, for each parameter value, a worst-case comparison over the entire parameter space, a structure that admits neither MCMC sampling nor standard variational inference. The present work addresses this gap.

\subsection*{Our Contributions}

Our contributions are as follows. We define a modified version of the $\rho$-posteriors of \citet{baraud2020robust}, which we call \emph{generalized $\rho$-posteriors} (or $\tilde{\rho}$-posteriors). While the original $\rho$-posterior is constructed from a supremum contrast that resists both PAC-Bayesian analysis and tractable computation, the $\tilde{\rho}$-posterior replaces this supremum with a softmax aggregation over competitor parameters, yielding a Gibbs posterior amenable to the PAC-Bayesian machinery of \citet{mcallester1999pac,catoni2007pac,alquier2024user}.

We show that, unlike the $\rho$-posteriors of \citet{baraud2020robust}, the $\tilde{\rho}$-posteriors can be analyzed via PAC-Bayesian bounds. In Theorem~\ref{thm:pac-bayes-independent}, we establish finite-sample bounds controlling the expected Hellinger risk of the $\tilde{\rho}$-posterior by the oracle approximation error plus a complexity penalty, under both i.i.d.\ and independent non-identically distributed observations. Through softmax aggregation, Theorem~\ref{thm:oracle} yields oracle inequalities with dimension-dependent rates matching the minimax-optimal rates obtained by \citet{baraud2020robust} for the original $\rho$-posterior. As an application, we derive corresponding results for fixed-design regression (Theorem~\ref{thm:pac-bayes-regression}).

A key advantage of the PAC-Bayesian perspective is that the resulting oracle inequalities extend naturally to variational approximations of the $\tilde{\rho}$-posterior. Theorem~\ref{thm:variational-oracle} establishes that variational approximations inherit the robustness properties of the exact $\tilde{\rho}$-posterior with explicit control on the approximation quality. For exponential families with mean-field Gaussian variational families, the resulting saddle-point problem admits favorable nonconvex--strongly concave geometry, ensuring convergence of first-order methods uniformly in the temperature parameter (Theorem~\ref{thm:nc-sc-main}).

Finally, numerical experiments on exponential families under contamination, regression with correlated designs, and real-world datasets confirm that the variational $\tilde{\rho}$-posteriors achieve robustness competitive with theoretical predictions, at computational cost comparable to standard variational Bayes.

\subsection*{Organization}

Section~\ref{sec:main} presents our main results. Section~\ref{sec:numerical} reports numerical experiments. All proofs are deferred to Appendix~\ref{sec:proofs}.
\subsection*{Setup and Notation}

Let $\mathcal{S}=(X_1,\ldots,X_n)$ denote $n$ independent observations taking values in a measurable space $(\mathcal{X},\mathcal{A})$ equipped with a $\sigma$-finite reference measure $\mu$, and let $\{P_\theta:\theta\in\Theta\}$ be a parametric model. We impose two standing assumptions.

\begin{asm}[Dominated model]\label{asm:dominated}
For each $\theta\in\Theta$ and $i\in[n]$, the model distribution $P_\theta^i$ for the $i$-th observation is dominated by the $\sigma$-finite measure $\mu$, and admits a density $p_\theta^i=dP_\theta^i/d\mu$. The product model distribution is $P_\theta^{(n)} := \bigotimes_{i=1}^n P_\theta^i$.
\end{asm}

\begin{asm}[Data-generating distribution]\label{asm:true}
The observed sample $\mathcal{S}=(X_1,\ldots,X_n)$ consists of $n$ independent observations, where $X_i$ is drawn from an unknown distribution $P_\star^i$ with density $p_\star^i$ with respect to $\mu$ for each $i\in[n]$. The product distribution is $P_\star^{(n)} := \bigotimes_{i=1}^n P_\star^i$. We write $\mathbb{E}_{\mathcal{S}}$ and $\mathbb{P}_{\mathcal{S}}$ for expectations and probabilities under $P_\star^{(n)}$.
\end{asm}

\begin{rem}[i.i.d. observations as a special case]\label{rem:iid}
When $p_\star^1 = \cdots = p_\star^n =: p_\star$ and $p_\theta^1 = \cdots = p_\theta^n =: p_\theta$ for all $\theta\in\Theta$, Assumptions~\ref{asm:dominated}--\ref{asm:true} reduce to the standard i.i.d.\ framework with $P_\star = p_\star \cdot \mu$ and model $\{P_\theta = p_\theta \cdot \mu : \theta\in\Theta\}$. All subsequent results specialize accordingly.
\end{rem}

\paragraph{Sample Hellinger distance.}
The natural metric for independent observations is the \emph{coordinate-averaged squared Hellinger distance}:
\begin{equation}\label{eq:sample-hellinger}
\mathcal{H}_n^2(P_\star^{(n)}, P_\theta^{(n)}) := \frac{1}{n}\sum_{i=1}^n \mathcal{H}^2(P_\star^i, P_\theta^i) = \frac{1}{n}\sum_{i=1}^n \mathcal{H}^2(p_\star^i, p_\theta^i),
\end{equation}
where $\mathcal{H}^2(P,Q) = \frac{1}{2}\int (\sqrt{dP} - \sqrt{dQ})^2$ denotes the squared Hellinger distance. In the i.i.d. case (Remark~\ref{rem:iid}), this reduces to $\mathcal{H}_n^2(P_\star^{(n)}, P_\theta^{(n)}) = \mathcal{H}^2(p_\star, p_\theta) =: \mathcal{H}^2(P_\star, P_\theta)$.

\paragraph{The contrast function $\psi$.}
Following \citet{baraud2017new}, the unbounded log-likelihood is replaced by a bounded contrast $\psi:\mathbb{R}_+\cup\{+\infty\}\to[-1,1]$,
\begin{equation}\label{eq:def-psi}
\displaystyle \psi(x) = \begin{cases}
  \displaystyle
  \frac{\sqrt{x}-1}{\sqrt{x}+1} & \text{for } 0\leq x < +\infty,\\
1 & \text{for } x = +\infty,
\end{cases}
\end{equation}
with conventions $0/0=1$ and $a/0=+\infty$ for $a>0$. The function $\psi$ is strictly increasing, antisymmetric ($\psi(1/x) = -\psi(x)$), bounded in $[-1,1]$, and Lipschitz continuous with constant $2$ as a function of $\sqrt{x}$. Setting $\phi = 4\psi$, one has $|\phi(x)-\log x| \leq 0.055|x-1|^3$ for $x\in[1/2,2]$, so that $\psi$ approximates the log-likelihood locally while remaining globally bounded.

\paragraph{Pairwise and supremum contrasts.}
For $\theta,\theta'\in\Theta$ and observation $X_i$, the \emph{coordinate-wise contrast} is defined by
\begin{equation}\label{eq:def-contrast}
\ell_\psi(x_i;\theta,\theta') := \psi\left(\frac{p_{\theta'}^i(x_i)}{p_\theta^i(x_i)}\right),
\end{equation}
with population and empirical versions given by coordinate-averaging:
\begin{align}
R_\psi(\theta,\theta') 
&:= \frac{1}{n}\sum_{i=1}^n\mathbb{E}_{X_i\sim P_\star^i}[\ell_\psi(X_i;\theta,\theta')]
= \frac{1}{n}\sum_{i=1}^n\int \psi\left(\frac{p_{\theta'}^i}{p_\theta^i}\right)\,dP_\star^i,\label{eq:def-Rpsi}\\
\hat{R}_\psi(\theta,\theta') 
&:= \frac{1}{n}\sum_{i=1}^n \ell_\psi(X_i;\theta,\theta').\label{eq:def-Rpsi-hat}
\end{align}
We also define the variance 
\begin{equation}\label{eq:def-variance}
V_\psi(\theta,\theta') := \mathbb{V}_{\mathcal{S}}[\hat{R}_\psi(\theta,\theta')]
= \frac{1}{n^2}\sum_{i=1}^n\mathbb{V}_{X_i\sim P_\star^i}[\ell_\psi(X_i;\theta,\theta')],
\end{equation}
by independence. The \emph{supremum contrasts} are
\begin{equation}\label{eq:def-supremum}
R_\psi^*(\theta) := \sup_{\theta'\in\Theta}R_\psi(\theta,\theta'),
\quad
\hat{R}_\psi^*(\theta) := \sup_{\theta'\in\Theta}\hat{R}_\psi(\theta,\theta'),
\end{equation}
and the $\rho$-estimator is $\hat{\theta}_{n,\psi} \in \argmin_{\theta\in\Theta}\hat{R}_\psi^*(\theta)$.
In the i.i.d.\ case (Remark~\ref{rem:iid}), these reduce to $R_\psi(\theta,\theta') = \mathbb{E}_{X\sim P_\star}[\ell_\psi(X;\theta,\theta')]$ and $V_\psi(\theta,\theta') = n^{-1}\mathbb{V}_{X\sim P_\star}[\ell_\psi(X;\theta,\theta')]$.

The following lemma, due to \citet{baraud2018rho}, relates the $\psi$-contrast to the Hellinger geometry of the product laws $P_\theta^{(n)}:=\bigotimes_{i=1}^n P_\theta^i$.

\begin{lem}[Proposition 3 in~\citet{baraud2018rho}]\label{lem:hellinger-comparison1}
For $\psi(x)=(\sqrt{x}-1)/(\sqrt{x}+1)$, there exist universal constants $a_0=4$, $a_1=3/8$, $a_2^2=3\sqrt{2}$ such that for all $(\theta,\theta')\in\Theta^2$,
\begin{align}
a_1\mathcal{H}_n^2(P_\star^{(n)},P_{\theta}^{(n)}) - a_0\mathcal{H}_n^2(P_\star^{(n)},P_{\theta'}^{(n)})
&\leq R_{\psi}(\theta,\theta')
\leq a_0\mathcal{H}_n^2(P_\star^{(n)},P_{\theta}^{(n)}) - a_1\mathcal{H}_n^2(P_\star^{(n)},P_{\theta'}^{(n)}), \label{eq:risk-bounds}\\
V_{\psi}(\theta,\theta')
&\leq \frac{a_2^2}{n}\left(\mathcal{H}_n^2(P_\star^{(n)},P_{\theta}^{(n)}) + \mathcal{H}_n^2(P_\star^{(n)},P_{\theta'}^{(n)})\right). \label{eq:variance-bound1}
\end{align}
In the i.i.d. case (Remark~\ref{rem:iid}), these bounds hold with $\mathcal{H}_n^2$ replaced by $\mathcal{H}^2$.
\end{lem}

We refer to Table~\ref{tab:notation} for reference.
\begin{table}[!htt]
  \centering
  \caption{Notation and definitions}
  \small
  \setlength{\tabcolsep}{2pt}
  \renewcommand{\arraystretch}{1.}
  \label{tab:notation}
  \begin{tabular}{@{}ll@{}}
  \toprule
  \textbf{Symbol} & \textbf{Definition} \\
  \midrule
  \multicolumn{2}{@{}l}{\textit{Data and model}} \\
  $\mathcal{X}$, $\mathcal{A}$ & Sample space and $\sigma$-algebra \\
  $\mu$ & Reference measure (e.g., Lebesgue or counting) \\
  $n$ & Sample size \\
  $\mathcal{S}=(X_1,\ldots,X_n)$ & Observed independent sample \\
  $P_\star^i$, $p_\star^i$ & True distribution and density of $X_i$ w.r.t.\ $\mu$ \\
  $P_\star^{(n)}$ & Product distribution: $\bigotimes_{i=1}^n P_\star^i$ \\
  $\Theta$ & Parameter space \\
  $P_\theta^i$, $p_\theta^i$ & Model distribution and density of $X_i$ for $\theta\in\Theta$ \\
  $P_\theta^{(n)}$ & Product model distribution: $\bigotimes_{i=1}^n P_\theta^i$ \\
  \midrule
  \multicolumn{2}{@{}l}{\textit{i.i.d.\ special case (Remark~\ref{rem:iid})}} \\
  $P^\star$, $p^\star$ & Common distribution when $p_\star^1=\cdots=p_\star^n$ \\
  $P_\theta$, $p_\theta$ & Common model when $p_\theta^1=\cdots=p_\theta^n$ for all $\theta$ \\
  \midrule
  \multicolumn{2}{@{}l}{\textit{Contrast and loss functions}} \\
  $\psi:\mathbb{R}_+\to[-1,1]$ & Bounded contrast function: $\psi(x)=(\sqrt{x}-1)/(\sqrt{x}+1)$ \\
  $\ell_\psi(x_i;\theta,\theta')$ & Coordinate-wise contrast: $\psi(p_{\theta'}^i(x_i)/p_\theta^i(x_i))$ \\
  $R_\psi(\theta,\theta')$ & Population contrast: $\frac{1}{n}\sum_{i=1}^n\mathbb{E}_{X_i\sim P_\star^i}[\ell_\psi(X_i;\theta,\theta')]$ \\
  $\hat{R}_\psi(\theta,\theta')$ & Empirical contrast: $\frac{1}{n}\sum_{i=1}^n \ell_\psi(X_i;\theta,\theta')$ \\
  $V_\psi(\theta,\theta')$ & Variance: $\frac{1}{n^2}\sum_{i=1}^n\mathbb{V}_{X_i\sim P_\star^i}[\ell_\psi(X_i;\theta,\theta')]$ \\
  $R_\psi^*(\theta)$ & Supremum population contrast: $\sup_{\theta'\in\Theta}R_\psi(\theta,\theta')$ \\
  $\hat{R}_\psi^*(\theta)$ & Supremum empirical contrast: $\sup_{\theta'\in\Theta}\hat{R}_\psi(\theta,\theta')$ \\
  \midrule
  \multicolumn{2}{@{}l}{\textit{Hellinger distance}} \\
  $\mathcal{H}^2(P,Q)$ & Squared Hellinger distance: $\frac{1}{2}\int(\sqrt{dP}-\sqrt{dQ})^2$ \\
  $\mathcal{H}_n^2(P_\star^{(n)},P_\theta^{(n)})$ & Sample Hellinger: $\frac{1}{n}\sum_{i=1}^n \mathcal{H}^2(p_\star^i,p_\theta^i)$ \\
  $\mathcal{H}^2(P^\star,P_\theta)$ & Standard Hellinger (i.i.d.\ case): $\mathcal{H}^2(p^\star,p_\theta)$ \\
  $\rho(P,Q)$ & Hellinger affinity: $1-\mathcal{H}^2(P,Q)$ \\
  \midrule
  \multicolumn{2}{@{}l}{\textit{Priors and posteriors}} \\
  $\pi$, $\pi'$ & Prior distributions on target $\theta$ and competitor $\theta'$ \\
  $\rho$, $\rho'$ & Posterior distributions on target $\theta$ and competitor $\theta'$ \\
  $\mathrm{KL}(\rho\Vert\pi)$ & Kullback-Leibler divergence: $\int\log(d\rho/d\pi)\,d\rho$ \\
  \midrule
  \multicolumn{2}{@{}l}{\textit{PAC-Bayes and temperature}} \\
  $\lambda>0$ & Temperature parameter \\
  $g(x)$ & Bernstein function: $(e^x-1-x)/x^2$ for $x\neq 0$, $1/2$ for $x=0$ \\
  $\beta_{n,\lambda}$ & Scaled temperature: $g(2\lambda/n)\cdot\lambda/n$ \\
  $B_{n,\lambda}$ & Bernstein coefficient: $g(2\lambda/n)\cdot\lambda^2/n$ \\
  $\Lambda_{\lambda}(\theta;\pi')$ & Softmax competitor: $\frac{1}{\lambda}\log\int e^{\lambda\hat{R}_\psi(\theta,\theta')}\pi'(d\theta')$ \\
  $\hat{\rho}_\lambda$ & Target Gibbs $\rho$-posterior at temperature $\lambda$ \\
  $\rho'_\lambda$ & Competitor Gibbs posterior \\
  \midrule
  \multicolumn{2}{@{}l}{\textit{Hellinger comparison constants}} \\
  $a_0$ & Upper bound constant: $a_0=4$ \\
  $a_1$ & Lower bound constant: $a_1=3/8$ \\
  $a_2^2$ & Variance constant: $a_2^2=3\sqrt{2}$ \\
  $\zeta_{n,\lambda}$ & Competitor temperature: $\lambda(a_1+\beta_{n,\lambda}a_2^2)/2$ \\
  $Z_{n,\lambda}(\pi')$ & Normalizing constant: $\int e^{-\zeta_{n,\lambda}\mathcal{H}_n^2(P_\star^{(n)},P_\vartheta^{(n)})}\pi'(d\vartheta)$ \\
  \midrule
  \multicolumn{2}{@{}l}{\textit{Estimators}} \\
  $\hat{\theta}_{n,\psi}$ & $\rho$-estimator: minimizer of $\hat{R}_\psi^*(\theta)$ \\
  $\hat{\theta}_{\mathrm{MLE}}$ & Maximum likelihood estimator \\
  $\hat{\theta}_B$ & Bayes estimator (posterior mean) \\
  $\delta\in(0,1)$ & Confidence level \\
  \bottomrule
  \end{tabular}
\end{table}

\section{Main Results}\label{sec:main}

Our main result is a finite-sample PAC-Bayes bound for a Gibbs posterior constructed from the $\psi$-contrast.

\medskip

\begin{thm}[PAC-Bayes bound for independent observations]\label{thm:pac-bayes-independent}
Fix $\delta\in(0,1)$, $\lambda>0$, and priors $\pi,\pi' \in \mathcal{P}(\Theta)$. Define the softmax competitor functional
\begin{equation}\label{eq:softmax-def}
\Lambda_{\lambda}(\theta;\pi') := \frac{1}{\lambda} \log\left(\int_{\Theta}e^{\lambda\,\hat{R}_\psi(\theta,\theta')}\pi'(d\theta')\right),
\end{equation}
and let $\hat\rho_\lambda$ be any minimizer of
\begin{equation}\label{eq:target-gibbs}
\rho \mapsto \mathbb{E}_{\theta\sim\rho}\left[\Lambda_{\lambda}(\theta;\pi')\right] + \frac{1}{\lambda}\mathrm{KL}(\rho\Vert\pi).
\end{equation}
Then with probability at least $1-\delta$ over $\mathcal{S} = (X_1,\ldots,X_n)$,
\begin{align}
&\big(a_1-\beta_{n,\lambda}a_2^2\big)\,\mathbb{E}_{\theta\sim\hat\rho_\lambda}\left[\mathcal{H}_n^2(P_\star^{(n)},P_\theta^{(n)})\right]\nonumber\\
&\le \inf_{\rho\in\mathcal{P}(\Theta)}\left\{\mathbb{E}_{\theta\sim\rho}[\Lambda_{\lambda}(\theta;\pi')]
+\frac{1}{\lambda}\mathrm{KL}(\rho\Vert\pi)\right\}\nonumber\\
&\quad + \inf_{\rho'\in\mathcal{P}(\Theta)}\left\{\big(a_0+\beta_{n,\lambda}a_2^2\big)\mathbb{E}_{\theta'\sim\rho'}\left[\mathcal{H}_n^2(P_\star^{(n)},P_{\theta'}^{(n)})\right] +\frac{2}{\lambda}\mathrm{KL}(\rho'\Vert\pi') \right\}\nonumber\\
&\quad +\frac{\log(1/\delta)}{\lambda},\label{eq:main-pac-bayes}
\end{align}
where $a_0=4$, $a_1=3/8$, $a_2^2=3\sqrt{2}$, and
\begin{equation}\label{eq:beta-def}
\beta_{n,\lambda} = g\left(\frac{2\lambda}{n}\right)\frac{\lambda}{n}, \quad g(x) = \begin{cases}
(e^x-1-x)/x^2 & x\neq 0\\
1/2 & x=0.
\end{cases}
\end{equation}
\end{thm}

\medskip

An explicit choice of temperature yields a simplified bound with interpretable constants.

\begin{cor}\label{cor:explicit-lambda-independent}
Under the setting of Theorem~\ref{thm:pac-bayes-independent}, take $\lambda=n/8$.
Then with probability at least $1-\delta$,
\begin{align}
\frac{1}{12}\,
\mathbb{E}_{\theta\sim\hat\rho_\lambda}\!\left[\mathcal{H}_n^2(P_\star^{(n)},P_\theta^{(n)})\right]
&\le
\inf_{\rho\in\mathcal{P}(\Theta)}
\left\{
\mathbb{E}_{\theta\sim\rho}[\Lambda_{\lambda}(\theta;\pi')]
+\frac{8}{n}\mathrm{KL}(\rho\Vert\pi)
\right\}\nonumber\\
&\quad+
\inf_{\rho'\in\mathcal{P}(\Theta)}
\left\{
\frac{13}{3}\,
\mathbb{E}_{\theta'\sim\rho'}\!\left[\mathcal{H}_n^2(P_\star^{(n)},P_{\theta'}^{(n)})\right]
+\frac{16}{n}\mathrm{KL}(\rho'\Vert\pi')
\right\}
+\frac{8\log(1/\delta)}{n}.
\label{eq:explicit-rate-ind}
\end{align}
\end{cor}

\medskip

The remainder of this section specializes to the i.i.d.\ case (Remark~\ref{rem:iid}); all results extend to independent observations upon replacing $\mathcal{H}^2(P_\star,P_\theta)$ by $\mathcal{H}_n^2(P_\star^{(n)},P_\theta^{(n)})$.

\begin{cor}[i.i.d. case]\label{cor:iid-case}
Suppose $X_1, \ldots, X_n$ are i.i.d.\ with $X_i \sim P_\star := p_\star \cdot \mu$ for all $i$, and the model densities satisfy $p_\theta^i = p_\theta$ for all $i$ and $\theta \in \Theta$. Then Theorem~\ref{thm:pac-bayes-independent} holds with $\mathcal{H}_n^2(P_\star^{(n)}, P_\theta^{(n)})$ replaced by $\mathcal{H}^2(P_\star, P_\theta)$ throughout. Specifically, with probability at least $1-\delta$,
\begin{align}
&\big(a_1-\beta_{n,\lambda}a_2^2\big)\,\mathbb{E}_{\theta\sim\hat\rho_\lambda}\left[\mathcal{H}^2(P_\star,P_\theta)\right]\nonumber\\
&\le \inf_{\rho\in\mathcal{P}(\Theta)}\left\{\mathbb{E}_{\theta\sim\rho}[\Lambda_{\lambda}(\theta;\pi')]
+\frac{1}{\lambda}\mathrm{KL}(\rho\Vert\pi)\right\}\nonumber\\
&\quad + \inf_{\rho'\in\mathcal{P}(\Theta)}\left\{\big(a_0+\beta_{n,\lambda}a_2^2\big)\mathbb{E}_{\theta'\sim\rho'}\left[\mathcal{H}^2(P_\star,P_{\theta'})\right] +\frac{2}{\lambda}\mathrm{KL}(\rho'\Vert\pi') \right\}\nonumber\\
&\quad +\frac{\log(1/\delta)}{\lambda}.\label{eq:iid-bound}
\end{align}
\end{cor}

\begin{proof}
In the i.i.d. case, the sample Hellinger distance~\eqref{eq:sample-hellinger} simplifies to
\[
\mathcal{H}_n^2(P_\star^{(n)}, P_\theta^{(n)}) = \frac{1}{n}\sum_{i=1}^n \mathcal{H}^2(p_\star, p_\theta) =  \mathcal{H}^2(P_\star, P_\theta).
\]
Similarly, the population contrast becomes
\[
R_\psi(\theta,\theta') = \frac{1}{n}\sum_{i=1}^n \mathbb{E}_{X \sim P_\star}[\ell_\psi(X;\theta,\theta')] = \mathbb{E}_{X \sim P_\star}[\ell_\psi(X;\theta,\theta')],
\]
and the variance reduces to
\[
V_\psi(\theta,\theta') = \frac{1}{n^2}\sum_{i=1}^n \mathbb{V}_{X \sim P_\star}[\ell_\psi(X;\theta,\theta')] = \frac{1}{n}\mathbb{V}_{X \sim P_\star}[\ell_\psi(X;\theta,\theta')].
\]
Theorem~\ref{thm:pac-bayes-independent} applies directly with these simplifications.
\end{proof}

\medskip

Setting $\lambda=n/8$ yields explicit constants.

\begin{cor}[Explicit temperature, i.i.d.\ case]\label{cor:explicit-lambda}
Under the setting of Corollary~\ref{cor:iid-case}, take $\lambda = n/8$. Then with probability at least $1-\delta$,
\begin{align}
 \frac{1}{12}\,\mathbb{E}_{\theta\sim \hat{\rho}_\lambda}[\mathcal{H}^2(P^\star,P_\theta)] 
  &\leq \inf_{\rho\in\mathcal{P}(\Theta)}\left\{\mathbb{E}_{\theta\sim\rho}[\Lambda_{\lambda}(\theta;\pi')]+\frac{8\,\mathrm{KL}(\rho\Vert\pi)}{n}\right\}\nonumber\\
  &\quad +\inf_{\rho'\in\mathcal{P}(\Theta)}\left\{\frac{13}{3}\,\mathbb{E}_{\theta'\sim\rho'}\left[\mathcal{H}^2(P^\star,P_{\theta'})\right] +\frac{16\,\mathrm{KL}(\rho'\Vert\pi') }{n}\right\}\nonumber\\
  &\quad +\frac{8\log(1/\delta)}{n}.\label{eq:explicit-rate}
  \end{align}
\end{cor}

\medskip

The proof of Theorem~\ref{thm:pac-bayes-independent} is given in Section~\ref{sec:proofs}.

\begin{rem}[Choice of competitor prior]\label{rem:pi-equals-pi-prime}
Theorem~\ref{thm:pac-bayes-independent} and Corollaries~\ref{cor:explicit-lambda-independent}--\ref{cor:explicit-lambda} allow the target prior $\pi$ and the competitor prior $\pi'$ to differ. In practice, there is seldom reason to choose them differently: a single prior $\pi$ encoding the available information about $\Theta$ serves both roles equally well. Accordingly, all subsequent results are stated under the simplifying assumption $\pi' = \pi$.  The proofs are carried out for general $\pi' \neq \pi$ and specialize immediately.
\end{rem}

\medskip

Bounding $\Lambda_\lambda$ in terms of Hellinger distance yields a purely geometric oracle inequality.

\begin{thm}[Oracle inequality, i.i.d. case]\label{thm:oracle}
  Fix $\delta\in(0,1/2)$ and $\lambda = n/8$. With probability at least $1-2\delta$ over $\mathcal{S}$,
  \begin{equation}\label{eq:oracle}
  \frac{1}{12}\,\mathbb{E}_{\theta\sim\hat\rho_\lambda}[\mathcal{H}^2(P^\star,P_\theta)]
  \leq \inf_{\rho\in\mathcal{P}(\Theta)}\left\{\frac{26}{3}\,\mathbb{E}_{\theta\sim\rho}\left[\mathcal{H}^2(P^\star,P_\theta)\right] + \frac{32\,\mathrm{KL}(\rho\|\pi)}{n}\right\}
  +\frac{16\log(1/\delta)}{n}.
  \end{equation}
  \end{thm}
\begin{proof}
Combining Corollary~\ref{cor:explicit-lambda} with Lemma~\ref{lem:lambda-hellinger} via a union bound (see Section~\ref{sec:proofs}), with probability at least $1-2\delta$,
\begin{align*}
\frac{1}{12}\,\mathbb{E}_{\theta\sim\hat\rho_\lambda}[\mathcal{H}^2(P^\star,P_\theta)]
&\leq \inf_{\rho\in\mathcal{P}(\Theta)}\left\{\frac{13}{3}\,\mathbb{E}_{\theta\sim\rho}[\mathcal{H}^2(P^\star,P_\theta)] + \frac{8\,\mathrm{KL}(\rho\|\pi)}{n}\right\}\\
&\quad + \inf_{\rho'\in\mathcal{P}(\Theta)}\left\{\frac{13}{3}\,\mathbb{E}_{\theta'\sim\rho'}[\mathcal{H}^2(P^\star,P_{\theta'})] + \frac{16\,\mathrm{KL}(\rho'\|\pi)}{n}\right\}
+ \frac{16\log(1/\delta)}{n}.
\end{align*}
Since $\pi'=\pi$, both infima have the same Hellinger coefficient $13/3$. Upper bounding the first KL coefficient $8/n$ by $16/n$, each infimum is bounded by $\inf_{\rho}\{(13/3)\,\mathbb{E}_\rho[\mathcal{H}^2] + 16\,\mathrm{KL}(\rho\|\pi)/n\}$. Adding the two copies and pulling out a factor of $2$ gives~\eqref{eq:oracle}.
\end{proof}

\begin{rem}[Temperature choice]\label{rem:explicit-temp}
The condition $\beta_{n,\lambda}<a_1/a_2^2=\sqrt2/16$ ensures positivity of $a_1-\beta_{n,\lambda}a_2^2$. With $\lambda=n/8$, one verifies $\beta_{n,\lambda}=g(1/4)/8\approx 0.068$, giving $a_1-\beta_{n,\lambda}a_2^2 \approx 0.086 \ge 1/12$ and $a_0+\beta_{n,\lambda}a_2^2 \approx 4.29 \le 13/3$. Any $\lambda$ satisfying $(\lambda/n)g(2\lambda/n) < \sqrt2/16$ is admissible.
\end{rem}

\medskip

Under a prior mass condition, the oracle inequality yields explicit convergence rates. The following definition introduces the relevant condition, taken from~\citet[Section~6.1]{alquier2024user}.

\begin{dfn}[Prior mass condition and Catoni dimension]\label{def:prior-mass-condition}
Let $\theta_0 \in \argmin_{\theta \in \Theta} \mathcal{H}^2(P^\star, P_\theta)$. We say that the \emph{prior mass condition} is satisfied with constants $c > 0$ and $d_\pi \geq 0$ if there exists $r_0 > 0$ such that, for any $r \leq r_0$,
\begin{equation}\label{eq:prior-mass}
\pi\bigl(\bigl\{\theta \in \Theta : \mathcal{H}^2(P^\star, P_\theta) \leq \mathcal{H}^2(P^\star, P_{\theta_0}) + r\bigr\}\bigr) \geq \left(\frac{r}{c}\right)^{d_\pi}.
\end{equation}
The exponent $d_\pi$ is called the \emph{Catoni dimension} of the model with respect to the prior~$\pi$. This type of condition is classical in the analysis of posterior contraction rates in Bayesian statistics~\citep{ghosal2017fundamentals}.
\end{dfn}

\medskip

The following lemma bounds the PAC-Bayesian trade-off under the prior mass condition.

\begin{lem}\label{lem:prior-mass-bound}
Under the prior mass condition (Definition~\ref{def:prior-mass-condition}) with constants $c$, $d_\pi$, and $r_0$, for any $a > 0$ and $\beta > 0$ satisfying $a\beta \geq d_\pi/r_0$,
\begin{equation}\label{eq:prior-mass-inf-bound}
\inf_{\rho \in \mathcal{P}(\Theta)} \left\{ a\,\mathbb{E}_{\theta\sim\rho}\bigl[\mathcal{H}^2(P^\star, P_\theta)\bigr] + \frac{\mathrm{KL}(\rho\Vert\pi)}{\beta} \right\}
\leq a\,\mathcal{H}^2(P^\star, P_{\theta_0}) + \frac{d_\pi}{\beta}\log\!\left(\frac{ec\,a\beta}{d_\pi}\right).
\end{equation}
\end{lem}

\begin{proof}
For any $r > 0$, define $B(r) := \bigl\{\theta \in \Theta : \mathcal{H}^2(P^\star, P_\theta) \leq \mathcal{H}^2(P^\star, P_{\theta_0}) + r\bigr\}$ and let $\rho_r := \pi(\cdot \mid B(r))$ denote the restriction of~$\pi$ to $B(r)$. Since $\rho_r \in \mathcal{P}(\Theta)$ for every $r > 0$,
\begin{equation}\label{eq:pmc-step1}
\inf_{\rho \in \mathcal{P}(\Theta)} \left\{ a\,\mathbb{E}_{\theta\sim\rho}[\mathcal{H}^2(P^\star, P_\theta)] + \frac{\mathrm{KL}(\rho\Vert\pi)}{\beta} \right\}
\leq \inf_{r > 0} \left\{ a\,\mathbb{E}_{\theta\sim\rho_r}[\mathcal{H}^2(P^\star, P_\theta)] + \frac{\mathrm{KL}(\rho_r\Vert\pi)}{\beta} \right\}.
\end{equation}
Since $\rho_r$ is supported on $B(r)$,
\begin{equation}\label{eq:pmc-hellinger}
\mathbb{E}_{\theta\sim\rho_r}[\mathcal{H}^2(P^\star, P_\theta)] \leq \mathcal{H}^2(P^\star, P_{\theta_0}) + r.
\end{equation}
Moreover, $\mathrm{KL}(\rho_r \Vert \pi) = -\log\pi(B(r))$, and the prior mass condition~\eqref{eq:prior-mass} gives $\pi(B(r)) \geq (r/c)^{d_\pi}$ for $r \leq r_0$, so
\begin{equation}\label{eq:pmc-kl}
\mathrm{KL}(\rho_r\Vert\pi) \leq d_\pi\log\!\left(\frac{c}{r}\right).
\end{equation}
Substituting~\eqref{eq:pmc-hellinger} and~\eqref{eq:pmc-kl} into~\eqref{eq:pmc-step1},
\begin{equation}\label{eq:pmc-step2}
\inf_{\rho \in \mathcal{P}(\Theta)} \left\{ a\,\mathbb{E}_{\theta\sim\rho}[\mathcal{H}^2(P^\star, P_\theta)] + \frac{\mathrm{KL}(\rho\Vert\pi)}{\beta} \right\}
\leq \inf_{0 < r \leq r_0} \left\{ a\bigl(\mathcal{H}^2(P^\star, P_{\theta_0}) + r\bigr) + \frac{d_\pi \log(c/r)}{\beta} \right\}.
\end{equation}
Setting $r = d_\pi/(a\beta)$, which satisfies $r \leq r_0$ since $a\beta \geq d_\pi/r_0$, yields
\begin{equation}\label{eq:pmc-final}
a\,\mathcal{H}^2(P^\star, P_{\theta_0}) + \frac{d_\pi}{\beta} + \frac{d_\pi}{\beta}\log\!\left(\frac{c\,a\beta}{d_\pi}\right)
= a\,\mathcal{H}^2(P^\star, P_{\theta_0}) + \frac{d_\pi}{\beta}\log\!\left(\frac{ec\,a\beta}{d_\pi}\right). \qedhere
\end{equation}
\end{proof}

\medskip

Combining Theorem~\ref{thm:oracle} with Lemma~\ref{lem:prior-mass-bound} yields the following concentration result.

\begin{cor}[Concentration under prior mass condition]\label{cor:prior-mass}
Fix $\delta \in (0,1/2)$ and $\lambda = n/8$. Suppose the prior mass condition (Definition~\ref{def:prior-mass-condition}) holds with constants $c$, $d_\pi$, and $r_0$ satisfying $13n \geq 48\,d_\pi/r_0$. Then with probability at least $1-2\delta$,
\begin{equation}\label{eq:concentration-prior-mass}
\mathbb{E}_{\theta\sim\hat\rho_\lambda}\bigl[\mathcal{H}^2(P^\star,P_\theta)\bigr]
\leq 104\,\mathcal{H}^2(P^\star, P_{\theta_0}) + \frac{384\,d_\pi}{n}\log\!\left(\frac{13ecn}{48d_\pi}\right) + \frac{192\log(1/\delta)}{n}.
\end{equation}
In particular, $\mathcal{H}^2(P^\star, P_{\theta_0}) + d_\pi\log(n)/n$ is a concentration rate for the $\tilde{\rho}$-posterior.
\end{cor}

\begin{proof}
Apply Lemma~\ref{lem:prior-mass-bound} with $a = 26/3$ and $\beta = n/32$, which satisfies $(26/3)(n/32) = 13n/48 \geq d_\pi/r_0$. Substituting into~\eqref{eq:oracle}:
\[
\frac{1}{12}\,\mathbb{E}_{\theta\sim\hat\rho_\lambda}[\mathcal{H}^2(P^\star,P_\theta)]
\leq \frac{26}{3}\,\mathcal{H}^2(P^\star, P_{\theta_0}) + \frac{32d_\pi}{n}\log\!\left(\frac{13ecn}{48d_\pi}\right) + \frac{16\log(1/\delta)}{n}.
\]
Multiplying both sides by $12$ yields~\eqref{eq:concentration-prior-mass}.
\end{proof}

\subsection{Application to Fixed-Design Regression}\label{sec:regression}

We illustrate Theorem~\ref{thm:pac-bayes-independent} in the context of fixed-design regression, recovering the rates of~\citet{baraud2020robust}.

\paragraph{Model setup.}
Following~\citet{baraud2020robust}, consider
\begin{equation}\label{eq:regression-model-basic}
Y_i = f^\star(w_i) + \varepsilon_i, \quad i \in [n],
\end{equation}
where $w_1, \ldots, w_n \in \mathcal{W}$ are fixed design points, $\varepsilon_i$ are i.i.d.\ errors with unknown density $p$ with respect to Lebesgue measure $\lambda$, and $f^\star: \mathcal{W} \to \mathbb{R}$ satisfies $\|f^\star\|_\infty \leq B$.

\paragraph{Embedding in the independent framework.}
Setting $X_i = (w_i, Y_i)$ with reference measure $\mu_i = \delta_{w_i} \otimes \lambda$, the true density is
\begin{equation}\label{eq:true-density-regression}
p_\star^i(w,y) = p(y - f^\star(w)) \cdot \mathds{1}(w = w_i).
\end{equation}
The parameter space is a function class $\Theta = \mathcal{F}$ with $\|f\|_\infty \leq B$ for all $f \in \mathcal{F}$. Fixing a candidate noise density $q$, which may differ from the true $p$, define
\begin{equation}\label{eq:model-density-regression}
p_f^i(w,y) = q(y - f(w)) \cdot \mathds{1}(w = w_i), \qquad f \in \mathcal{F}.
\end{equation}
This defines an independent (non-i.i.d.) model to which Theorem~\ref{thm:pac-bayes-independent} applies.

\paragraph{Structural assumption.}
The key requirement is that translations of $q$ are controlled in Hellinger distance.

\begin{asm}[Order-$\alpha$ candidate density]\label{asm:order-alpha-candidate}
The candidate density $q$ is unimodal and of order $\alpha \in (-1,1]$: there exist constants $0 < c_q \leq C_q$ such that
\begin{equation}\label{eq:order-alpha-condition}
c_q\bigl(|\delta|^{1+\alpha} \wedge C_q^{-1}\bigr) \leq \mathcal{H}^2(q_\delta, q) \leq C_q\bigl(|\delta|^{1+\alpha} \wedge C_q^{-1}\bigr) \quad \forall\, \delta \in \mathbb{R},
\end{equation}
where $q_\delta(\cdot) = q(\cdot - \delta)$ denotes translation by $\delta$.
\end{asm}

This is Definition~26 of~\citet{baraud2018rho}; examples include uniform ($\alpha = 0$) and Gaussian ($\alpha = 1$) densities, and more generally any regular translation model \citep[Chapter~VI]{ibragimov2013statistical}.

\paragraph{Regression loss.}
Following~\citet{baraud2018rho}, define the empirical $(1+\alpha)$-loss
\begin{equation}\label{eq:empirical-loss-d}
d_{1+\alpha}(f, g) := \sum_{i=1}^n \left(|f(w_i) - g(w_i)|^{1+\alpha} \wedge C_q^{-1}\right).
\end{equation}
When $\|f\|_\infty, \|g\|_\infty \leq B$ and $2B \leq C_q^{-1/(1+\alpha)}$, the truncation is inactive, giving $d_{1+\alpha}(f, g) = n \|f - g\|_{n,1+\alpha}^{1+\alpha}$, where
\begin{equation}\label{eq:empirical-norm}
\|f - g\|_{n,1+\alpha}^{1+\alpha} := \frac{1}{n}\sum_{i=1}^n |f(w_i) - g(w_i)|^{1+\alpha}
\end{equation}
is the empirical $(1+\alpha)$-norm. For comparison with~\citet{baraud2020robust}, define also the population $(1+\alpha)$-norm
\begin{equation}\label{eq:population-norm}
\|f - g\|_{1+\alpha} := \left(\int_{\mathcal{W}} |f(w) - g(w)|^{1+\alpha} \, dP_W(w)\right)^{1/(1+\alpha)}.
\end{equation}

All results below are stated in terms of $\|\cdot\|_{n,1+\alpha}$; under standard design conditions, these translate to bounds on the population norm.

The following lemma connects the sample Hellinger distance to the regression loss.
\begin{lem}\label{lem:hellinger-to-regression}
Under the regression setup above and Assumption~\ref{asm:order-alpha-candidate}, assume further that
\begin{equation}\label{eq:boundedness-condition}
2B \leq C_q^{-1/(1+\alpha)}.
\end{equation}
Then there exist constants $c_\alpha, C_\alpha > 0$ depending only on $(c_q, C_q, B, \alpha)$ such that for all $f \in \mathcal{F}$,
\begin{equation}\label{eq:hellinger-to-loss}
c_\alpha \|f - f^\star\|_{n,1+\alpha}^{1+\alpha} - \mathcal{H}^2(p,q) 
\leq \mathcal{H}_n^2(P_\star^{(n)}, P_f^{(n)}) 
\leq C_\alpha \|f - f^\star\|_{n,1+\alpha}^{1+\alpha} + 2\mathcal{H}^2(p,q).
\end{equation}
\end{lem}

The additive decomposition into regression error and noise misspecification $\mathcal{H}^2(p,q)$ is the source of robustness: misspecification of $q$ contributes a bias term rather than invalidating the procedure.

\medskip

The main PAC-Bayes bound for regression requires no covering or entropy condition.

\begin{thm}[PAC-Bayes bound for fixed-design regression]\label{thm:pac-bayes-regression}
Assume the fixed-design regression model~\eqref{eq:regression-model-basic}, Assumption~\ref{asm:order-alpha-candidate}, and the boundedness condition $2B \leq C_q^{-1/(1+\alpha)}$. Fix $\delta \in (0,1/2)$ and a prior $\pi \in \mathcal{P}(\mathcal{F})$. Set $\lambda = n/8$ and let $\hat\rho_\lambda$ be the Gibbs posterior defined by~\eqref{eq:target-gibbs}. Then with probability at least $1-2\delta$,
\begin{align}
&\mathbb{E}_{f \sim \hat\rho_\lambda}\bigl[\|f - f^\star\|_{n,1+\alpha}^{1+\alpha}\bigr] \nonumber\\
&\leq \frac{12}{c_\alpha} \Bigg[
\inf_{\rho \in \mathcal{P}(\mathcal{F})} \left\{ \frac{13C_\alpha}{3}\,\mathbb{E}_{f \sim \rho}\bigl[\|f - f^\star\|_{n,1+\alpha}^{1+\alpha}\bigr] + \frac{8\,\mathrm{KL}(\rho\|\pi)}{n} \right\} \nonumber\\
&\qquad + \inf_{\rho' \in \mathcal{P}(\mathcal{F})} \left\{ \frac{13C_\alpha}{3}\,\mathbb{E}_{f' \sim \rho'}\bigl[\|f' - f^\star\|_{n,1+\alpha}^{1+\alpha}\bigr] + \frac{16\,\mathrm{KL}(\rho'\|\pi)}{n} \right\} \nonumber\\
&\qquad + \frac{209}{12}\,\mathcal{H}^2(p,q) + \frac{16\log(1/\delta)}{n}
\Bigg], \label{eq:pac-bayes-regression-kl}
\end{align}
where $c_\alpha = c_q/2$ and $C_\alpha = 2C_q$ are as in Lemma~\ref{lem:hellinger-to-regression}.
\end{thm}

The bound holds for arbitrary priors on $\mathcal{F}$, with noise misspecification entering only through the additive term $\mathcal{H}^2(p,q)$, the hallmark of $\rho$-estimation.

\medskip

To obtain explicit rates, we specialize to priors on finite $\varepsilon$-nets.

\begin{asm}[Metric entropy of function class]\label{asm:function-class-entropy}
There exists a non-increasing function $H: (0,\infty) \to \mathbb{R}_+$ such that for every $\varepsilon > 0$, one can find a subset $\mathcal{F}_\varepsilon \subset \mathcal{F}$ with $|\mathcal{F}_\varepsilon| \leq \exp(H(\varepsilon))$ satisfying
\[
\inf_{g \in \mathcal{F}_\varepsilon} \|f - g\|_\infty \leq \varepsilon \quad \text{for all } f \in \mathcal{F}.
\]
\end{asm}

\begin{cor}[Entropy bound via $\varepsilon$-net prior]\label{cor:regression-entropy}
Assume the setting of Theorem~\ref{thm:pac-bayes-regression} and additionally Assumption~\ref{asm:function-class-entropy}. Fix $\varepsilon > 0$ and let $\pi$ be the uniform distribution on an $\varepsilon$-net $\mathcal{F}_\varepsilon$ with $|\mathcal{F}_\varepsilon| \leq e^{H(\varepsilon)}$. Then with probability at least $1-2\delta$,
\begin{equation}\label{eq:regression-entropy-bound}
\mathbb{E}_{f \sim \hat\rho_\lambda}\bigl[\|f - f^\star\|_{n,1+\alpha}^{1+\alpha}\bigr]
\leq K_\alpha \left[
\mathcal{H}^2(p,q) + \inf_{f \in \mathcal{F}} \|f - f^\star\|_\infty^{1+\alpha} + \varepsilon^{1+\alpha} + \frac{H(\varepsilon)}{n} + \frac{\log(1/\delta)}{n}
\right],
\end{equation}
where $K_\alpha > 0$ depends only on $(c_q, C_q, B, \alpha)$.
\end{cor}

The three error sources --- noise misspecification $\mathcal{H}^2(p,q)$, model misspecification $\inf_{f \in \mathcal{F}} \|f - f^\star\|_\infty^{1+\alpha}$, and statistical complexity $\varepsilon^{1+\alpha} + H(\varepsilon)/n$ --- contribute additively.

\medskip

Optimizing $\varepsilon$ yields the minimax rate.

\begin{cor}[Optimal rate under polynomial entropy]\label{cor:regression-rate}
Assume the setting of Corollary~\ref{cor:regression-entropy} with $\alpha \geq 0$, and suppose $H(\varepsilon) \leq M\varepsilon^{-d}$ for constants $M, d > 0$. Choose
\begin{equation}\label{eq:epsilon-optimal}
\varepsilon_n = \left(\frac{M}{n}\right)^{1/(d+1+\alpha)}.
\end{equation}
Then with probability at least $1-2\delta$,
\begin{equation}\label{eq:regression-rate}
\mathbb{E}_{f \sim \hat\rho_\lambda}\bigl[\|f - f^\star\|_{n,1+\alpha}\bigr]
\leq K_\alpha' \left[
\mathcal{H}^2(p,q) + \inf_{f \in \mathcal{F}} \|f - f^\star\|_\infty^{1+\alpha} + n^{-\frac{1+\alpha}{d+1+\alpha}} + \frac{\log(1/\delta)}{n}
\right]^{1/(1+\alpha)},
\end{equation}
where $K_\alpha' > 0$ depends on $(c_q, C_q, B, \alpha, M, d)$.
\end{cor}

\begin{rem}
The $\varepsilon$-net prior is illustrative; Theorem~\ref{thm:pac-bayes-regression} accommodates sparsity-inducing priors~\cite{dalalyan2012sparse,alquier2011pac}, Gaussian processes~\cite{reeb2018learning}, and priors on deep neural networks~\cite{pmlr-v119-cherief-abdellatif20a}, among others.
\end{rem}

The rate $n^{-(1+\alpha)/(d+1+\alpha)}$ is minimax-optimal for nonparametric regression with metric entropy of order $\varepsilon^{-d}$; setting $d = V/(1+\alpha)$ recovers the classical rate $n^{-(1+\alpha)/(V+1+\alpha)}$. Proofs are given in Section~\ref{sec:proofs-regression}.

\subsection{Variational Approximation}

The oracle inequalities above optimize over all of $\mathcal{P}(\Theta)$, yielding Gibbs posteriors with intractable partition functions. We now restrict both $\rho$ and $\rho'$ to tractable variational families $\mathcal{F}, \mathcal{F}' \subset \mathcal{P}(\Theta)$.

By the Donsker--Varadhan formula, the softmax competitor~\eqref{eq:softmax-def} admits the representation
\[
\Lambda_{\lambda}(\theta;\pi)
= \frac{1}{\lambda} \log\left(\int_{\Theta}e^{\lambda\,\hat{R}_\psi(\theta,\theta')}\pi(d\theta')\right)
= \sup_{\rho' \in \mathcal{P}(\Theta)} \left\{ \mathbb{E}_{\theta'\sim\rho'}[\hat{R}_\psi(\theta,\theta')] - \frac{1}{\lambda}\mathrm{KL}(\rho'\Vert\pi) \right\}.
\]
Restricting to $\mathcal{F}' \subset \mathcal{P}(\Theta)$ yields a computable lower bound.

\medskip

\begin{dfn}[Variational softmax]\label{def:var-softmax}
Let $\mathcal{F}' \subset \mathcal{P}(\Theta)$ be a variational family. For each $\theta\in\Theta$, define the \emph{variational softmax}
\[
\Lambda^{\mathcal{F}'}_{\lambda}(\theta;\pi)
:=\sup_{\rho'\in\mathcal{F}'}
\left\{
\mathbb{E}_{\theta'\sim\rho'}[\hat{R}_\psi(\theta,\theta')]
-\frac{1}{\lambda}\,\mathrm{KL}(\rho'\Vert\pi)
\right\}.
\]
By definition, $\Lambda^{\mathcal{F}'}_{\lambda}(\theta;\pi) \leq \Lambda_{\lambda}(\theta;\pi)$ for all $\theta \in \Theta$.
\end{dfn}

Similarly, restricting the target posterior to $\mathcal{F} \subset \mathcal{P}(\Theta)$ gives the \emph{variational target posterior}
\begin{equation}\label{eq:target-gibbs-var}
\tilde{\rho}_\lambda
\in \arg\min_{\rho\in\mathcal{F}}
\left\{\mathbb{E}_{\theta\sim\rho}[\Lambda^{\mathcal{F}'}_{\lambda}(\theta;\pi)]
+\frac{1}{\lambda}\mathrm{KL}(\rho\Vert\pi)\right\}.
\end{equation}

\medskip

\begin{thm}[Variational oracle inequality]\label{thm:variational-oracle}
  Fix $\delta \in(0,1/2)$, $\lambda = n/8$, and $\mathcal{F} = \mathcal{F}'$. Let $\tilde{\rho}_\lambda$ be the variational target posterior defined in \eqref{eq:target-gibbs-var}.
  With probability at least $1-2\delta$,
  \begin{equation}\label{eq:variational-oracle}
  \frac{1}{12}\,\mathbb{E}_{\theta\sim\tilde{\rho}_\lambda}[\mathcal{H}^2(P^\star,P_\theta)]
  \leq\inf_{\rho\in\mathcal{F}}\left\{
  \frac{26}{3}\,\mathbb{E}_{\theta\sim\rho}[\mathcal{H}^2(P^\star,P_\theta)]
  +\frac{32}{n}\mathrm{KL}(\rho\Vert\pi)\right\}
  +\frac{16\log(1/\delta)}{n}.
  \end{equation}
  \end{thm}
\begin{proof}
The proof of Theorem~\ref{thm:variational-oracle} in Appendix~\ref{app:variational} shows that, with probability at least $1-2\delta$,
\begin{align*}
\frac{1}{12}\,\mathbb{E}_{\theta\sim\tilde{\rho}_\lambda}[\mathcal{H}^2(P^\star,P_\theta)]
&\leq \inf_{\rho\in\mathcal{F}}\left\{\frac{13}{3}\,\mathbb{E}_{\theta\sim\rho}[\mathcal{H}^2(P^\star,P_\theta)] + \frac{8}{n}\mathrm{KL}(\rho\Vert\pi)\right\}\\
&\quad + \inf_{\rho'\in\mathcal{F}}\left\{\frac{13}{3}\,\mathbb{E}_{\theta'\sim\rho'}[\mathcal{H}^2(P^\star,P_{\theta'})] + \frac{16}{n}\mathrm{KL}(\rho'\Vert\pi)\right\}
+ \frac{16\log(1/\delta)}{n}.
\end{align*}
Since $\mathcal{F} = \mathcal{F}'$, both infima have the same Hellinger coefficient $13/3$. Upper bounding the first KL coefficient $8/n$ by $16/n$, each infimum is bounded by $\inf_{\rho\in\mathcal{F}}\{(13/3)\,\mathbb{E}_\rho[\mathcal{H}^2] + 16\,\mathrm{KL}(\rho\|\pi)/n\}$. Adding the two copies and pulling out a factor of $2$ gives~\eqref{eq:variational-oracle}.
\end{proof}

\medskip

As in Corollary~\ref{cor:prior-mass}, the oracle inequality simplifies under a prior mass condition adapted to the variational family, taken from~\citet[Section~6.1]{alquier2024user}.

\begin{cor}[Variational concentration under prior mass condition]\label{cor:variational-prior-mass}
Fix $\delta\in(0,1/2)$ and $\lambda = n/8$. Assume that there exists $\rho_n \in \mathcal{F}$ satisfying
\begin{equation}\label{eq:gen-prior-mass}
\mathbb{E}_{\theta\sim\rho_n}\bigl[\mathcal{H}^2(P^\star,P_\theta)\bigr] \leq \frac{1}{n}
\quad\text{and}\quad
\mathrm{KL}(\rho_n\Vert\pi) \leq d_\pi\log n,
\end{equation}
where $d_\pi$ is the Catoni dimension (Definition~\ref{def:prior-mass-condition}).
Then with probability at least $1-2\delta$,
\begin{equation}\label{eq:variational-concentration}
\mathbb{E}_{\theta\sim\tilde{\rho}_\lambda}\bigl[\mathcal{H}^2(P^\star,P_\theta)\bigr]
\leq \frac{384\,d_\pi\log n + 104}{n} + \frac{192\log(1/\delta)}{n}.
\end{equation}
In particular, the variational $\tilde{\rho}$-posterior concentrates at the rate $\mathcal{O}(d_\pi\log(n)/n)$.
\end{cor}

\begin{proof}
Since $\rho_n \in \mathcal{F}$, it is a feasible candidate for the infimum in~\eqref{eq:variational-oracle}. Substituting the bounds from~\eqref{eq:gen-prior-mass}:
\begin{equation}\label{eq:var-cor-step1}
\frac{26}{3}\,\mathbb{E}_{\theta\sim\rho_n}[\mathcal{H}^2(P^\star,P_\theta)]
+ \frac{32\,\mathrm{KL}(\rho_n\Vert\pi)}{n}
\leq \frac{26}{3n} + \frac{32\,d_\pi\log n}{n}.
\end{equation}
Substituting~\eqref{eq:var-cor-step1} into~\eqref{eq:variational-oracle}:
\begin{equation}\label{eq:var-cor-step2}
\frac{1}{12}\,\mathbb{E}_{\theta\sim\tilde{\rho}_\lambda}[\mathcal{H}^2(P^\star,P_\theta)]
\leq \frac{26}{3n} + \frac{32\,d_\pi\log n}{n} + \frac{16\log(1/\delta)}{n}.
\end{equation}
Multiplying both sides by $12$ yields~\eqref{eq:variational-concentration}.
\end{proof}

\begin{rem}[Comparison with the unrestricted case]
When $\mathcal{F} = \mathcal{P}(\Theta)$, condition~\eqref{eq:gen-prior-mass} is implied by the prior mass condition (Definition~\ref{def:prior-mass-condition}) by choosing $\rho_n = \pi(\cdot\mid B(1/n))$, which gives $\mathrm{KL}(\rho_n\Vert\pi) = -\log\pi(B(1/n)) \leq d_\pi\log(cn)$. For strict variational subfamilies $\mathcal{F} \subsetneq \mathcal{P}(\Theta)$, condition~\eqref{eq:gen-prior-mass} is a genuine requirement on the approximation capacity of $\mathcal{F}$: it asks that $\mathcal{F}$ contains a distribution that simultaneously concentrates near $P^\star$ in Hellinger distance and remains close to the prior in KL divergence. The rate $\mathcal{O}(d_\pi\log(n)/n)$ is suboptimal by a logarithmic factor compared to the minimax rate $\mathcal{O}(d_\pi/n)$.
\end{rem}

The proof is given in Appendix~\ref{app:variational}.
\subsection{Computational Analysis}

Computing the variational posterior~\eqref{eq:target-gibbs-var} requires solving an optimization problem that we now reformulate as a saddle-point problem admitting efficient first-order methods. We develop this for canonical exponential families with Gaussian variational approximations, establishing convergence guarantees uniform in $\lambda$.

\subsubsection*{Variational objectives}

Let $\Phi\subset\mathbb{R}^{d_\phi}$ and $\mathcal{N}\subset\mathbb{R}^{d_\nu}$ parameterize $\mathcal{F}$ and $\mathcal{F}'$ surjectively, writing $\rho_\phi\in\mathcal{F}$ and $\rho'_\nu\in\mathcal{F}'$. Two natural objectives arise depending on the order of optimization over the competitor $\rho'$ and integration over $\theta\sim\rho_\phi$.

\begin{dfn}[Pointwise and joint variational objectives]
\label{def:variational-objectives}
\begin{enumerate}
  \item The \emph{pointwise variational objective}
    $\mathcal{J}:\Phi\to\mathbb{R}$ is
    \begin{equation}\label{eq:J-pointwise}
      \mathcal{J}(\phi)
      := \mathbb{E}_{\theta\sim\rho_\phi}
         \bigl[\Lambda^{\mathcal{F}'}_\lambda(\theta;\pi)\bigr]
         + \frac{1}{\lambda}\mathrm{KL}(\rho_\phi\|\pi),
    \end{equation}
    where the supremum over $\rho'\in\mathcal{F}'$ implicit in $\Lambda^{\mathcal{F}'}_\lambda$ is taken separately for each $\theta$.

  \item The \emph{joint variational objective}
    $\tilde{\mathcal{J}}:\Phi\to\mathbb{R}$ is
    \begin{equation}\label{eq:J-joint}
      \tilde{\mathcal{J}}(\phi)
      := \sup_{\rho'\in\mathcal{F}'}
         \Bigl\{
           \mathbb{E}_{(\theta,\theta')\sim\rho_\phi\otimes\rho'}
           \bigl[\hat{R}_\psi(\theta,\theta')\bigr]
           -\frac{1}{\lambda}\mathrm{KL}(\rho'\|\pi)
         \Bigr\}
         + \frac{1}{\lambda}\mathrm{KL}(\rho_\phi\|\pi),
    \end{equation}
    where a single $\rho'$ is chosen to be optimal on average over $\theta\sim\rho_\phi$.
\end{enumerate}
\end{dfn}

Jensen's inequality gives $\tilde{\mathcal{J}}(\phi)\le\mathcal{J}(\phi)$, with equality when the pointwise maximizer $\rho'^\star(\theta)$ does not depend on $\theta$. The gap $\Delta(\phi):=\mathcal{J}(\phi)-\tilde{\mathcal{J}}(\phi)\ge 0$ is bounded in Proposition~\ref{prop:gap-bound}.

\subsubsection*{Saddle-point reformulation}

\begin{dfn}[Primal--dual objective]\label{def:saddle-objective-main}
Define $\mathcal{L}_n:\Phi\times\mathcal{N}\to\mathbb{R}$ by
\begin{equation}\label{eq:saddle-obj}
  \mathcal{L}_n(\phi,\nu)
  :=\mathbb{E}_{\theta\sim\rho_\phi}
    \mathbb{E}_{\theta'\sim\rho'_\nu}
    \bigl[\hat{R}_\psi(\theta,\theta')\bigr]
    +\frac{1}{\lambda}\mathrm{KL}(\rho_\phi\|\pi)
    -\frac{1}{\lambda}\mathrm{KL}(\rho'_\nu\|\pi).
\end{equation}
\end{dfn}

\begin{thm}[Saddle-point equivalence]
\label{thm:saddle-equiv-main}
Under the surjectivity assumption, the following hold.
\begin{enumerate}
  \item[\emph{(i)}] For every $\phi\in\Phi$, the supremum over $\nu$ is attained and
    \begin{equation}\label{eq:saddle-Jtilde}
      \sup_{\nu\in\mathcal{N}}\mathcal{L}_n(\phi,\nu)
      \;=\;\tilde{\mathcal{J}}(\phi).
    \end{equation}

  \item[\emph{(ii)}] For every $\phi\in\Phi$,
    \begin{equation}\label{eq:Jtilde-leq-J}
      \tilde{\mathcal{J}}(\phi)
      \;\le\;\mathcal{J}(\phi)
      \;\le\;\tilde{\mathcal{J}}(\phi)+\Delta(\phi),
    \end{equation}
    where the gap satisfies
    \begin{equation}\label{eq:gap-bound}
      \Delta(\phi)
      \;\le\;
      \frac{\bar{G}^2}{16}\,\mathrm{tr}(\Sigma_\phi),
    \end{equation}
    with $\Sigma_\phi$ the covariance of $\rho_\phi$ and $\bar{G}$ the score bound from Assumption~\ref{asm:regularity-main}(3).

  \item[\emph{(iii)}] The infimum satisfies
    \begin{equation}\label{eq:inf-equiv}
      \inf_{\phi\in\Phi}\sup_{\nu\in\mathcal{N}}\mathcal{L}_n(\phi,\nu)
      \;=\;\inf_{\phi\in\Phi}\tilde{\mathcal{J}}(\phi)
      \;\le\;\inf_{\phi\in\Phi}\mathcal{J}(\phi).
    \end{equation}
    Moreover, $\phi^\star$ is a stationary point of $\tilde{\mathcal{J}}$ if and only if there exists $\nu^\star\in\mathcal{N}$ such that $(\phi^\star,\nu^\star)$ is a first-order stationary point of $\mathcal{L}_n$.

  \item[\emph{(iv)}] In the PAC-Bayes regime $\lambda=\Theta(n)$, if $\mathrm{tr}(\Sigma_{\phi^\star})=\mathcal{O}(d/n)$, then $\Delta(\phi^\star)=\mathcal{O}(d\bar{G}^2/n)\to 0$, so $\tilde{\mathcal{J}}(\phi^\star)=\mathcal{J}(\phi^\star)$ asymptotically.
\end{enumerate}
\end{thm}

\subsubsection*{Convergence guarantees for exponential families}

\begin{dfn}[Canonical exponential family]
Let $\{p_\theta:\theta\in\Theta\subset\mathbb{R}^d\}$ be a canonical exponential family on $(\mathcal{X},\mathcal{A})$:
\[
  p_\theta(x)=h(x)\exp\bigl\{\langle\theta,T(x)\rangle-A(\theta)\bigr\},
\]
where $T:\mathcal{X}\to\mathbb{R}^d$ is the sufficient statistic, $A:\Theta\to\mathbb{R}$ the log-partition function, and $\mu(\theta):=\nabla A(\theta)$, $I(\theta):=\nabla^2 A(\theta)$ the mean map and Fisher information matrix.
\end{dfn}

\begin{dfn}[Variational families and priors]
We consider mean-field Gaussian variational families:
\begin{itemize}
  \item \textit{Target posterior:}
    $\rho_\phi=\mathcal{N}(m,\Sigma)$,
    $\phi=(m,\Sigma)\in\mathbb{R}^d\times\mathbb{S}_{++}^d$.
  \item \textit{Competitor posterior:}
    $\rho'_\nu=\mathcal{N}(m',\mathrm{diag}(\sigma'^2))$,
    $\nu=(m',s)\in\mathbb{R}^d\times\mathbb{R}^d$,
    $\sigma_i^{\prime 2}=e^{s_i}$.
  \item \textit{Prior:}
    $\pi=\mathcal{N}(m_\pi,\Sigma_\pi)$.
\end{itemize}
\end{dfn}

\begin{asm}[Regularity conditions]
\label{asm:regularity-main}
The following hold uniformly on a compact region $\bar\Theta\subset\Theta$ containing the supports of $\rho_\phi$ and $\rho'_\nu$:
\begin{enumerate}
  \item \emph{Fisher lower bound:}
    $I(\theta)\succeq\sigma_0^2\mathbb{I}_d$ for some $\sigma_0>0$.
  \item \emph{Bounded second moment:}
    $\mathbb{E}_{X\sim P^\star}\bigl[\|T(X)-\mu(\theta')\|^2\bigr]\le B^2<\infty$ uniformly in $\theta'\in\bar\Theta$.
  \item \emph{Bounded scores:}
    $\bar{G}:=\sup_{\theta\in\bar\Theta}\sup_{x\in\mathcal{X}}\|\nabla_\theta\log p_\theta(x)\|<\infty$.
  \item \emph{Bounded log-likelihood ratios:}
    there exists $M<\infty$ such that $|\log p_{\theta'}(x)-\log p_\theta(x)|\le M$ for all $\theta,\theta'\in\bar\Theta$ and $P^\star$-a.e.\ $x$. This is implied by (3) together with compactness of $\bar\Theta$.
  \item \emph{Curvature margin:}
    $\mu_0:= c_M\sigma_0^2 - B^2/4 > 0$,
    where $c_M:=\min_{t\in[e^{-M/2},\,e^{M/2}]}t/(t+1)^2 > 0$.
\end{enumerate}
\end{asm}

Here $c_M$ is the minimum of $\varphi'(u)$ on $\{|u|\le M\}$ with $\varphi(u):=\psi(e^u)$ (Lemma~\ref{lem:varphi-bounds-app}). The curvature margin $\mu_0>0$ requires that the weighted Fisher information dominates the variance of the sufficient statistic.

\begin{prop}[Gap bound]\label{prop:gap-bound}
Under Assumption~\ref{asm:regularity-main}(3), the gap between the pointwise and joint variational objectives satisfies, for every $\phi \in \Phi$,
\begin{equation}\label{eq:gap-bound-prop}
0 \;\leq\; \Delta(\phi) \;:=\; \mathcal{J}(\phi) - \tilde{\mathcal{J}}(\phi) \;\leq\; \frac{\bar{G}^2}{16}\,\mathrm{tr}(\Sigma_\phi),
\end{equation}
where $\Sigma_\phi$ is the covariance matrix of $\rho_\phi$ and $\bar{G}$ is the score bound from Assumption~\ref{asm:regularity-main}(3). In particular, in the PAC-Bayes regime $\lambda = \Theta(n)$, if $\mathrm{tr}(\Sigma_{\phi^\star}) = \mathcal{O}(d/n)$ then $\Delta(\phi^\star) = \mathcal{O}(d\bar{G}^2/n) \to 0$, so the two objectives coincide asymptotically.
\end{prop}

Proofs are given in Appendix~\ref{app:saddle-equiv}.

\begin{thm}[NC-SC geometry]
\label{thm:nc-sc-main}
Under Assumption~\ref{asm:regularity-main}:
\begin{enumerate}
  \item $\mathcal{L}_n(\phi,\nu)$ is $L$-smooth in $(\phi,\nu)$ with $L=L_\psi+L_{\mathrm{KL}}/\lambda$, where $L_\psi$ and $L_{\mathrm{KL}}$ are independent of $\lambda$.
  \item For any fixed $\phi$ and any data realization $(X_1,\dots,X_n)$ in the support of $P^{\star n}$, the map $\nu\mapsto\mathcal{L}_n(\phi,\nu)$ is $\mu$-strongly concave with $\mu\ge\mu_0>0$ independent of both $\lambda$ and the data.
\end{enumerate}
The NC-SC condition number $\kappa:=L/\mu=\mathcal{O}(1)$ uniformly in $\lambda>0$.
\end{thm}

\begin{cor}[Optimization guarantee]\label{cor:optimization-main}
Assume that $\mathcal{L}_n$ satisfies the NC--SC geometry of
Theorem~\ref{thm:nc-sc-main}.
Let $(x_t)_{t\ge0}$ be the iterates of a projected stochastic
extragradient method applied to
\[
  \min_{\phi\in\Phi}\max_{\nu\in\mathcal{N}}\;\mathcal{L}_n(\phi,\nu),
\]
with stepsize $\eta\asymp 1/L$ and unbiased stochastic gradients with
variance bounded by $\sigma^2$.
Then, for any $T\ge1$,
\[
  \min_{0\le t<T}\mathbb{E}\!\left[\|\mathcal{G}_\eta(x_t)\|^2\right]
  \;=\;
  \mathcal{O}\!\left(
    \frac{\kappa\,\Delta}{T}
    \;+\;
    \frac{\kappa\,\sigma^2}{\sqrt{T}}
  \right),
\]
where $\mathcal{G}_\eta$ denotes the gradient mapping,
$\Delta:=\mathcal{L}_n(x_0)-\inf_{\phi}\sup_{\nu}\mathcal{L}_n(\phi,\nu)$
is the initial optimality gap, and
$\kappa:=L/\mu$ is the condition number.

Importantly, since $\kappa=\mathcal{O}(1)$ uniformly in $\lambda$
(Theorem~\ref{thm:nc-sc-main}),
this guarantee holds for any $\lambda>0$, including
$\lambda=\Theta(n)$;
see~\citet{juditsky2011solving,lin2020gradient}.
\end{cor}

Complete proofs are in Appendices~\ref{app:saddle-equiv}--\ref{app:nc-sc}.

\section{Numerical Experiments}\label{sec:numerical}

We assess the variational $\tilde{\rho}$-posterior on one-dimensional exponential families under $\varepsilon$-contamination. In each case the model is correctly specified, but the data-generating distribution is a contaminated mixture.

\paragraph{Gaussian location model.}
Consider
\[
P_\theta = \mathcal{N}(\theta,1), \qquad \theta \in \mathbb{R},
\]
with true parameter $\theta^\star = 0$ and contaminated distribution
\[
P_\varepsilon^\star
=
(1-\varepsilon)\,\mathcal{N}(0,1)
+ \varepsilon\,\mathcal{N}(8,1).
\]
A Gaussian prior $\pi = \mathcal{N}(0,4)$ is placed on $\theta$, and the $\tilde{\rho}$-posterior is approximated by a Gaussian variational family $q_\phi(\theta) = \mathcal{N}(m,s^2)$, optimized via the saddle-point objective $\mathcal{L}_n(\phi,\nu)$ with temperature $\lambda = \tau n$ ($\tau=0.5$), using Adam with $200$ iterations and the reparameterization trick. The $\tilde{\rho}$-estimator is the variational posterior mean $\hat\theta_\rho = m$. We compare with the MLE $\hat\theta_{\mathrm{MLE}} = \bar X$ and the conjugate Bayes posterior mean $\hat\theta_B$.

\paragraph{Poisson intensity model.}
Consider
\[
X_i \mid \mu \sim \mathrm{Pois}(\mu), \qquad \mu>0,
\]
with true intensity $\mu^\star = 3$ and contaminated distribution
\[
P_\varepsilon^\star
=
(1-\varepsilon)\,\mathrm{Pois}(3)
+ \varepsilon\,\mathrm{Pois}(30).
\]
The Bayesian estimator uses a $\mathrm{Gamma}(1,1)$ prior with conjugate posterior mean $\hat\mu_B$. For the $\tilde{\rho}$-posterior, we reparametrize $\eta=\log\mu$, place a Gaussian prior on $\eta$, and optimize a Gaussian variational family $q_\psi(\eta) = \mathcal{N}(m_\rho,s_\rho^2)$ as above. The $\tilde{\rho}$-estimator is
\[
\hat\mu_\rho
= \mathbb{E}_{q_\psi}[e^\eta]
= \exp\big(m_\rho + \frac12 s_\rho^2\big).
\]
We compare with $\hat\mu_{\mathrm{MLE}} = \bar X$ and $\hat\mu_B$.

\paragraph{Uniform scale model.}
Consider
\[
X_i \mid \theta \sim \mathrm{Uniform}(0,\theta), \qquad \theta>0,
\]
with true parameter $\theta^\star = 1$ and contaminated distribution
\[
P_\varepsilon^\star
=
(1-\varepsilon)\,\mathrm{Uniform}(0,1)
+ \varepsilon\,\mathrm{Uniform}(101,102),
\]
where contaminated observations lie far outside the support of the clean distribution. The MLE $\hat\theta_{\mathrm{MLE}} = \max_i X_i$ is highly sensitive to outliers. The Bayesian estimator uses the prior $\pi(\theta) \propto \theta^{-\alpha}\mathds{1}_{\theta\ge a}$ with $\alpha = 2$ and $a = 0.5$, giving Pareto posterior mean
\[
\hat\theta_B
=
\frac{n+\alpha-1}{n+\alpha-2}\cdot (a\vee X_{(n)}),
\qquad  X_{(n)}=\max_i X_i.
\]
For the $\tilde{\rho}$-posterior, we reparametrize $u = \log\theta$ and use a Gaussian variational approximation $q_\psi(u) = \mathcal{N}(m_\rho, s_\rho^2)$, giving $\tilde{\rho}$-estimator
\[
\hat\theta_\rho
=
\exp\!\big(m_\rho + \frac12 s_\rho^2\big).
\]

\paragraph{Experimental setup.}
For each configuration $(n,\tau,\varepsilon)$, $T = 1000$ independent datasets are generated from $P_\varepsilon^\star$ with $n=200$, $\tau = 0.5$, and $\varepsilon \in \{0, 0.05, 0.08, 0.10\}$. Performance is measured by the empirical risk
\[
\widehat R(\delta;\theta^\star)
=
\frac{1}{T}
\sum_{k=1}^{T}
\big(\delta(X^{(k)})-\theta^\star\big)^2.
\]
\paragraph{Results.}
Figures~\ref{fig:gaussian_results}--\ref{fig:uniform_results} report posterior risk (left), RMSE (middle), and posterior densities at $\varepsilon=10\%$ (right) for each model. Across all three settings, the $\tilde{\rho}$-posterior maintains stability under contamination while the MLE and Bayes estimators deteriorate rapidly. The effect is most pronounced in the uniform model, where a single outlier can inflate the MLE to $\approx 102$; the $\tilde{\rho}$-posterior remains concentrated near $\theta^\star = 1$.

\begin{figure}[!ht]
\centering
\includegraphics[width=0.32\textwidth]{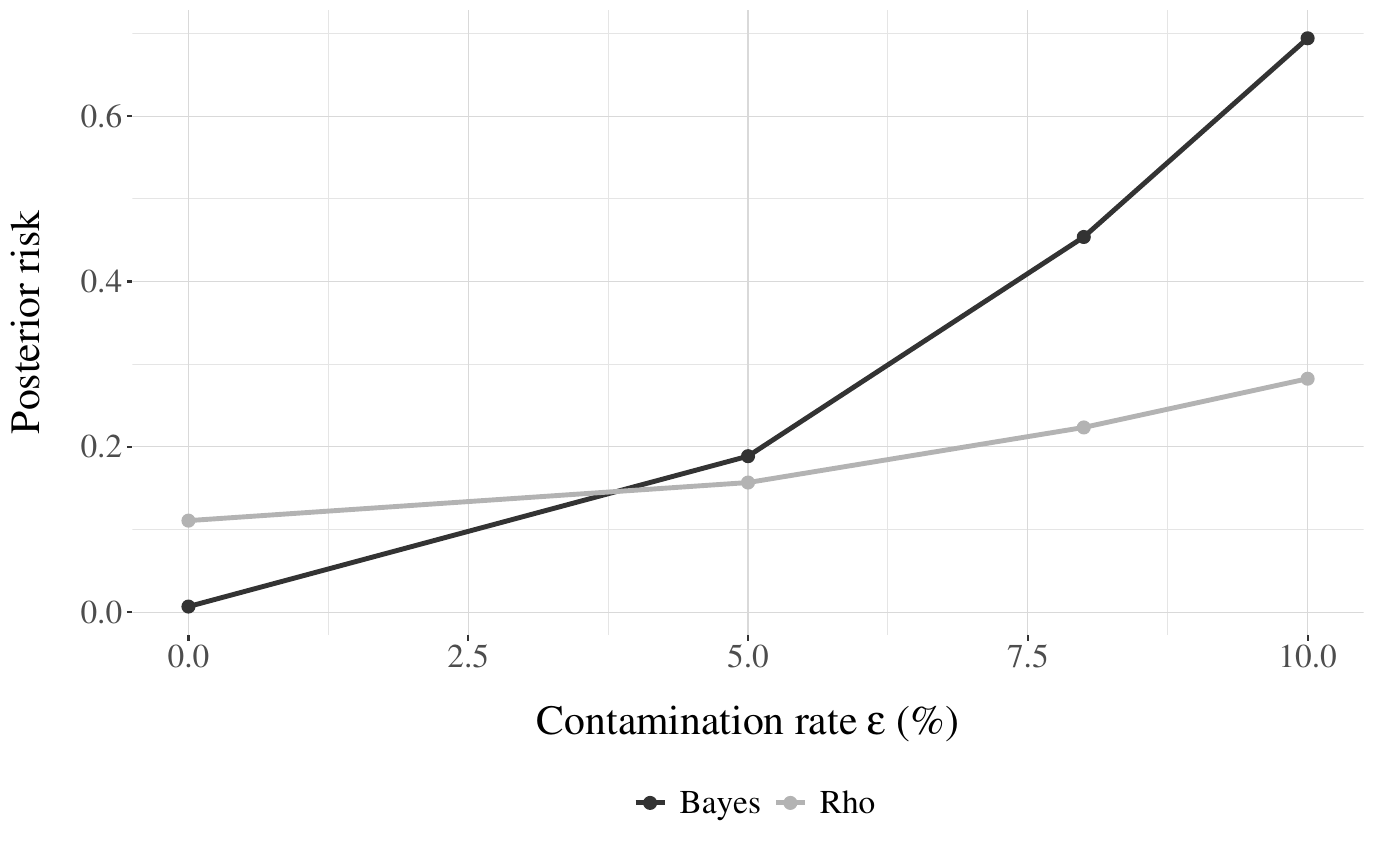}
\includegraphics[width=0.32\textwidth]{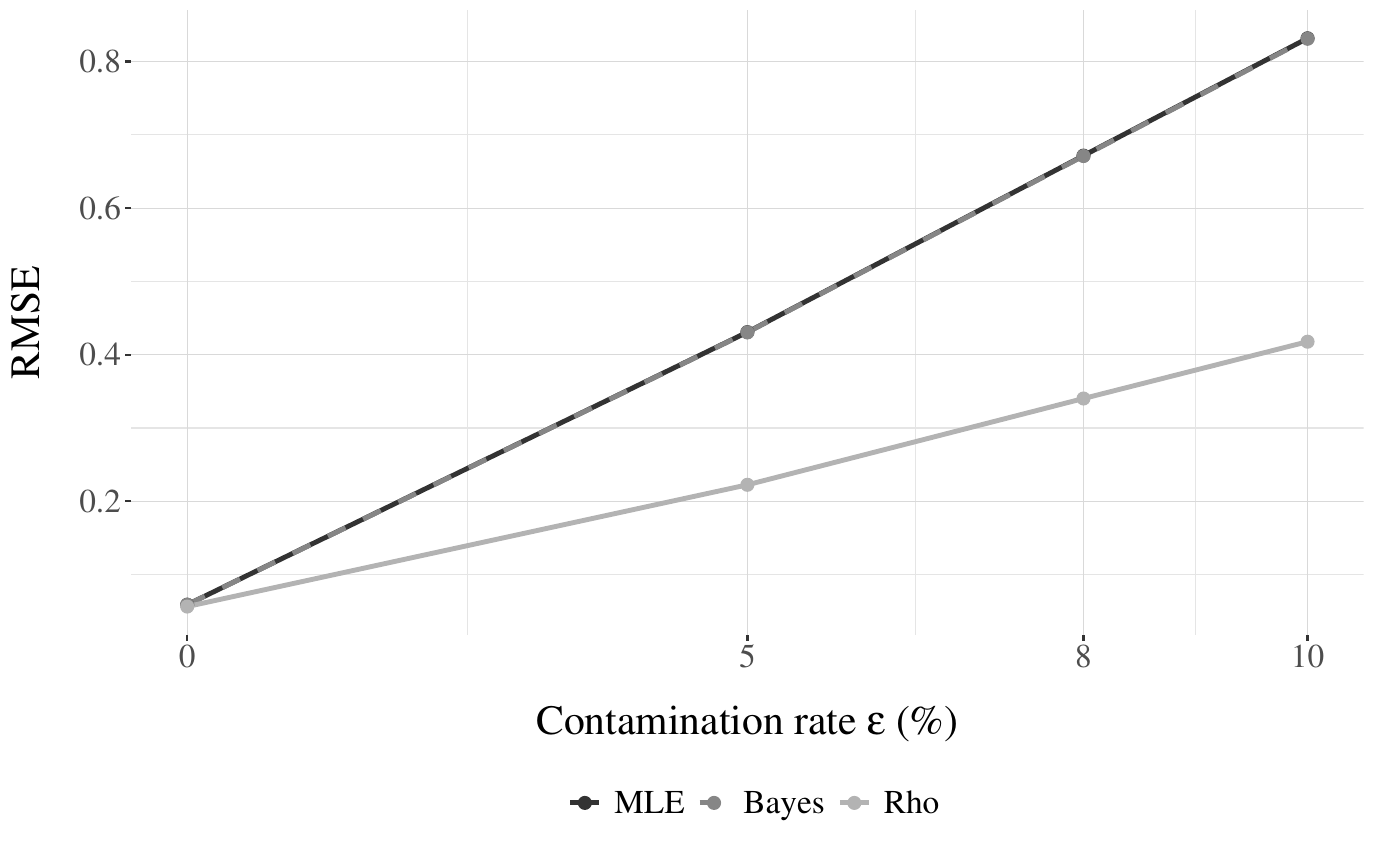}
\includegraphics[width=0.32\textwidth]{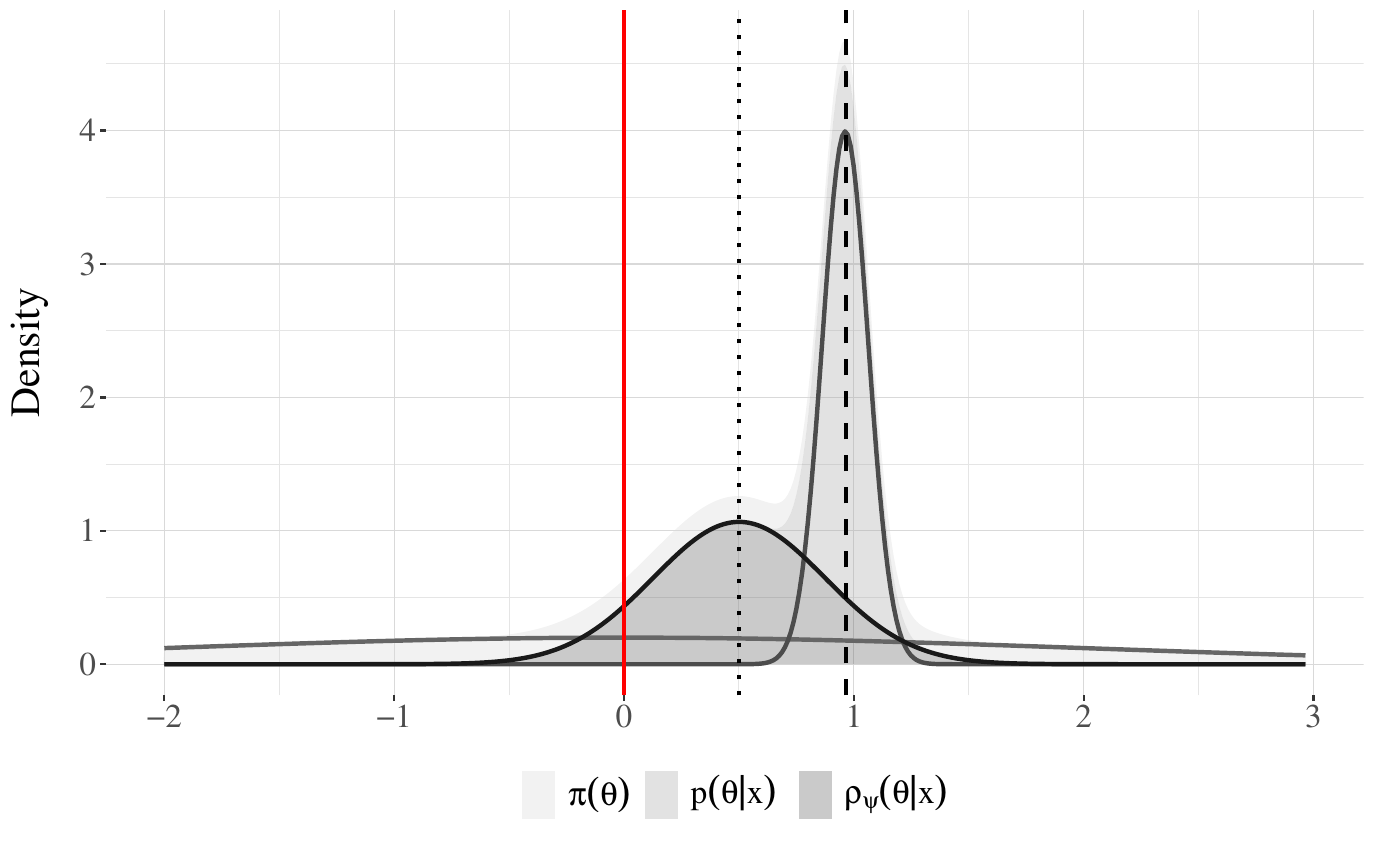}
\caption{Gaussian location model results for $n=200$ and $\tau=0.5$. \textit{Left:} Posterior risk vs.\ contamination rate $\varepsilon$. \textit{Middle:} RMSE vs.\ contamination rate. \textit{Right:} Posterior densities for a single dataset at $\varepsilon = 10\%$.}
\label{fig:gaussian_results}
\end{figure}

\begin{figure}[!ht]
\centering
\includegraphics[width=0.32\textwidth]{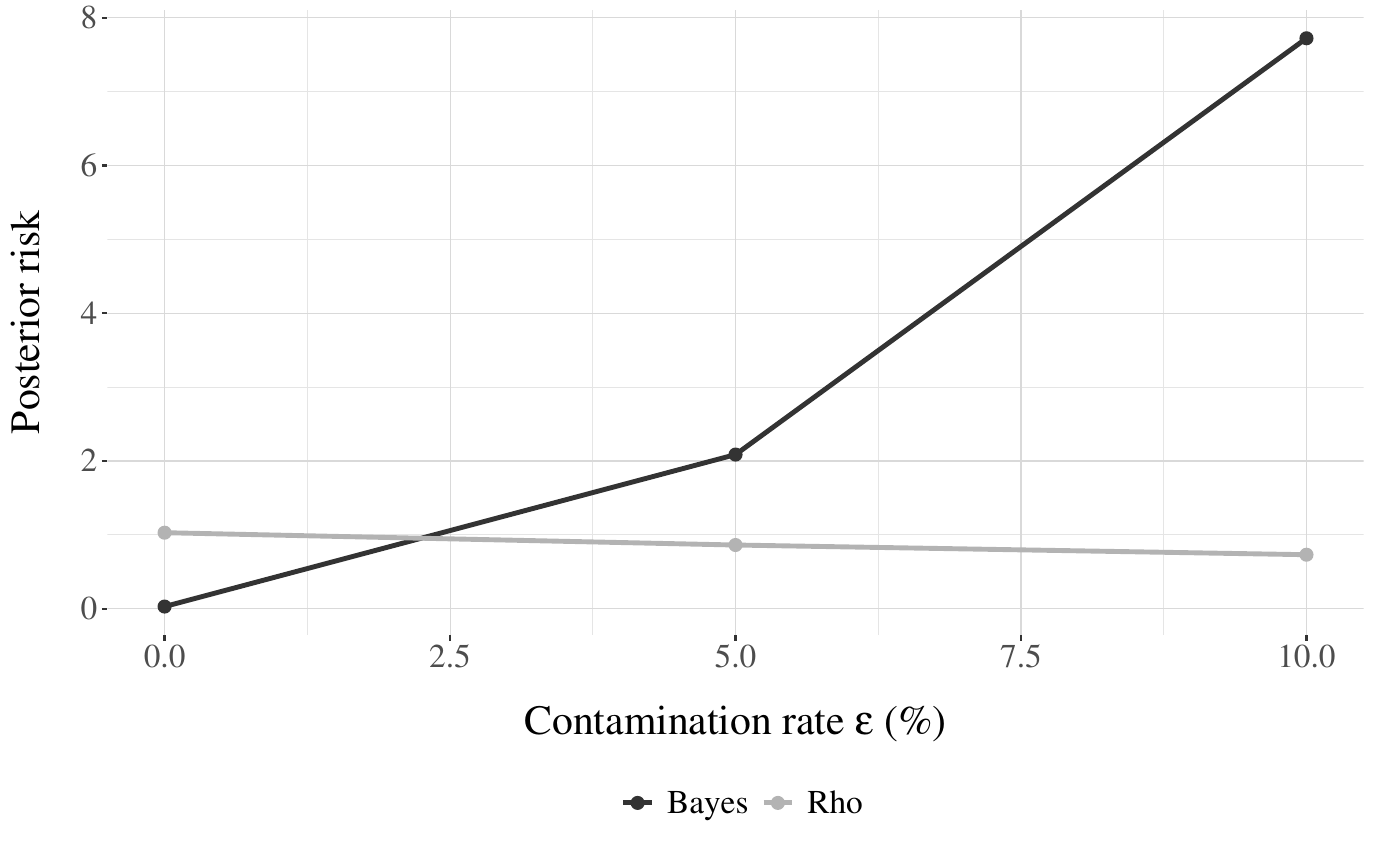}
\includegraphics[width=0.32\textwidth]{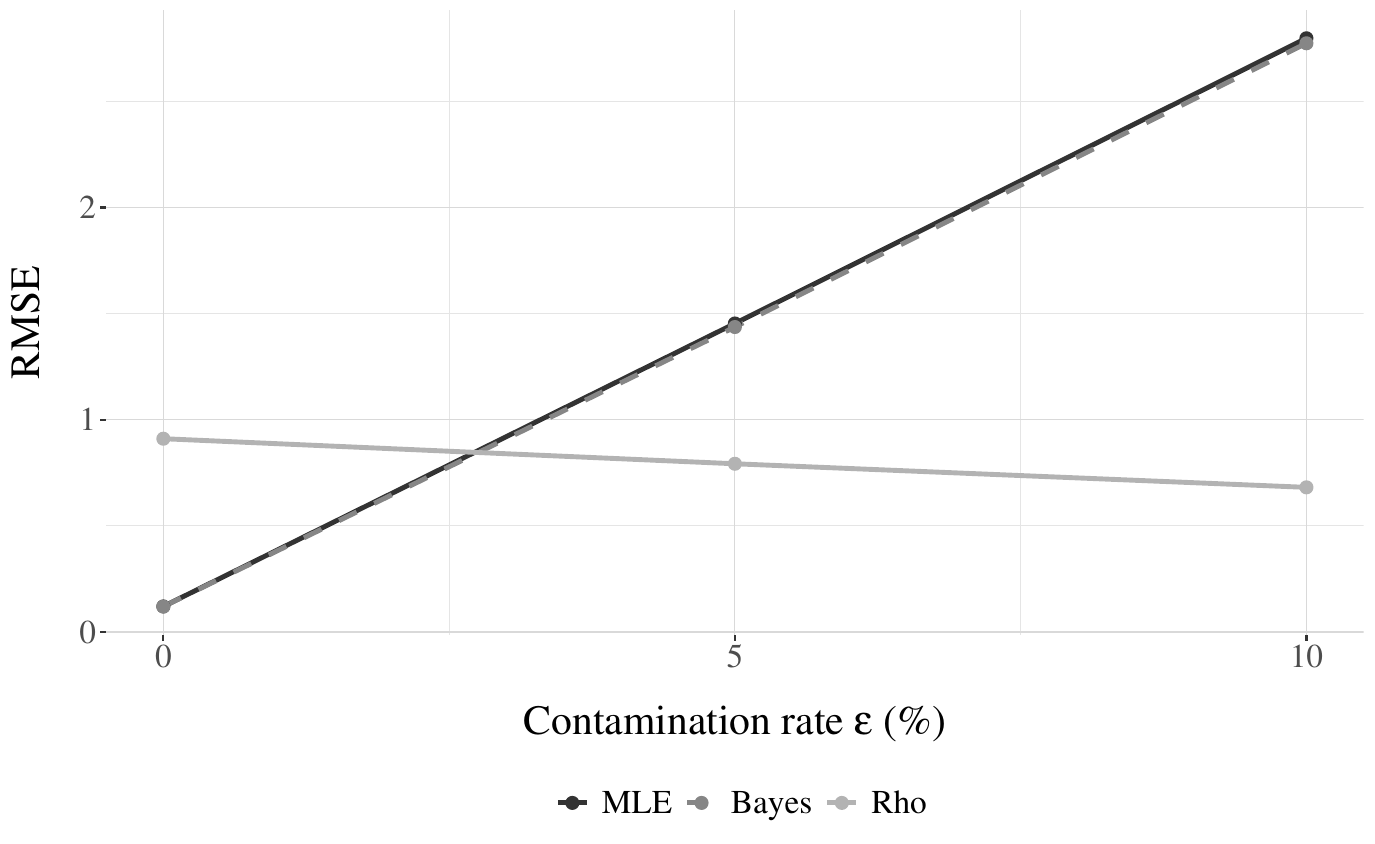}
\includegraphics[width=0.32\textwidth]{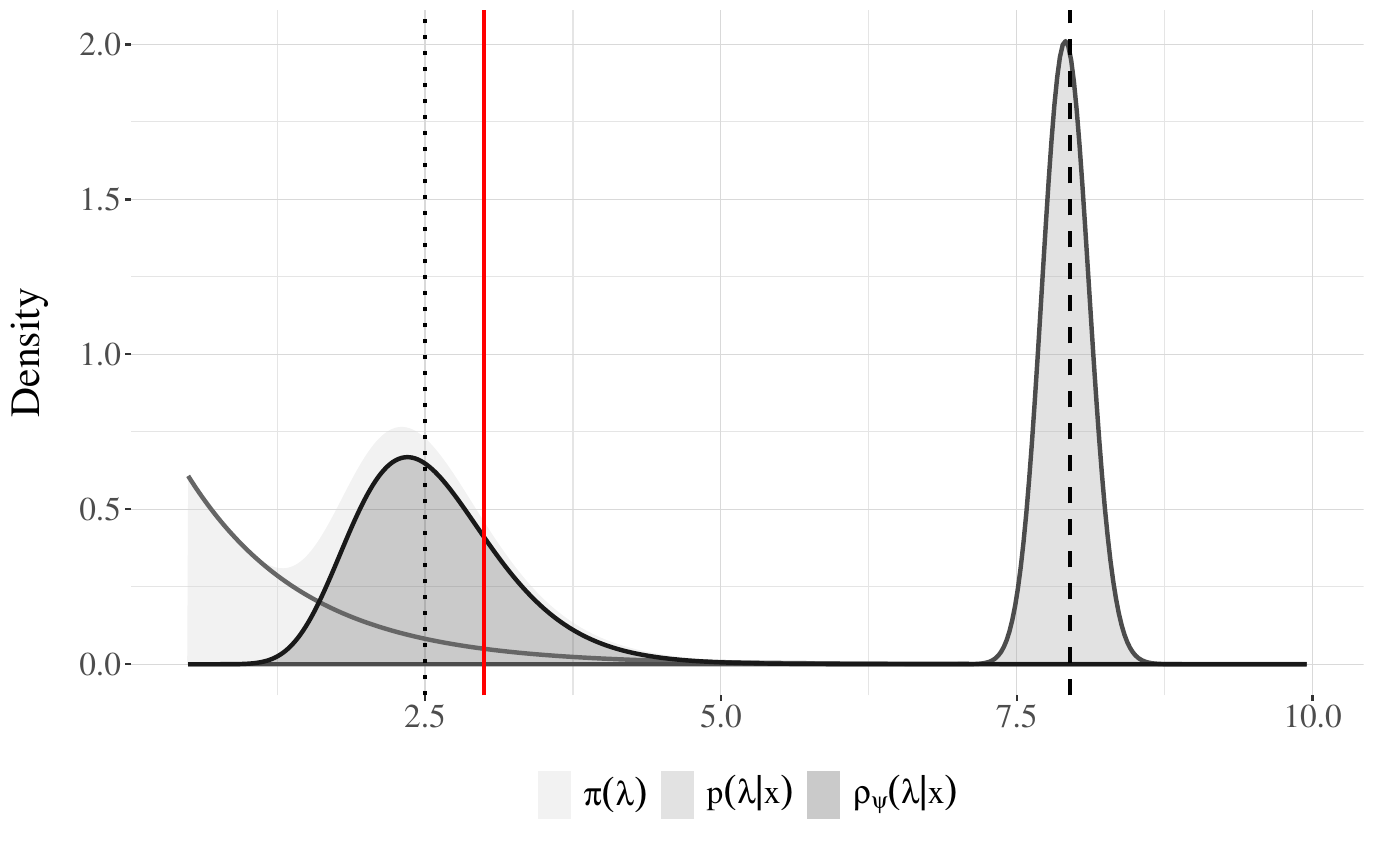}
\caption{Poisson intensity model results for $n=200$ and $\tau=0.5$. \textit{Left:} Posterior risk vs.\ contamination rate $\varepsilon$. \textit{Middle:} RMSE vs.\ contamination rate. \textit{Right:} Posterior densities for a single dataset at $\varepsilon = 10\%$.}
\label{fig:poisson_results}
\end{figure}

\begin{figure}[!ht]
\centering
\includegraphics[width=0.32\textwidth]{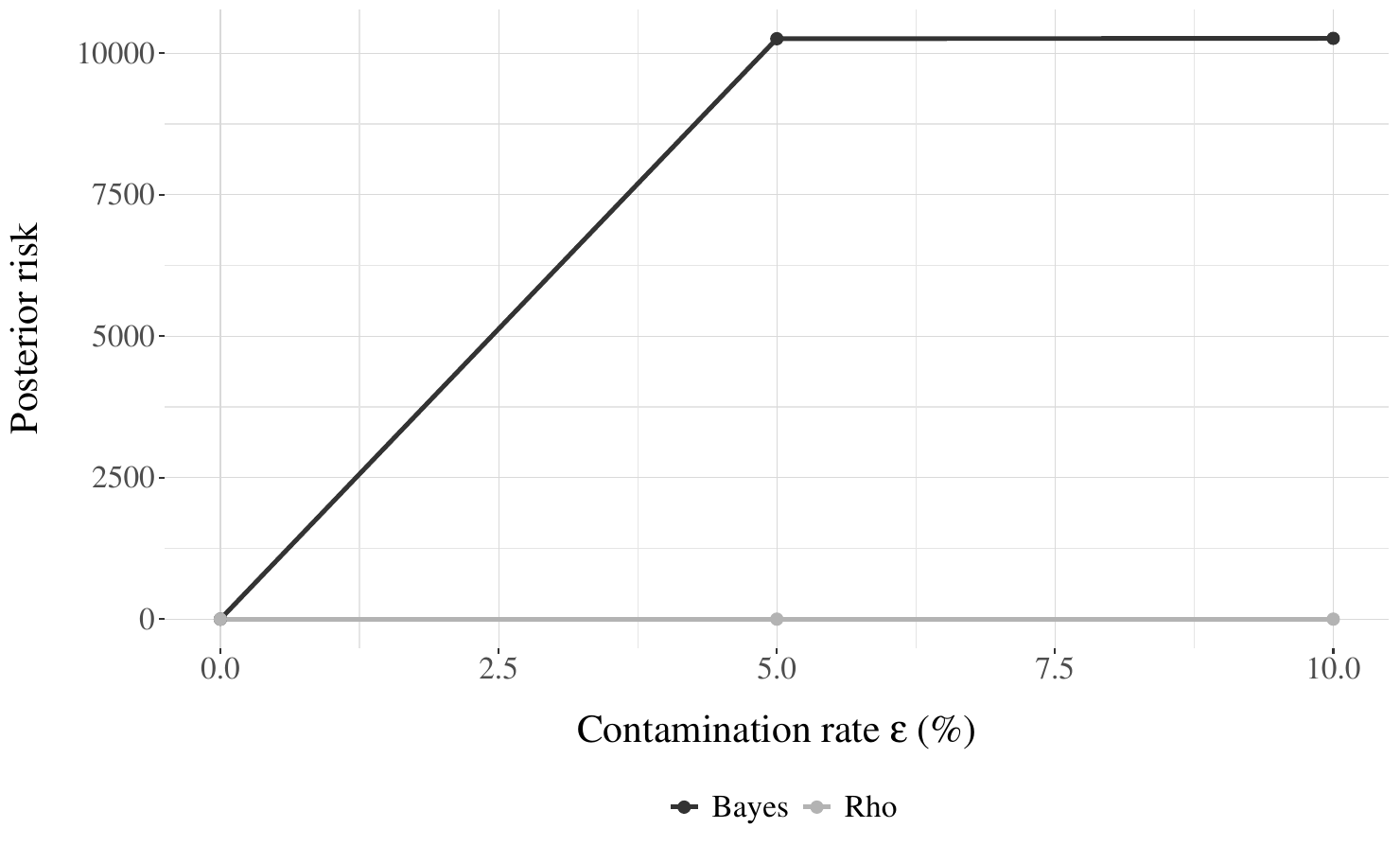}
\includegraphics[width=0.32\textwidth]{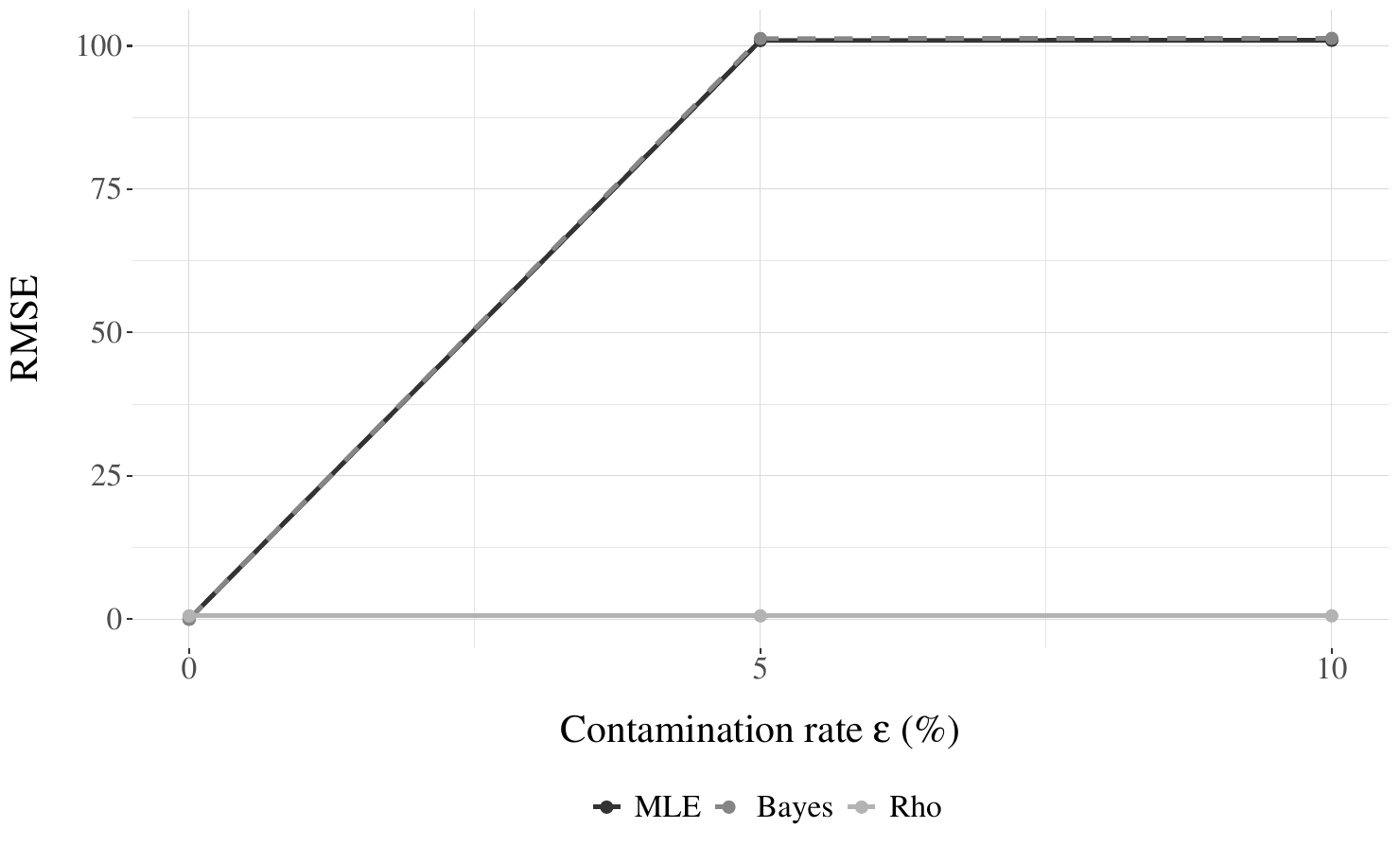}
\includegraphics[width=0.32\textwidth]{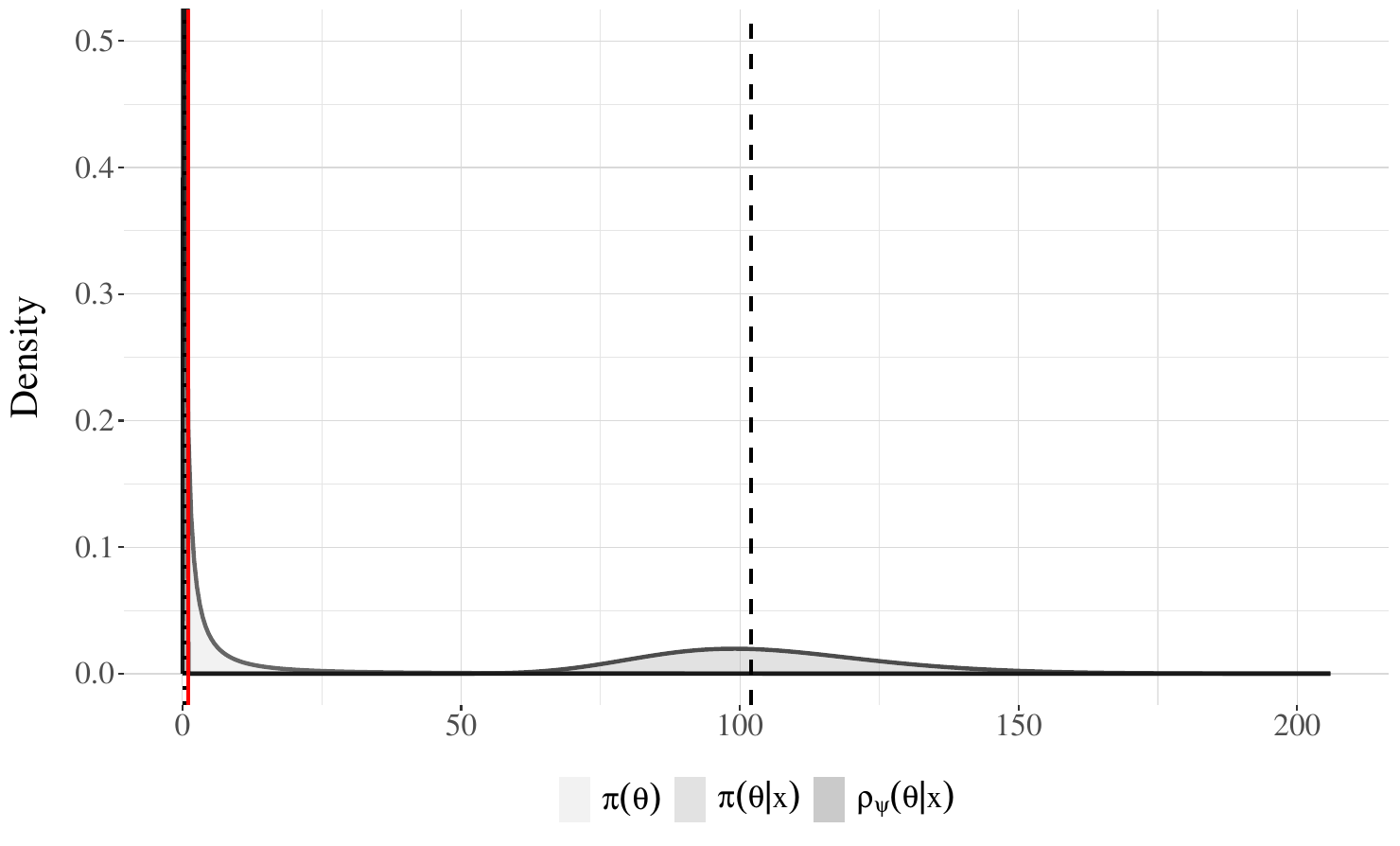}
\caption{Uniform scale model results for $n=200$ and $\tau=0.5$. \textit{Left:} Posterior risk vs.\ contamination rate $\varepsilon$. \textit{Middle:} RMSE vs.\ contamination rate. \textit{Right:} Posterior densities for a single dataset at $\varepsilon = 10\%$.}
\label{fig:uniform_results}
\end{figure}

\medskip
 
We next evaluate the $\tilde{\rho}$-posterior on three increasingly complex regression settings, each designed to test different aspects of robustness under heavy-tailed contamination.

\paragraph{Fourier basis regression with smooth target function.}
We first consider a regression problem with known smooth structure. The model takes the form $Y_i = f^*(w_i) + \xi_i$ for $i=1,\ldots,n$, where $w_i$ are equally spaced points in $[0,1]$ and the true regression function is
\[
f^*(w) = \sin(2\pi w) + 0.3\cos(6\pi w).
\]
The design matrix $\Phi \in \mathbb{R}^{n \times p}$ consists of Fourier basis functions with $K=6$ frequency components, yielding $p = 13$ features including the intercept. The noise follows an $\varepsilon$-contaminated mixture
\[
\xi_i \sim (1-\varepsilon)\mathcal{N}(0,1) + \varepsilon \cdot \text{Pareto}_{\text{two-sided}}(6, 2),
\]
where the two-sided Pareto generates symmetric outliers with minimum magnitude 6. We set $n=200$ and vary $\varepsilon \in \{0, 0.05, 0.08, 0.10\}$.

For this setting, we compared three estimators: MLE (ordinary least squares), standard Bayes posterior mean with conjugate Gaussian prior $\beta \sim \mathcal{N}(0, 4I_p)$, and the $\tilde{\rho}$-posterior approximated via variational inference with temperature $\lambda = 0.5n$. 

Figure~\ref{fig:fourier_results} reports results averaged over 1000 replications. Both MLE and Bayes posterior risk grow from near zero to over $2$ at $\varepsilon=10\%$; the $\tilde{\rho}$-posterior risk increases only from $0.07$ to $0.14$. RMSE exhibits the same pattern.

\begin{figure}[!ht]
\centering
\includegraphics[width=0.48\textwidth]{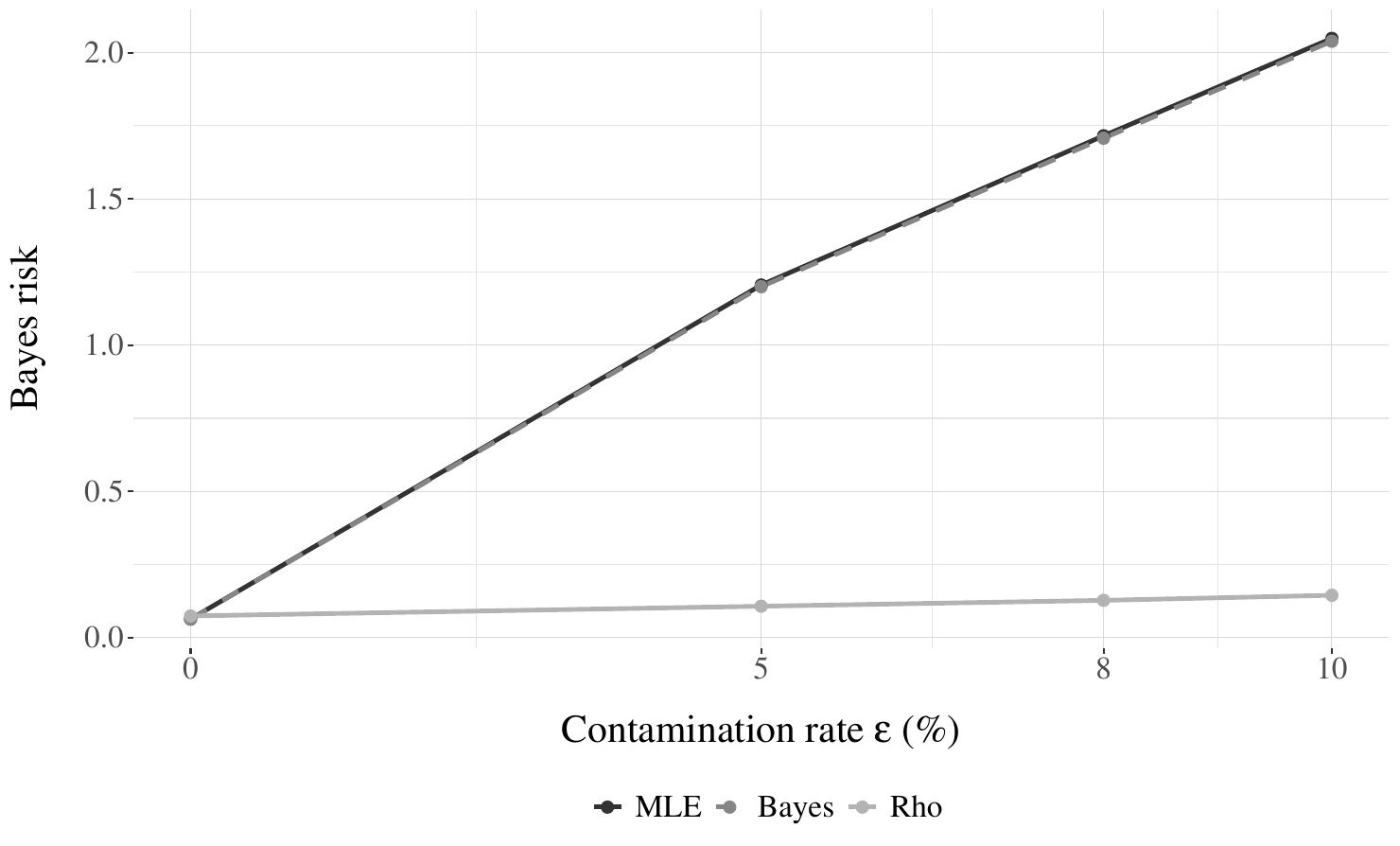}
\includegraphics[width=0.48\textwidth]{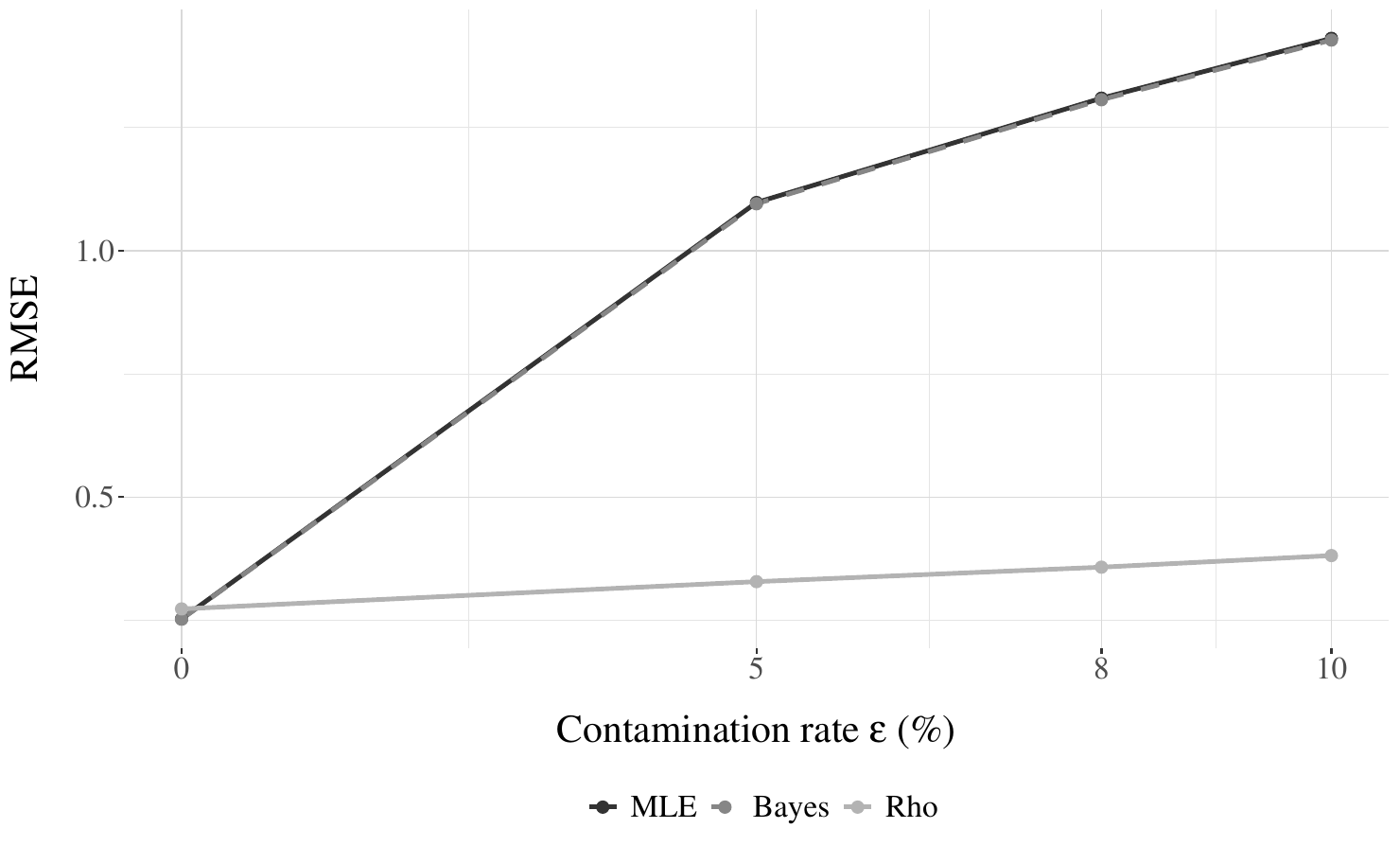}
\caption{Fourier basis regression ($n=200$, $K=6$). \textit{Left:} Posterior risk. \textit{Right:} RMSE. The $\tilde{\rho}$-posterior maintains stability under contamination.}
\label{fig:fourier_results}
\end{figure}

\paragraph{Correlated design with sparse parameters.}
The second setting uses $n=100$ observations, $d=10$ correlated features with Toeplitz covariance $\Sigma_{jk} = 0.7^{|j-k|}$, and a sparse true parameter $\beta^*$ with five nonzero entries. Contamination uses two-sided Pareto noise with minimum magnitude $10$ and shape $1.5$. Figure~\ref{fig:correlated_results} shows OLS risk growing to over $150$ at $\varepsilon=10\%$, while the $\tilde{\rho}$-posterior remains below $4$.

\begin{figure}[!ht]
\centering
\includegraphics[width=0.32\textwidth]{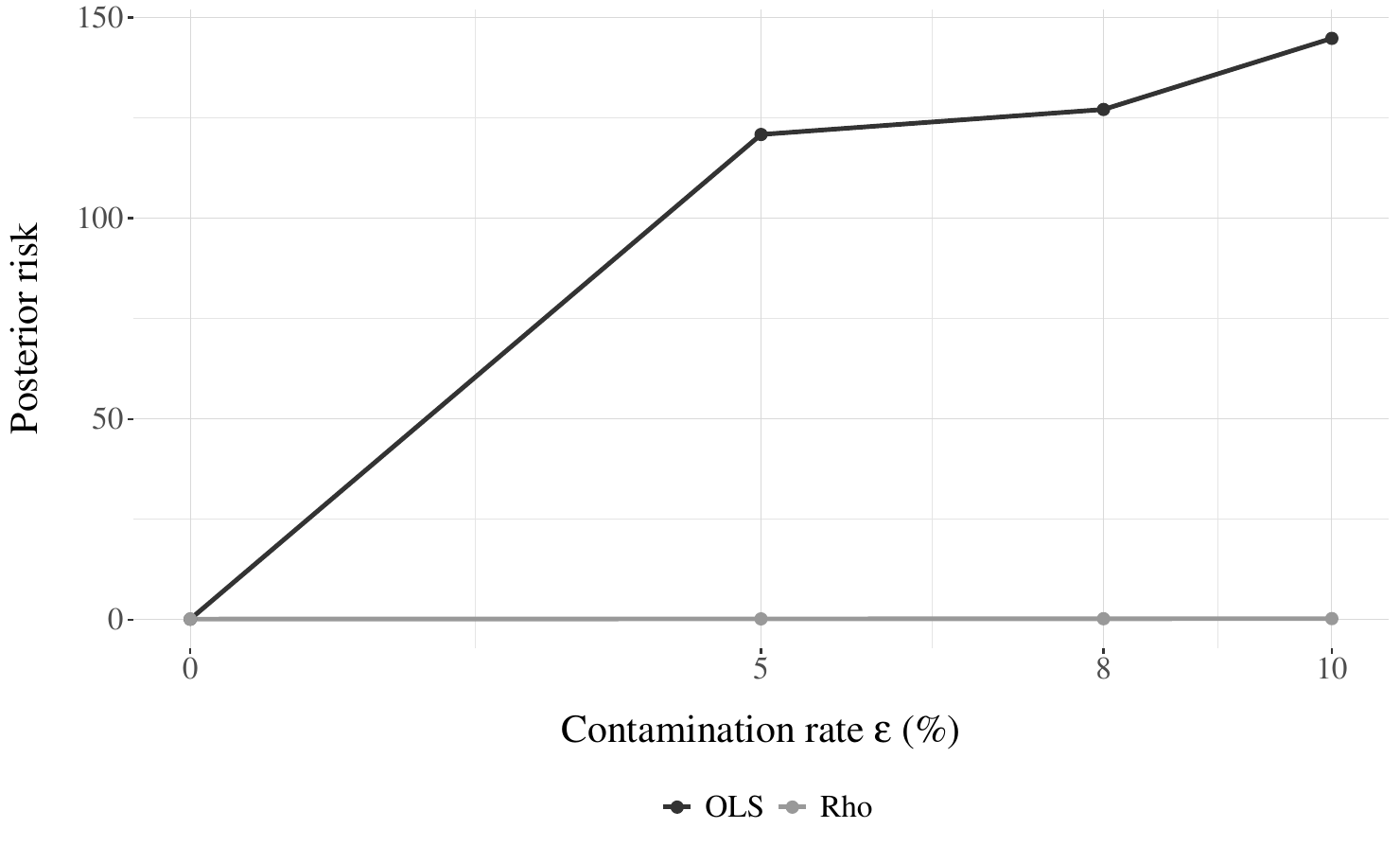}
\includegraphics[width=0.32\textwidth]{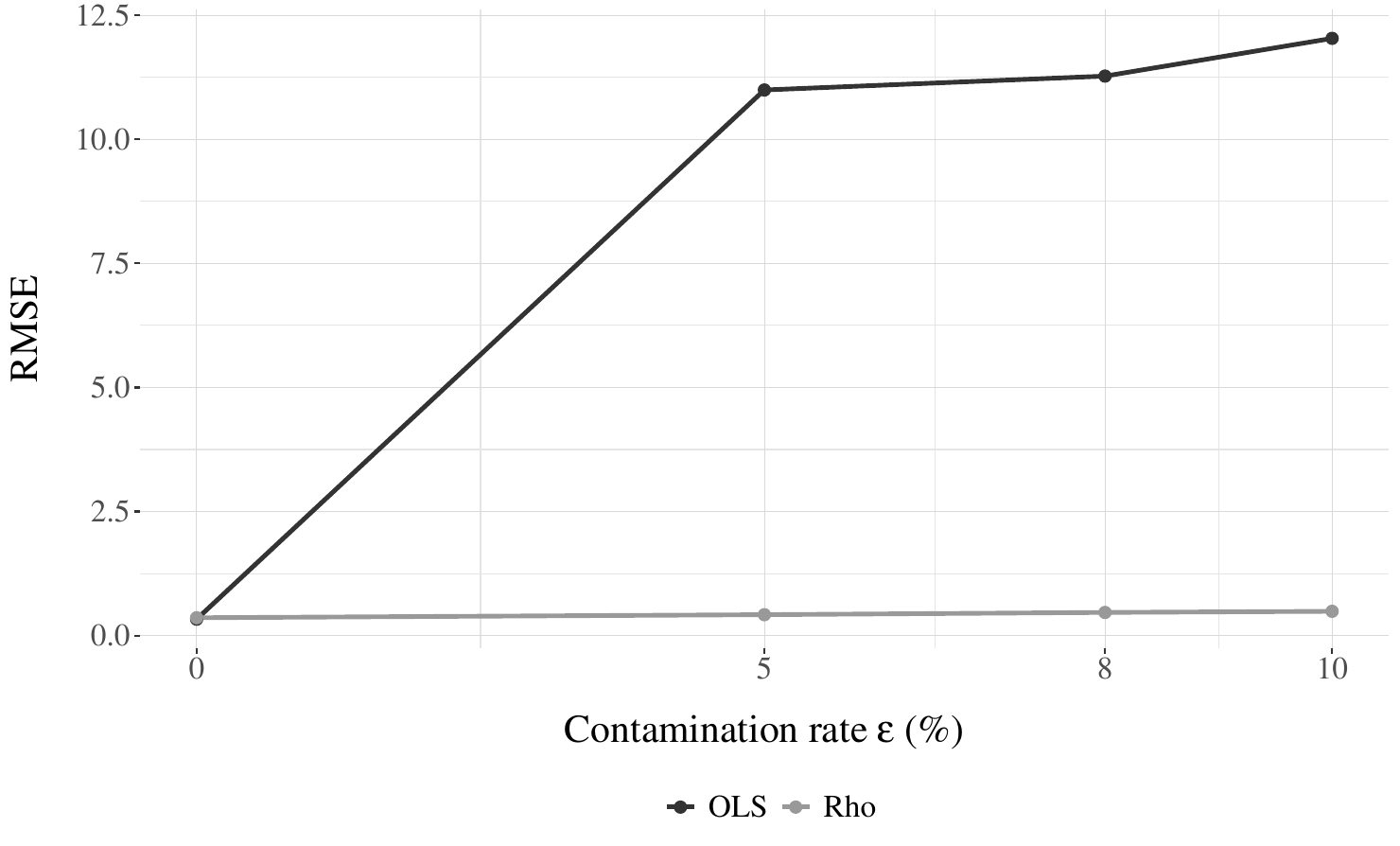}
\includegraphics[width=0.32\textwidth]{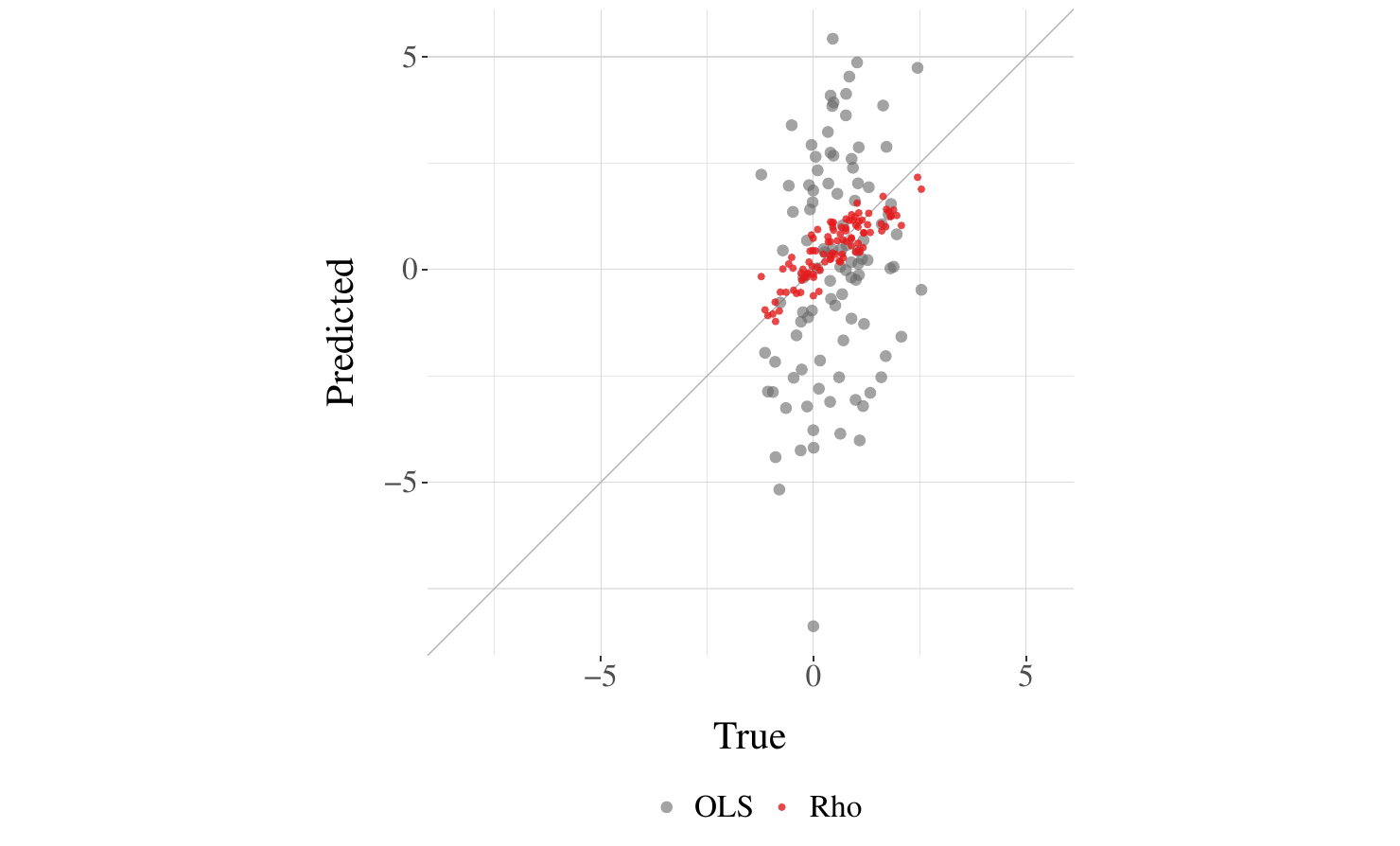}
\caption{Correlated design regression ($n=100$, $p=11$). \textit{Left:} Posterior risk. \textit{Middle:} RMSE. \textit{Right:} Predicted vs.\ true values at $\varepsilon=10\%$.}
\label{fig:correlated_results}
\end{figure}

\paragraph{Real-world datasets.}
We evaluate on the Ames Housing dataset \citep{de2011ames} ($n=2{,}930$, $79$ features) and the Abalone dataset \citep{nash1994population} ($n=4{,}177$, $8$ features). Training labels are contaminated by adding $\pm 15 \times \mathrm{MAD}(Y)$ to a random $\varepsilon$-fraction; test data remains clean. We compare OLS, Huber regression \citep{hammouda2024outlier}, and the variational $\tilde{\rho}$-posterior. Figure~\ref{fig:realworld_results} shows test residual distributions at $\varepsilon=10\%$ over 1000 trials: the $\tilde{\rho}$-posterior achieves concentration competitive with Huber regression.

\begin{figure}[!ht]
\centering
\includegraphics[width=0.48\textwidth]{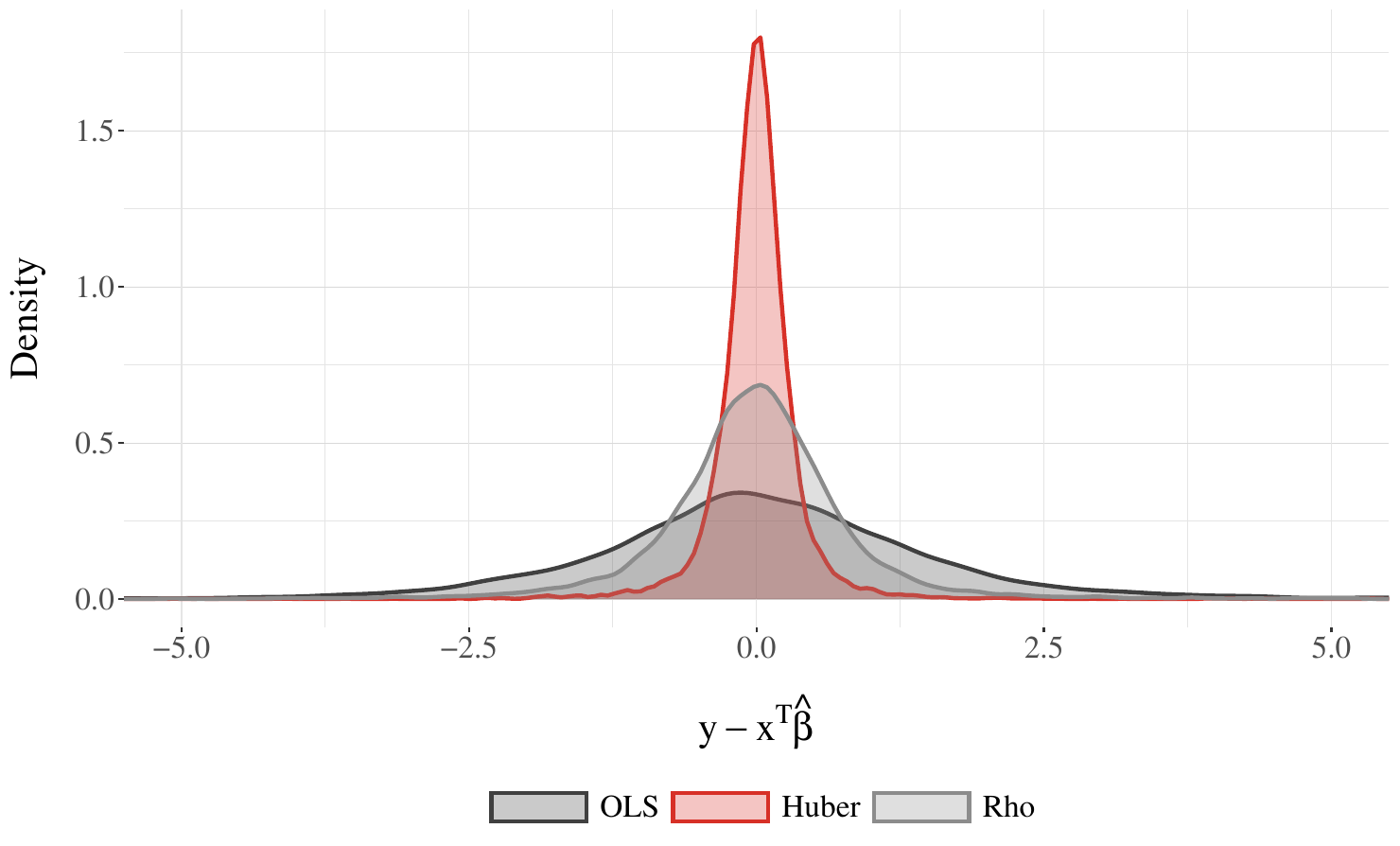}
\includegraphics[width=0.48\textwidth]{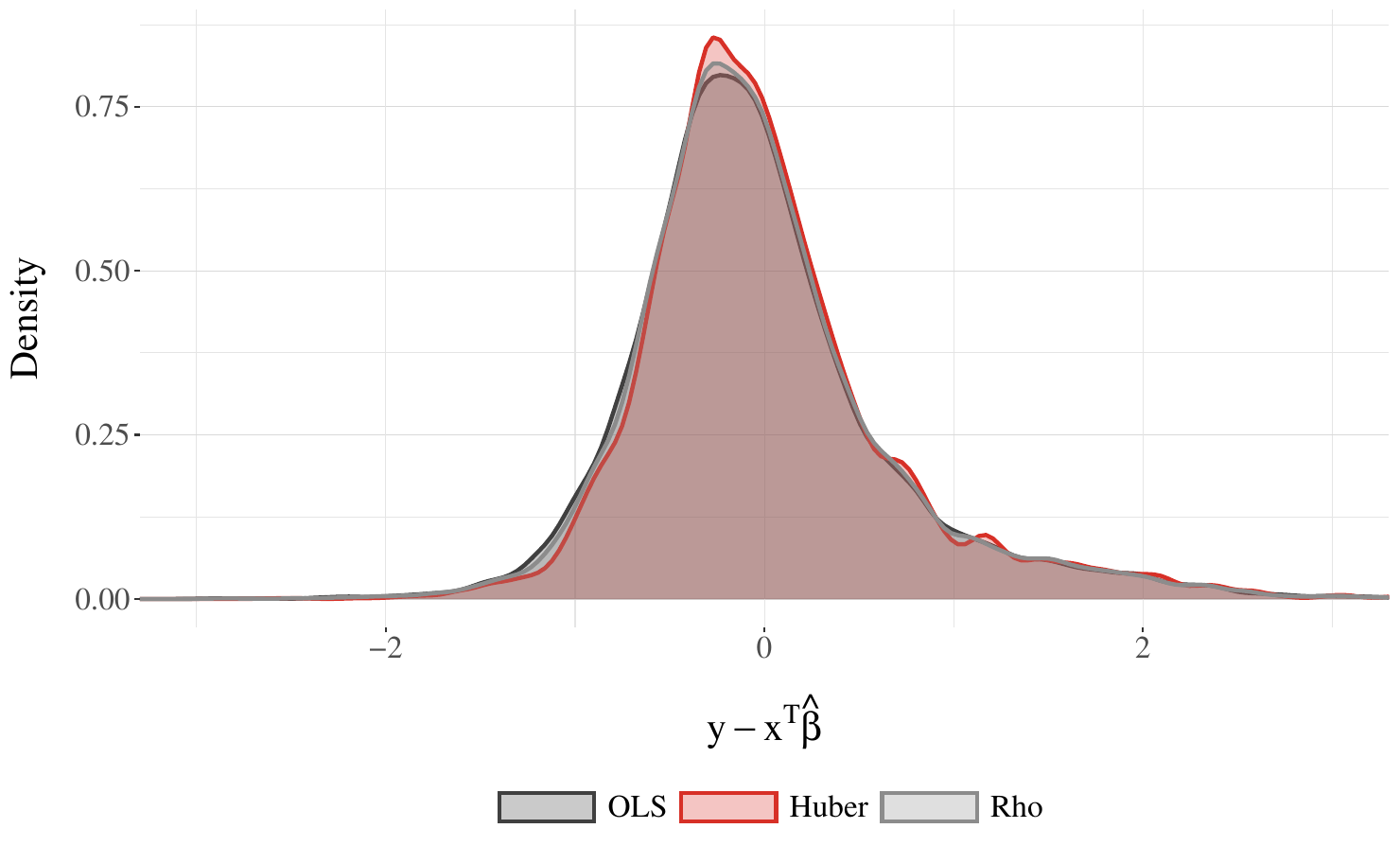}
\caption{Test residual distributions at $\varepsilon=10\%$ contamination. \textit{Left:} Ames Housing. \textit{Right:} Abalone.}
\label{fig:realworld_results}
\end{figure}

These experiments confirm that the robustness established in Theorems~\ref{thm:pac-bayes-independent}--\ref{thm:variational-oracle} extends to practical settings with varying dimensionality and contamination patterns.

\section{Conclusion}\label{sec:conclusion}

We introduced the $\tilde{\rho}$-posterior, a softmax relaxation of the $\rho$-posterior of \citet{baraud2020robust}, and developed a complete PAC-Bayesian framework for its analysis and computation. By replacing the intractable supremum contrast with a softmax aggregation over competitor parameters, we obtained a Gibbs posterior amenable to the PAC-Bayesian machinery, while preserving the competitor-based risk decomposition that is the hallmark of the approach of \citet{baraud2017new,baraud2020robust}. Finite-sample oracle inequalities (Theorems~\ref{thm:pac-bayes-independent} and~\ref{thm:oracle}) establish concentration rates for the $\tilde{\rho}$-posterior under Hellinger risk. These rates are inherited by variational approximations (Theorem~\ref{thm:variational-oracle}), and the resulting saddle-point optimization can be solved efficiently via first-order methods (Theorem~\ref{thm:nc-sc-main}).

Several directions remain open. The softmax principle that underlies our construction applies to any bounded contrast function $\psi$. The specific choice used here is motivated both by its connection to Hellinger distance and the original $\rho$-estimator, and by the favorable optimization landscape it induces: the associated saddle-point problem enjoys a nonconvex-strongly concave structure that guarantees convergence of standard first-order methods. Other bounded contrasts, potentially connected to other divergences between probability measures, may offer different trade-offs between robustness guarantees, concentration rates, and computational properties. Whether alternative bounded contrasts can be identified that simultaneously yield strong robustness to contamination, sharp concentration, and favorable optimization geometry is an interesting question for future work.

\subsection*{Acknowledgements}
We thank Yannick Baraud and Gabriel Romon for carefully reading an earlier version of this manuscript and for bringing several errors to our attention. We are grateful for their valuable feedback. Any remaining errors are solely our responsibility.

\bibliographystyle{agsm}
\bibliography{refs}

\appendix
\section{Proofs}\label{sec:proofs}

We collect technical preliminaries in Section~\ref{sec:prelim} and prove Theorem~\ref{thm:pac-bayes-independent} in Section~\ref{sec:proof-main}.

\subsection{Technical Preliminaries}\label{sec:prelim}

\begin{lem}[Bernstein inequality for independent bounded random variables]\label{lem:bernstein}
Let $Y_1,\ldots,Y_n$ be independent real-valued random variables with $\mathbb{E}[Y_i]=0$ for all $i$, and suppose $|Y_i|\leq C$ almost surely for some $C>0$. Define the total variance $\sigma^2 := \sum_{i=1}^n \mathbb{V}(Y_i)$ and the function
\begin{equation}\label{eq:g-function}
g(x) := \begin{cases}
\displaystyle \frac{e^x-1-x}{x^2} & x\neq 0\\
\displaystyle \frac{1}{2} & x=0.
\end{cases}
\end{equation}
Then for all $t\in\mathbb{R}$,
\begin{equation}\label{eq:bernstein-bound}
\log \mathbb{E}\left[\exp\left(t\sum_{i=1}^n Y_i\right)\right]
\leq \sigma^2 t^2 g(C|t|).
\end{equation}
\end{lem}

\begin{proof}
See \citet{boucheron2013concentration}. The moment generating function factorizes by independence, and each factor is bounded via Bernstein's lemma for a single bounded variable.
\end{proof}

\begin{lem}[Donsker-Varadhan variational formula, \citet{donsker1975asymptotic}]\label{lem:dv}
Let $\pi$ be a probability measure on $\Theta$ and $h:\Theta\to\mathbb{R}$ measurable with $\int e^{h(\theta)}\pi(d\theta)<\infty$. Then
\[
\log \left(\int e^{h(\theta)}\pi(d\theta)\right)
= \sup_{\rho\in\mathcal{P}(\Theta)}
\left\{ \int h(\theta)\rho(d\theta) - \mathrm{KL}(\rho\Vert\pi) \right\}.
\]
If $h$ is bounded above on $\mathrm{supp}(\pi)$, the supremum is attained at the Gibbs measure
\[
\pi^h(d\theta)=\frac{e^{h(\theta)}\pi(d\theta)}{\int e^{h(\vartheta)}\pi(d\vartheta)}.
\]
\end{lem}
We recall the following lemma.

\begin{lem}[Hellinger comparison for the $\psi$-contrast]\label{lem:hellinger-comparison}
For the contrast function $\psi(x)=(\sqrt{x}-1)/(\sqrt{x}+1)$, there exist universal constants
$a_0=4$, $a_1=3/8$, $a_2^2=3\sqrt{2}$ such that for all $(\theta,\theta')\in\Theta^2$,
\begin{align}
a_1\mathcal{H}_n^2(P_\star^{(n)},P_{\theta}^{(n)}) - a_0\mathcal{H}_n^2(P_\star^{(n)},P_{\theta'}^{(n)})
&\leq R_{\psi}(\theta,\theta')
\leq a_0\mathcal{H}_n^2(P_\star^{(n)},P_{\theta}^{(n)}) - a_1\mathcal{H}_n^2(P_\star^{(n)},P_{\theta'}^{(n)}), \label{eq:risk-bounds1}\\
V_{\psi}(\theta,\theta')
&\leq \frac{a_2^2}{n}\left(\mathcal{H}_n^2(P_\star^{(n)},P_{\theta}^{(n)}) + \mathcal{H}_n^2(P_\star^{(n)},P_{\theta'}^{(n)})\right). \label{eq:variance-bound}
\end{align}
\end{lem}

\begin{proof}
This follows from Proposition~3 in~\citet{baraud2018rho} applied coordinate-wise.
For each $i \in [n]$, define
\[
R_\psi^{(i)}(\theta,\theta')
:= \mathbb{E}_{X_i \sim P_\star^i}\!\left[\ell_\psi(X_i;\theta,\theta')\right],
\qquad
V_\psi^{(i)}(\theta,\theta')
:= \mathbb{V}_{X_i \sim P_\star^i}\!\left[\ell_\psi(X_i;\theta,\theta')\right].
\]
By Proposition~3 in~\citet{baraud2018rho}, for each coordinate $i$,
\begin{align*}
a_1 \mathcal{H}^2(P_\star^i, P_\theta^i) - a_0 \mathcal{H}^2(P_\star^i, P_{\theta'}^i)
&\leq R_\psi^{(i)}(\theta,\theta')
\leq a_0 \mathcal{H}^2(P_\star^i, P_\theta^i) - a_1 \mathcal{H}^2(P_\star^i, P_{\theta'}^i),\\
V_\psi^{(i)}(\theta,\theta')
&\leq a_2^2\left(\mathcal{H}^2(P_\star^i, P_\theta^i) + \mathcal{H}^2(P_\star^i, P_{\theta'}^i)\right).
\end{align*}

For the population contrast, recall that
\[
R_\psi(\theta,\theta')
=\frac{1}{n}\sum_{i=1}^n R_\psi^{(i)}(\theta,\theta').
\]
Averaging the coordinate-wise risk bounds over $i\in[n]$ yields~\eqref{eq:risk-bounds},
using the definition
\[
\mathcal{H}_n^2(P_\star^{(n)},P_\theta^{(n)})
:=\frac{1}{n}\sum_{i=1}^n \mathcal{H}^2(P_\star^i,P_\theta^i).
\]

For the variance, using independence of $(X_1,\dots,X_n)$ and the definition
$V_\psi(\theta,\theta')=\mathbb{V}_{\mathcal S}[\hat R_\psi(\theta,\theta')]$, we have
\[
V_\psi(\theta,\theta')
=\mathbb{V}\!\left[\frac{1}{n}\sum_{i=1}^n \ell_\psi(X_i;\theta,\theta')\right]
=\frac{1}{n^2}\sum_{i=1}^n V_\psi^{(i)}(\theta,\theta').
\]
Therefore, applying the coordinate-wise variance bound and averaging gives
\begin{align*}
V_\psi(\theta,\theta')
&\leq \frac{1}{n^2}\sum_{i=1}^n
a_2^2\left(\mathcal{H}^2(P_\star^i, P_\theta^i) + \mathcal{H}^2(P_\star^i, P_{\theta'}^i)\right)\\
&= \frac{a_2^2}{n}\left(\mathcal{H}_n^2(P_\star^{(n)},P_{\theta}^{(n)})
+ \mathcal{H}_n^2(P_\star^{(n)},P_{\theta'}^{(n)})\right),
\end{align*}
which is exactly~\eqref{eq:variance-bound}.
\end{proof}


\subsection{Proof of Theorem~\ref{thm:pac-bayes-independent}}\label{sec:proof-main}

Fix $(\theta,\theta')\in\Theta^2$. For each $i \in \{1,\ldots,n\}$, define the random variable
\begin{equation}\label{eq:U-def}
U_i(\theta,\theta') := \ell_\psi(X_i;\theta,\theta') = \psi\left(\frac{p_{\theta'}^i(X_i)}{p_\theta^i(X_i)}\right) \in [-1,1],
\end{equation}
and its centered version
\begin{equation}\label{eq:Y-def}
Y_i(\theta,\theta') := U_i(\theta,\theta') - \mathbb{E}_{X_i \sim P_\star^i}[U_i(\theta,\theta')].
\end{equation}

Since $|U_i| \leq 1$, we have $|Y_i| \leq 2$ almost surely. The variables $\{Y_i\}_{i=1}^n$ are independent with $\mathbb{E}[Y_i] = 0$, and
\[
\sum_{i=1}^n Y_i(\theta,\theta')
=n\big(\hat R_\psi(\theta,\theta')-R_\psi(\theta,\theta')\big).
\]
Moreover, by independence,
\[
\mathbb{V}\!\left[\sum_{i=1}^n Y_i(\theta,\theta')\right]
=\sum_{i=1}^n\mathbb{V}\!\left[Y_i(\theta,\theta')\right]
=n^2\,V_\psi(\theta,\theta'),
\]
where the last equality uses the definition~\eqref{eq:def-variance}. Note that the individual variances $\mathbb{V}_{X_i \sim P_\star^i}[\ell_\psi(X_i;\theta,\theta')]$ may differ across coordinates in the general independent case; we are simply summing them. In the i.i.d. case, this reduces to $n$ times a common variance.
\medskip

Applying Lemma~\ref{lem:bernstein} with $C=2$ and $t = \lambda/n$,

\begin{equation}\label{eq:bernstein-forward}
\log\mathbb{E}_{\mathcal{S}}\exp\!\left(\lambda(\hat R_\psi-R_\psi)\right)
\le
\lambda^2 g\!\left(\frac{2\lambda}{n}\right)V_\psi(\theta,\theta').
\end{equation}

\medskip

Applying the same bound to $-Y_i$,
\begin{equation}\label{eq:bernstein-reverse}
\log\mathbb{E}_{\mathcal{S}}\exp\!\left(\lambda(R_\psi-\hat R_\psi)\right)
\le
\lambda^2 g\!\left(\frac{2\lambda}{n}\right)V_\psi(\theta,\theta').
\end{equation}
Equivalently,
\begin{equation}\label{eq:fixed-mgf}
\mathbb{E}_{\mathcal{S}}\exp\!\left(
\lambda(R_\psi-\hat R_\psi)-\lambda^2 g\!\left(\frac{2\lambda}{n}\right)V_\psi
\right)\le 1.
\end{equation}

\medskip

Integrating~\eqref{eq:fixed-mgf} over $(\theta,\theta')\sim \pi\otimes\pi'$ and applying Markov's inequality, with probability at least $1-\delta$,
\[
\int \exp\!\left(
\lambda(R_\psi-\hat R_\psi)-\lambda^2 g\!\left(\frac{2\lambda}{n}\right)V_\psi
\right)\,d(\pi\otimes\pi')
\le \frac1\delta.
\]
Taking logarithms and applying Lemma~\ref{lem:dv} with $h(\theta,\theta'):=\lambda(R_\psi-\hat{R}_\psi)-\lambda^2 g(2\lambda/n)V_\psi$ and prior $\pi\otimes\pi'$, for any $\rho\otimes\rho'$ with $\rho\ll\pi$ and $\rho'\ll\pi'$,
\begin{equation}\label{eq:dv-applied}
\int h(\theta,\theta')\,\mu(d\theta,d\theta') - \mathrm{KL}(\mu\Vert\pi\otimes\pi') \leq \log(1/\delta).
\end{equation}

The KL divergence decomposes as $\mathrm{KL}(\rho\otimes\rho'\Vert\pi\otimes\pi') = \mathrm{KL}(\rho\Vert\pi) + \mathrm{KL}(\rho'\Vert\pi')$. Substituting into~\eqref{eq:dv-applied} and dividing by $\lambda$,
\begin{align}
\mathbb{E}_{\rho\otimes\rho'}\!\left[R_\psi(\theta,\theta')\right]
-\lambda g\!\left(\frac{2\lambda}{n}\right)\mathbb{E}_{\rho\otimes\rho'}\!\left[V_\psi(\theta,\theta')\right]
&\le
\mathbb{E}_{\rho\otimes\rho'}\!\left[\hat R_\psi(\theta,\theta')\right]
+\frac{1}{\lambda}\mathrm{KL}(\rho\Vert\pi)
+\frac{1}{\lambda}\mathrm{KL}(\rho'\Vert\pi')\nonumber\\
&\quad+\frac{\log(1/\delta)}{\lambda}.
\label{eq:pre-variance}
\end{align}

\medskip
By Lemma~\ref{lem:hellinger-comparison},
\[
V_\psi(\theta,\theta')
\le
\frac{a_2^2}{n}\Big(\mathcal H_n^2(P_\star^{(n)},P_\theta^{(n)})
+\mathcal H_n^2(P_\star^{(n)},P_{\theta'}^{(n)})\Big).
\]
Hence
\[
\lambda g\!\left(\frac{2\lambda}{n}\right)V_\psi(\theta,\theta')
\le
\beta_{n,\lambda}a_2^2\Big(\mathcal H_n^2(P_\star^{(n)},P_\theta^{(n)})
+\mathcal H_n^2(P_\star^{(n)},P_{\theta'}^{(n)})\Big),
\]
with $\beta_{n,\lambda}=g(2\lambda/n)\lambda/n$.
Substituting this into~\eqref{eq:pre-variance} yields that, with probability at least $1-\delta$,
\begin{align}
\mathbb{E}_{\rho\otimes\rho'}[R_\psi]
&\le
\mathbb{E}_{\rho\otimes\rho'}[\hat R_\psi]
+\beta_{n,\lambda}a_2^2
\left(
\mathbb{E}_{\rho}[\mathcal H_n^2(P_\star^{(n)},P_\theta^{(n)})]
+\mathbb{E}_{\rho'}[\mathcal H_n^2(P_\star^{(n)},P_{\theta'}^{(n)})]
\right)\nonumber\\
&\quad+\frac{1}{\lambda}\mathrm{KL}(\rho\Vert\pi)
+\frac{1}{\lambda}\mathrm{KL}(\rho'\Vert\pi')
+\frac{\log(1/\delta)}{\lambda}.
\label{eq:after-variance}
\end{align}

\medskip

By Lemma~\ref{lem:hellinger-comparison},
\[
R_\psi(\theta,\theta')
\ge
a_1\mathcal H_n^2(P_\star^{(n)},P_\theta^{(n)})
-a_0\mathcal H_n^2(P_\star^{(n)},P_{\theta'}^{(n)}).
\]
Taking $\rho\otimes\rho'$-expectations and combining with~\eqref{eq:after-variance} gives
\begin{align}
(a_1-\beta_{n,\lambda}a_2^2)\,
\mathbb{E}_{\rho}\!\left[\mathcal H_n^2(P_\star^{(n)},P_\theta^{(n)})\right]
&\le
\mathbb{E}_{\rho\otimes\rho'}[\hat R_\psi(\theta,\theta')]
+\big(a_0+\beta_{n,\lambda}a_2^2\big)\,
\mathbb{E}_{\rho'}\!\left[\mathcal H_n^2(P_\star^{(n)},P_{\theta'}^{(n)})\right]\nonumber\\
&\quad+\frac{1}{\lambda}\mathrm{KL}(\rho\Vert\pi)
+\frac{1}{\lambda}\mathrm{KL}(\rho'\Vert\pi')
+\frac{\log(1/\delta)}{\lambda}.
\label{eq:hell-stage}
\end{align}

\medskip
Fix $\theta \in \Theta$ and apply the Donsker-Varadhan formula (Lemma~\ref{lem:dv}) with $h(\theta')=\lambda\hat{R}_\psi(\theta,\theta')$ and prior $\pi'$. For any $\rho'\ll\pi'$,
\begin{equation}\label{eq:dv-inner}
\mathbb{E}_{\theta'\sim\rho'}[\hat{R}_\psi(\theta,\theta')] - \frac{1}{\lambda}\mathrm{KL}(\rho'\Vert\pi')
\leq \frac{1}{\lambda}\log\left(\int_\Theta e^{\lambda\hat{R}_\psi(\theta,\theta')}\pi'(d\theta')\right)
=: \Lambda_\lambda(\theta;\pi').
\end{equation}

Rearranging~\eqref{eq:dv-inner},
\begin{equation}\label{eq:softmax-bound}
\mathbb{E}_{\theta'\sim\rho'}[\hat{R}_\psi(\theta,\theta')] 
\leq \Lambda_\lambda(\theta;\pi') + \frac{1}{\lambda}\mathrm{KL}(\rho'\Vert\pi').
\end{equation}

Averaging over $\theta\sim\rho$ yields
\[
\mathbb{E}_{\rho\otimes\rho'}[\hat R_\psi(\theta,\theta')]
\le
\mathbb{E}_{\theta\sim\rho}[\Lambda_\lambda(\theta;\pi')]
+\frac{1}{\lambda}\mathrm{KL}(\rho'\Vert\pi').
\]
Plugging into~\eqref{eq:hell-stage} gives
\begin{align}
(a_1-\beta_{n,\lambda}a_2^2)\,
\mathbb{E}_{\rho}\!\left[\mathcal H_n^2(P_\star^{(n)},P_\theta^{(n)})\right]
&\le
\mathbb{E}_{\theta\sim\rho}[\Lambda_\lambda(\theta;\pi')]
+\big(a_0+\beta_{n,\lambda}a_2^2\big)\,
\mathbb{E}_{\rho'}\!\left[\mathcal H_n^2(P_\star^{(n)},P_{\theta'}^{(n)})\right]\nonumber\\
&\quad+\frac{1}{\lambda}\mathrm{KL}(\rho\Vert\pi)
+\frac{2}{\lambda}\mathrm{KL}(\rho'\Vert\pi')
+\frac{\log(1/\delta)}{\lambda}.
\label{eq:final-pre-inf}
\end{align}
Taking the infimum over $\rho'\in\mathcal P(\Theta)$ and evaluating at $\rho=\hat\rho_\lambda$, the minimizer of $\rho\mapsto\mathbb{E}_{\rho}[\Lambda_\lambda(\theta;\pi')]+\lambda^{-1}\mathrm{KL}(\rho\Vert\pi)$, yields~\eqref{eq:main-pac-bayes}. \qed

\subsection{Proof of Corollary~\ref{cor:explicit-lambda-independent}}

\begin{proof}
Set $\lambda = n/8$. Then
\[
\beta_{n,\lambda} = g\left(\frac{2\lambda}{n}\right)\frac{\lambda}{n} = \frac{1}{8} g\left(\frac{1}{4}\right),
\]
where $g(x) = (e^x - 1 - x)/x^2$. Evaluating,
\[
g\left(\frac{1}{4}\right) = 16\left(e^{1/4} - \frac{5}{4}\right) \approx 0.5444,
\]
hence $\beta_{n,\lambda}\approx 0.068$.

For the leading coefficient,
\[
a_1-\beta_{n,\lambda}a_2^2
=\frac{3}{8}-\beta_{n,\lambda}\cdot 3\sqrt2
\approx 0.375-0.068\times 4.242
\approx 0.086
\;\ge\;\frac{1}{12}.
\]
For the competitor coefficient,
\[
a_0+\beta_{n,\lambda}a_2^2
=4+\beta_{n,\lambda}\cdot 3\sqrt2
\approx 4+0.289
\approx 4.289
\;\le\;\frac{13}{3}.
\]

Substituting these bounds into Theorem~\ref{thm:pac-bayes-independent} yields \eqref{eq:explicit-rate-ind}.

\end{proof}




\subsection{Proof of Theorem~\ref{thm:oracle}}

The proof combines Theorem~\ref{thm:pac-bayes-independent} with an upper bound on $\Lambda_\lambda$ that holds with high probability.
\begin{lem}[Upper bound on $\Lambda_\lambda$]\label{lem:lambda-hellinger}
Fix $\delta\in(0,1)$ and $\lambda>0$ such that $\beta_{n,\lambda}<a_1/a_2^2$.
Then for any fixed $\rho\in\mathcal{P}(\Theta)$, with probability at least $1-\delta$,
\begin{equation}\label{eq:lambda-upper-correct}
\mathbb{E}_{\theta\sim\rho}\big[\Lambda_\lambda(\theta;\pi')\big]
\le
(a_0+\beta_{n,\lambda}a_2^2)\,
\mathbb{E}_{\theta\sim\rho}\!\left[\mathcal H_n^2(P_\star^{(n)},P_\theta^{(n)})\right]
+\frac{\log(1/\delta)}{\lambda}.
\end{equation}
\end{lem}

\begin{proof}
Fix $\theta\in\Theta$. Apply Bernstein's MGF bound exactly as in the proof of
Theorem~\ref{thm:pac-bayes-independent} with $t=\lambda/n$, to obtain for each $\theta'\in\Theta$,
\[
\mathbb{E}_{\mathcal S}\exp\!\left(
\lambda(\hat R_\psi(\theta,\theta')-R_\psi(\theta,\theta'))
-\lambda^2 g\!\left(\frac{2\lambda}{n}\right)V_\psi(\theta,\theta')
\right)\le 1.
\]
Integrating over $\theta'\sim\pi'$ and applying Markov's inequality, with probability at least $1-\delta$,
\begin{equation}\label{eq:int-markov-Lambda-correct}
\int_{\Theta}\exp\!\left(
\lambda(\hat R_\psi(\theta,\theta')-R_\psi(\theta,\theta'))
-\lambda^2 g\!\left(\frac{2\lambda}{n}\right)V_\psi(\theta,\theta')
\right)\pi'(d\theta')\le \frac{1}{\delta}.
\end{equation}
Taking $\log$ and applying Lemma~\ref{lem:dv} yields that for any $\rho'\ll\pi'$,
\[
\mathbb{E}_{\theta'\sim\rho'}[\hat R_\psi(\theta,\theta')]
\le
\mathbb{E}_{\theta'\sim\rho'}\!\left[
R_\psi(\theta,\theta')
+\lambda g\!\left(\frac{2\lambda}{n}\right)V_\psi(\theta,\theta')
\right]
+\frac{\mathrm{KL}(\rho'\Vert\pi')+\log(1/\delta)}{\lambda}.
\]
Now choose $\rho'=\rho'_\lambda:=\rho'_\lambda(\theta)$, the Gibbs distribution proportional to
$\exp(\lambda\hat R_\psi(\theta,\theta'))\pi'(d\theta')$. Then
\[
\Lambda_\lambda(\theta;\pi')
=
\mathbb{E}_{\theta'\sim\rho'_\lambda}[\hat R_\psi(\theta,\theta')]
-\frac{1}{\lambda}\mathrm{KL}(\rho'_\lambda\Vert\pi').
\]
Substituting gives
\[
\Lambda_\lambda(\theta;\pi')
\le
\mathbb{E}_{\theta'\sim\rho'_\lambda}\!\left[
R_\psi(\theta,\theta')
+\lambda g\!\left(\frac{2\lambda}{n}\right)V_\psi(\theta,\theta')
\right]
+\frac{\log(1/\delta)}{\lambda}.
\]
Using the \emph{upper} risk bound and the variance bound from Lemma~\ref{lem:hellinger-comparison},
\[
R_\psi(\theta,\theta')
\le a_0\mathcal H_n^2(P_\star^{(n)},P_\theta^{(n)})
-a_1\mathcal H_n^2(P_\star^{(n)},P_{\theta'}^{(n)}),
\]
and
\[
\lambda g\!\left(\frac{2\lambda}{n}\right)V_\psi(\theta,\theta')
\le
\beta_{n,\lambda}a_2^2\Big(
\mathcal H_n^2(P_\star^{(n)},P_\theta^{(n)})
+\mathcal H_n^2(P_\star^{(n)},P_{\theta'}^{(n)})
\Big),
\]
we obtain
\begin{align*}
\Lambda_\lambda(\theta;\pi')
&\le
(a_0+\beta_{n,\lambda}a_2^2)\,\mathcal H_n^2(P_\star^{(n)},P_\theta^{(n)})
+(\beta_{n,\lambda}a_2^2-a_1)\,
\mathbb{E}_{\theta'\sim\rho'_\lambda}
\big[\mathcal H_n^2(P_\star^{(n)},P_{\theta'}^{(n)})\big]
+\frac{\log(1/\delta)}{\lambda}.
\end{align*}
Since $\beta_{n,\lambda}<a_1/a_2^2$, the coefficient $(\beta_{n,\lambda}a_2^2-a_1)$ is negative.
Dropping the corresponding (nonpositive) term yields
\[
\Lambda_\lambda(\theta;\pi')
\le
(a_0+\beta_{n,\lambda}a_2^2)\,\mathcal H_n^2(P_\star^{(n)},P_\theta^{(n)})
+\frac{\log(1/\delta)}{\lambda}.
\]
Taking expectation over $\theta\sim\rho$ gives \eqref{eq:lambda-upper-correct}.
\end{proof}

Now apply Lemma~\ref{lem:lambda-hellinger}. By Theorem~\ref{thm:pac-bayes-independent}, with probability at least $1-\delta$,
\begin{align*}
&\big(a_1-\beta_{n,\lambda}a_2^2\big)\,\mathbb{E}_{\theta\sim\hat\rho_\lambda}[\mathcal{H}^2(P^\star,P_\theta)]\\
&\leq \inf_{\rho\in\mathcal{P}(\Theta)}\left\{\mathbb{E}_{\theta\sim\rho}[\Lambda_{\lambda}(\theta;\pi')]+\frac{1}{\lambda}\mathrm{KL}(\rho\Vert\pi)\right\}\\
&\quad + \inf_{\rho'\in\mathcal{P}(\Theta)}\left\{(a_0+\beta_{n,\lambda}a_2^2)\mathbb{E}_{\theta'\sim\rho'}\left[\mathcal{H}^2(P^\star,P_{\theta'})\right] +\frac{2}{\lambda}\mathrm{KL}(\rho'\Vert\pi')\right\}
+\frac{\log(1/\delta)}{\lambda}.
\end{align*}
By Lemma~\ref{lem:lambda-hellinger}, with probability at least $1-\delta$, for any $\rho\in\mathcal{P}(\Theta)$,
\[
\mathbb{E}_{\theta\sim\rho}[\Lambda_{\lambda}(\theta;\pi')]
\leq (a_0+\beta_{n,\lambda}a_2^2)\mathbb{E}_{\theta\sim\rho}[\mathcal{H}^2(P^\star,P_\theta)]+\frac{\log(1/\delta)}{\lambda}.
\]
A union bound over the two events gives probability at least $1-2\delta$, yielding~\eqref{eq:oracle}.\qed
\subsection{Proofs for Regression Application}\label{sec:proofs-regression}
\begin{proof}[Proof of Lemma~\ref{lem:hellinger-to-regression}]
We adapt~\citet[Proposition~3]{baraud2020robust}. By~\eqref{eq:sample-hellinger},
\begin{equation}\label{eq:sample-hell-decomp}
\mathcal{H}_n^2(P_\star^{(n)}, P_f^{(n)}) = \frac{1}{n}\sum_{i=1}^n \mathcal{H}^2(p_\star^i, p_f^i).
\end{equation}
Since $w_i$ is deterministic, each coordinate Hellinger distance reduces to
\begin{equation}\label{eq:coord-hell-explicit}
\mathcal{H}^2(p_\star^i, p_f^i) = \mathcal{H}^2(p_{f^\star(w_i)}, q_{f(w_i)}),
\end{equation}
where $p_\delta(\cdot) = p(\cdot - \delta)$ and $q_\delta(\cdot) = q(\cdot - \delta)$. By translation invariance, $\mathcal{H}^2(p_{f^\star(w_i)}, q_{f^\star(w_i)}) = \mathcal{H}^2(p, q)$.

\medskip
\noindent\textbf{Upper bound.}
The triangle inequality for Hellinger distance gives
\begin{equation}\label{eq:hell-triangle}
\mathcal{H}^2(u,w) \leq 2\mathcal{H}^2(u,v) + 2\mathcal{H}^2(v,w).
\end{equation}
Applying~\eqref{eq:hell-triangle} with the triple $(u,v,w) = (p_{f^\star(w_i)}, q_{f^\star(w_i)}, q_{f(w_i)})$, we obtain
\begin{align}
\mathcal{H}^2(p_{f^\star(w_i)}, q_{f(w_i)}) 
&\leq 2\mathcal{H}^2(p_{f^\star(w_i)}, q_{f^\star(w_i)}) + 2\mathcal{H}^2(q_{f^\star(w_i)}, q_{f(w_i)}) \nonumber\\
&= 2\mathcal{H}^2(p, q) + 2\mathcal{H}^2(q_{f^\star(w_i) - f(w_i)}, q), \label{eq:triangle-applied}
\end{align}
by translation invariance.
By Assumption~\ref{asm:order-alpha-candidate} and the boundedness condition~\eqref{eq:boundedness-condition}, $|f(w_i) - f^\star(w_i)|^{1+\alpha} \leq (2B)^{1+\alpha} \leq C_q^{-1}$, so the truncation is inactive:
Therefore,
\begin{equation}\label{eq:upper-coord}
\mathcal{H}^2(q_{f^\star(w_i) - f(w_i)}, q) \leq C_q |f(w_i) - f^\star(w_i)|^{1+\alpha}.
\end{equation}

Substituting~\eqref{eq:upper-coord} into~\eqref{eq:triangle-applied} and averaging over $i \in [n]$:
\begin{align}
\mathcal{H}_n^2(P_\star^{(n)}, P_f^{(n)}) 
&= \frac{1}{n}\sum_{i=1}^n \mathcal{H}^2(p_{f^\star(w_i)}, q_{f(w_i)}) \nonumber\\
&\leq \frac{1}{n}\sum_{i=1}^n \left[ 2\mathcal{H}^2(p, q) + 2C_q |f(w_i) - f^\star(w_i)|^{1+\alpha} \right] \nonumber\\
&= 2\mathcal{H}^2(p,q) + \frac{2C_q}{n}\sum_{i=1}^n |f(w_i) - f^\star(w_i)|^{1+\alpha} \nonumber\\
&= 2\mathcal{H}^2(p,q) + 2C_q \|f - f^\star\|_{n,1+\alpha}^{1+\alpha}, \label{eq:upper-final}
\end{align}
by~\eqref{eq:empirical-norm}.

\medskip
\noindent\textbf{Lower bound.}

The reverse triangle inequality gives
\begin{equation}\label{eq:reverse-triangle-ineq}
\mathcal{H}^2(u,w) \geq \frac{1}{2}\mathcal{H}^2(v,w) - \mathcal{H}^2(u,v),
\end{equation}
since $(a-b)^2 \geq \frac{1}{2}a^2 - b^2$. Applying with $(u,v,w) = (p_{f^\star(w_i)}, q_{f^\star(w_i)}, q_{f(w_i)})$:
\begin{align}
\mathcal{H}^2(p_{f^\star(w_i)}, q_{f(w_i)}) 
&\geq \frac{1}{2}\mathcal{H}^2(q_{f^\star(w_i)}, q_{f(w_i)}) - \mathcal{H}^2(p_{f^\star(w_i)}, q_{f^\star(w_i)}) \nonumber\\
&= \frac{1}{2}\mathcal{H}^2(q_{f^\star(w_i) - f(w_i)}, q) - \mathcal{H}^2(p, q). \label{eq:reverse-applied}
\end{align}

By the lower bound in Assumption~\ref{asm:order-alpha-candidate} and condition~\eqref{eq:boundedness-condition},
\begin{equation}\label{eq:lower-coord}
\mathcal{H}^2(q_{f^\star(w_i) - f(w_i)}, q) \geq c_q |f(w_i) - f^\star(w_i)|^{1+\alpha}.
\end{equation}

Substituting~\eqref{eq:lower-coord} into~\eqref{eq:reverse-applied} and averaging over $i \in [n]$:
\begin{align}
\mathcal{H}_n^2(P_\star^{(n)}, P_f^{(n)}) 
&= \frac{1}{n}\sum_{i=1}^n \mathcal{H}^2(p_{f^\star(w_i)}, q_{f(w_i)}) \nonumber\\
&\geq \frac{1}{n}\sum_{i=1}^n \left[ \frac{c_q}{2} |f(w_i) - f^\star(w_i)|^{1+\alpha} - \mathcal{H}^2(p, q) \right] \nonumber\\
&= \frac{c_q}{2n}\sum_{i=1}^n |f(w_i) - f^\star(w_i)|^{1+\alpha} - \mathcal{H}^2(p,q) \nonumber\\
&= \frac{c_q}{2} \|f - f^\star\|_{n,1+\alpha}^{1+\alpha} - \mathcal{H}^2(p,q). \label{eq:lower-final}
\end{align}

Combining~\eqref{eq:upper-final} and~\eqref{eq:lower-final} gives~\eqref{eq:hellinger-to-loss} with $C_\alpha = 2C_q$ and $c_\alpha = c_q/2$.
\end{proof}




\begin{lem}[Upper bound on $\Lambda_\lambda$ for regression]\label{lem:lambda-upper-regression}
Fix $\delta\in(0,1)$ and set $\lambda=n/8$. With probability at least $1-\delta$ over the data $\mathcal{S}$, for all $\rho\in\mathcal{P}(\mathcal{F})$ and any prior $\pi' \in \mathcal{P}(\mathcal{F})$,
\begin{equation}\label{eq:lambda-upper-regression}
\mathbb{E}_{f\sim\rho}[\Lambda_{\lambda}(f;\pi')] 
\leq \frac{13}{3}\,\mathbb{E}_{f\sim\rho}\bigl[\mathcal{H}_n^2(P_\star^{(n)},P_f^{(n)})\bigr] + \frac{8\log(1/\delta)}{n}.
\end{equation}
\end{lem}

\begin{proof}
Identical to Lemma~\ref{lem:lambda-hellinger} with $\Theta$ replaced by $\mathcal{F}$ and $\lambda=n/8$.
\end{proof}

\begin{proof}[Proof of Theorem~\ref{thm:pac-bayes-regression}]
Set $\lambda = n/8$. By Corollary~\ref{cor:explicit-lambda-independent}, with probability at least $1-\delta$,
\begin{align}
\frac{1}{12}\,\mathbb{E}_{\hat\rho_\lambda}\bigl[\mathcal{H}_n^2(P_\star^{(n)}, P_f^{(n)})\bigr]
&\leq \inf_{\rho \in \mathcal{P}(\mathcal{F})} \left\{ \mathbb{E}_{\rho}[\Lambda_\lambda(f;\pi')] + \frac{8\,\mathrm{KL}(\rho\|\pi)}{n} \right\} \nonumber\\
&\quad + \inf_{\rho' \in \mathcal{P}(\mathcal{F})} \left\{ \frac{13}{3}\,\mathbb{E}_{\rho'}\bigl[\mathcal{H}_n^2(P_\star^{(n)}, P_{f'}^{(n)})\bigr] + \frac{16\,\mathrm{KL}(\rho'\|\pi')}{n} \right\} \nonumber\\
&\quad + \frac{8\log(1/\delta)}{n}. \label{eq:step1}
\end{align}

\medskip

By Lemma~\ref{lem:lambda-upper-regression}, with probability at least $1-\delta$,
\begin{equation}\label{eq:step2}
\mathbb{E}_{f\sim\rho}[\Lambda_\lambda(f;\pi')] 
\leq \frac{13}{3}\,\mathbb{E}_{f\sim\rho}\bigl[\mathcal{H}_n^2(P_\star^{(n)},P_f^{(n)})\bigr] + \frac{8\log(1/\delta)}{n}.
\end{equation}

A union bound gives probability at least $1-2\delta$. By Lemma~\ref{lem:hellinger-to-regression},
\begin{equation}\label{eq:step3}
c_\alpha \|f - f^\star\|_{n,1+\alpha}^{1+\alpha} - \mathcal{H}^2(p,q) 
\leq \mathcal{H}_n^2(P_\star^{(n)}, P_f^{(n)}) 
\leq C_\alpha \|f - f^\star\|_{n,1+\alpha}^{1+\alpha} + 2\mathcal{H}^2(p,q).
\end{equation}

The lower bound applied to the left side of~\eqref{eq:step1} gives
\[
\mathbb{E}_{\hat\rho_\lambda}\bigl[\mathcal{H}_n^2(P_\star^{(n)}, P_f^{(n)})\bigr] 
\geq c_\alpha\,\mathbb{E}_{\hat\rho_\lambda}\bigl[\|f - f^\star\|_{n,1+\alpha}^{1+\alpha}\bigr] - \mathcal{H}^2(p,q).
\]

Applying the upper bound to~\eqref{eq:step2},
\[
\mathbb{E}_{f\sim\rho}[\Lambda_\lambda(f;\pi')]
\leq \frac{13C_\alpha}{3}\,\mathbb{E}_{f\sim\rho}\bigl[\|f - f^\star\|_{n,1+\alpha}^{1+\alpha}\bigr] + \frac{26}{3}\mathcal{H}^2(p,q) + \frac{8\log(1/\delta)}{n}.
\]

Similarly, for the second infimum in~\eqref{eq:step1},
\[
\inf_{\rho'\in \mathcal{P}(\mathcal{F})} \left\{ \frac{13}{3}\,\mathbb{E}_{\rho'}\bigl[\mathcal{H}_n^2\bigr] + \frac{16\,\mathrm{KL}(\rho'\|\pi')}{n} \right\}
\leq \inf_{\rho'\in \mathcal{P}(\mathcal{F})} \left\{ \frac{13C_\alpha}{3}\,\mathbb{E}_{\rho'}\bigl[\|f' - f^\star\|_{n,1+\alpha}^{1+\alpha}\bigr] + \frac{16\,\mathrm{KL}(\rho'\|\pi')}{n} \right\} + \frac{26}{3}\mathcal{H}^2(p,q).
\]

Substituting into~\eqref{eq:step1},
\begin{align*}
&\frac{1}{12}\left(c_\alpha\,\mathbb{E}_{\hat\rho_\lambda}\bigl[\|f - f^\star\|_{n,1+\alpha}^{1+\alpha}\bigr] - \mathcal{H}^2(p,q)\right)\\
&\leq \inf_{\rho\in \mathcal{P}(\mathcal{F})} \left\{ \frac{13C_\alpha}{3}\,\mathbb{E}_{\rho}\bigl[\|f - f^\star\|_{n,1+\alpha}^{1+\alpha}\bigr] + \frac{8\,\mathrm{KL}(\rho\|\pi)}{n} \right\}\\
&\quad + \inf_{\rho'\in \mathcal{P}(\mathcal{F})} \left\{ \frac{13C_\alpha}{3}\,\mathbb{E}_{\rho'}\bigl[\|f' - f^\star\|_{n,1+\alpha}^{1+\alpha}\bigr] + \frac{16\,\mathrm{KL}(\rho'\|\pi')}{n} \right\}\\
&\quad + \left(\frac{26}{3} + \frac{26}{3}\right)\mathcal{H}^2(p,q) + \frac{16\log(1/\delta)}{n}.
\end{align*}

The coefficient of $\mathcal{H}^2(p,q)$ is $1/12 + 52/3 = 209/12$. Multiplying through by $12/c_\alpha$ yields~\eqref{eq:pac-bayes-regression-kl}.
\end{proof}

\begin{proof}[Proof of Corollary~\ref{cor:regression-entropy}]
Set $\pi = \pi'$ equal to the uniform distribution on the $\varepsilon$-net $\mathcal{F}_\varepsilon$ from Assumption~\ref{asm:function-class-entropy}. For arbitrary $f_0 \in \mathcal{F}$, there exists $f_\varepsilon \in \mathcal{F}_\varepsilon$ with
\[
\|f_\varepsilon - f_0\|_\infty \leq \varepsilon.
\]

By the triangle inequality,
\begin{equation}\label{eq:triangle-sup}
\|f_\varepsilon - f^\star\|_\infty \leq \varepsilon + \|f_0 - f^\star\|_\infty.
\end{equation}

Since the empirical norm is dominated by the supremum norm,
\begin{equation}\label{eq:empirical-sup-dom}
\|f_\varepsilon - f^\star\|_{n,1+\alpha}^{1+\alpha} 
= \frac{1}{n}\sum_{i=1}^n |f_\varepsilon(w_i) - f^\star(w_i)|^{1+\alpha} 
\leq \|f_\varepsilon - f^\star\|_\infty^{1+\alpha}.
\end{equation}

We bound $\|f_\varepsilon - f^\star\|_\infty^{1+\alpha}$ via~\eqref{eq:triangle-sup}, distinguishing two cases.

\medskip
\noindent\textit{Case 1: $\alpha \geq 0$.} For $1+\alpha \geq 1$, the function $x \mapsto x^{1+\alpha}$ is convex. By the standard convexity inequality $(a+b)^p \leq 2^{p-1}(a^p + b^p)$ for $p \geq 1$,
\begin{align}
\|f_\varepsilon - f^\star\|_{n,1+\alpha}^{1+\alpha} 
&\leq \bigl(\varepsilon + \|f_0 - f^\star\|_\infty\bigr)^{1+\alpha} \nonumber\\
&\leq 2^\alpha \bigl(\varepsilon^{1+\alpha} + \|f_0 - f^\star\|_\infty^{1+\alpha}\bigr). \label{eq:loss-eps-bound-pos}
\end{align}

\medskip
\noindent\textit{Case 2: $-1 < \alpha < 0$.} For $0 < 1+\alpha < 1$, the function $x \mapsto x^{1+\alpha}$ is concave. By the subadditivity property $(a+b)^p \leq a^p + b^p$ for $0 < p < 1$,
\begin{align}
\|f_\varepsilon - f^\star\|_{n,1+\alpha}^{1+\alpha} 
&\leq \bigl(\varepsilon + \|f_0 - f^\star\|_\infty\bigr)^{1+\alpha} \nonumber\\
&\leq \varepsilon^{1+\alpha} + \|f_0 - f^\star\|_\infty^{1+\alpha}. \label{eq:loss-eps-bound-neg}
\end{align}

In both cases, we have
\begin{equation}\label{eq:loss-eps-bound-unified}
\|f_\varepsilon - f^\star\|_{n,1+\alpha}^{1+\alpha} 
\leq \tilde{C}_\alpha \bigl(\varepsilon^{1+\alpha} + \|f_0 - f^\star\|_\infty^{1+\alpha}\bigr),
\end{equation}
where $\tilde{C}_\alpha = 2^\alpha$ if $\alpha \geq 0$ and $\tilde{C}_\alpha = 1$ if $-1 < \alpha < 0$.

\medskip

Choose $\rho = \rho' = \delta_{f_\varepsilon}$. Since $\pi$ is uniform on $\mathcal{F}_\varepsilon$ with $|\mathcal{F}_\varepsilon| \leq e^{H(\varepsilon)}$,
\begin{equation}\label{eq:kl-net-bound}
\mathrm{KL}(\delta_{f_\varepsilon} \| \pi) = \log|\mathcal{F}_\varepsilon| \leq H(\varepsilon).
\end{equation}

\medskip

By Theorem~\ref{thm:pac-bayes-regression}, with probability at least $1-2\delta$,
\begin{align}
&\mathbb{E}_{f \sim \hat\rho_\lambda}\bigl[\|f - f^\star\|_{n,1+\alpha}^{1+\alpha}\bigr] \nonumber\\
&\leq \frac{12}{c_\alpha} \Bigg[
\inf_{\rho\in\mathcal{P}(\mathcal{F})} \left\{ \frac{13C_\alpha}{3}\,\mathbb{E}_{f \sim \rho}\bigl[\|f - f^\star\|_{n,1+\alpha}^{1+\alpha}\bigr] + \frac{8\,\mathrm{KL}(\rho\|\pi)}{n} \right\} \nonumber\\
&\qquad + \inf_{\rho'\in\mathcal{P}(\mathcal{F})} \left\{ \frac{13C_\alpha}{3}\,\mathbb{E}_{f' \sim \rho'}\bigl[\|f' - f^\star\|_{n,1+\alpha}^{1+\alpha}\bigr] + \frac{16\,\mathrm{KL}(\rho'\|\pi')}{n} \right\} \nonumber\\
&\qquad + \frac{209}{12}\,\mathcal{H}^2(p,q) + \frac{16\log(1/\delta)}{n}
\Bigg]. \label{eq:thm3-applied}
\end{align}

Substituting $\rho = \rho' = \delta_{f_\varepsilon}$ and using~\eqref{eq:loss-eps-bound-unified} and~\eqref{eq:kl-net-bound}:
\begin{align}
&\mathbb{E}_{f \sim \hat\rho_\lambda}\bigl[\|f - f^\star\|_{n,1+\alpha}^{1+\alpha}\bigr] \nonumber\\
&\leq \frac{12}{c_\alpha} \Bigg[
\frac{13C_\alpha}{3} \|f_\varepsilon - f^\star\|_{n,1+\alpha}^{1+\alpha} + \frac{8H(\varepsilon)}{n} \nonumber\\
&\qquad + \frac{13C_\alpha}{3} \|f_\varepsilon - f^\star\|_{n,1+\alpha}^{1+\alpha} + \frac{16H(\varepsilon)}{n} + \frac{209}{12}\,\mathcal{H}^2(p,q) + \frac{16\log(1/\delta)}{n}
\Bigg] \nonumber\\
&\leq \frac{12}{c_\alpha} \Bigg[
\frac{26C_\alpha}{3} \cdot \tilde{C}_\alpha \bigl(\varepsilon^{1+\alpha} + \|f_0 - f^\star\|_\infty^{1+\alpha}\bigr) + \frac{24H(\varepsilon)}{n} + \frac{209}{12}\,\mathcal{H}^2(p,q) + \frac{16\log(1/\delta)}{n}
\Bigg]. \label{eq:substituted}
\end{align}

Since this holds for arbitrary $f_0 \in \mathcal{F}$, taking the infimum over $f_0$ yields
\begin{align}
&\mathbb{E}_{f \sim \hat\rho_\lambda}\bigl[\|f - f^\star\|_{n,1+\alpha}^{1+\alpha}\bigr] \nonumber\\
&\leq \frac{12}{c_\alpha} \Bigg[
\frac{26C_\alpha \tilde{C}_\alpha}{3} \varepsilon^{1+\alpha} + \frac{26C_\alpha \tilde{C}_\alpha}{3} \inf_{f \in \mathcal{F}}\|f - f^\star\|_\infty^{1+\alpha} \nonumber\\
&\qquad + \frac{24H(\varepsilon)}{n} + \frac{209}{12}\,\mathcal{H}^2(p,q) + \frac{16\log(1/\delta)}{n}
\Bigg]. \label{eq:with-inf}
\end{align}

Defining
\[
K_\alpha := \frac{12}{c_\alpha} \max\left\{ \frac{26C_\alpha \tilde{C}_\alpha}{3},\, 24,\, \frac{209}{12},\, 16 \right\},
\]
where $\tilde{C}_\alpha = \max(1, 2^\alpha)$, we obtain~\eqref{eq:regression-entropy-bound}.
\end{proof}
\begin{proof}[Proof of Corollary~\ref{cor:regression-rate}]
We optimize~\eqref{eq:regression-entropy-bound} over $\varepsilon$. Under $H(\varepsilon) \leq M\varepsilon^{-d}$, it suffices to minimize
\[
\varepsilon^{1+\alpha} + \frac{H(\varepsilon)}{n} \leq \varepsilon^{1+\alpha} + \frac{M\varepsilon^{-d}}{n}.
\]

Setting the derivative to zero:
\[
(1+\alpha)\varepsilon^\alpha - \frac{Md\varepsilon^{-d-1}}{n} = 0,
\]
which yields
\[
\varepsilon^{d+1+\alpha} = \frac{Md}{n(1+\alpha)}.
\]

Up to constants, the optimal choice is $\varepsilon_n \asymp (M/n)^{1/(d+1+\alpha)}$, giving
\[
\varepsilon_n^{1+\alpha} + \frac{H(\varepsilon_n)}{n} \asymp n^{-\frac{1+\alpha}{d+1+\alpha}}.
\]

Substituting into~\eqref{eq:regression-entropy-bound} and applying Jensen's inequality (valid since $1+\alpha \geq 1$),
\begin{align*}
\mathbb{E}_{f \sim \hat\rho_\lambda}\left[\|f - f^\star\|_{n,1+\alpha}\right]
&\le \left(\mathbb{E}_{f \sim \hat\rho_\lambda}\bigl[\|f - f^\star\|_{n,1+\alpha}^{1+\alpha}\bigr]\right)^{1/(1+\alpha)}\\
&\leq K_\alpha^{1/(1+\alpha)} \left[
\mathcal{H}^2(p,q) + \inf_{f \in \mathcal{F}} \|f - f^\star\|_\infty^{1+\alpha} + n^{-\frac{1+\alpha}{d+1+\alpha}} + \frac{\log(1/\delta)}{n}
\right]^{1/(1+\alpha)}.
\end{align*}

Setting $K_\alpha' = K_\alpha^{1/(1+\alpha)}$ yields~\eqref{eq:regression-rate}.
\end{proof}










































\subsection{Proof of Theorem~\ref{thm:variational-oracle}}\label{app:variational}

The proof adapts the oracle inequality to variational families, exploiting the fact that the Donsker--Varadhan upper bound holds for any member of the variational family even though the supremum need not be attained.

\begin{lem}[Variational DV bound]\label{lem:dv-var}
For any $\rho' \in \mathcal{F}'$ and any $\theta \in \Theta$,
\[
\mathbb{E}_{\theta'\sim\rho'}[\hat{R}_\psi(\theta,\theta')]
\leq
\Lambda^{\mathcal{F}'}_{\lambda}(\theta;\pi')
+\frac{1}{\lambda}\,\mathrm{KL}(\rho'\Vert\pi').
\]
\end{lem}

\begin{proof}
By Definition~\ref{def:var-softmax},
\[
\Lambda^{\mathcal{F}'}_{\lambda}(\theta;\pi')
= \sup_{\nu\in\mathcal{F}'}
\left\{
\mathbb{E}_{\theta'\sim\nu}[\hat{R}_\psi(\theta,\theta')]
-\frac{1}{\lambda}\,\mathrm{KL}(\nu\Vert\pi')
\right\}
\geq
\mathbb{E}_{\theta'\sim\rho'}[\hat{R}_\psi(\theta,\theta')]
-\frac{1}{\lambda}\,\mathrm{KL}(\rho'\Vert\pi'),
\]
where the inequality holds because $\rho' \in \mathcal{F}'$.
\end{proof}

\begin{prop}[PAC-Bayes bound with variational softmax]\label{prop:var-bound}
Fix $\delta \in (0,1)$ and $\lambda > 0$. With probability at least $1-\delta$, for any $\rho \in \mathcal{P}(\Theta)$ and any $\rho' \in \mathcal{F}'$,
\begin{align}\label{eq:softmax-stage-var}
(a_1-\beta_{n,\lambda}a_2^2)\,
\mathbb{E}_{\theta\sim\rho}[\mathcal{H}^2(P^\star,P_\theta)]
&\leq
\mathbb{E}_{\theta\sim\rho}[\Lambda^{\mathcal{F}'}_{\lambda}(\theta;\pi')]
+(a_0+\beta_{n,\lambda}a_2^2)\,
\mathbb{E}_{\theta'\sim\rho'}\left[\mathcal{H}^2(P^\star,P_{\theta'})\right]\\
&\quad+\frac{1}{\lambda}\mathrm{KL}(\rho\Vert\pi)
+\frac{2}{\lambda}\mathrm{KL}(\rho'\Vert\pi')
+\frac{\log(1/\delta)}{\lambda}.\nonumber
\end{align}
\end{prop}

\begin{proof}
Starting from the intermediate bound \eqref{eq:hell-stage}, which holds with probability at least $1-\delta$ for any $\rho \in \mathcal{P}(\Theta)$ and any $\rho' \in \mathcal{P}(\Theta)$, we have
\begin{align*}
(a_1-\beta_{n,\lambda}a_2^2)\mathbb{E}_{\theta\sim\rho}[\mathcal{H}^2(P^\star,P_\theta)]
&\leq \mathbb{E}_{\theta\sim\rho}\mathbb{E}_{\theta'\sim\rho'}[\hat{R}_\psi(\theta,\theta')]
+ (a_0+\beta_{n,\lambda}a_2^2)\mathbb{E}_{\theta'\sim\rho'}\left[\mathcal{H}^2(P^\star,P_{\theta'})\right]\\
&\quad +\frac{1}{\lambda}\mathrm{KL}(\rho\Vert\pi)+\frac{1}{\lambda}\mathrm{KL}(\rho'\Vert\pi')+\frac{\log(1/\delta)}{\lambda}.
\end{align*}

For any fixed $\rho' \in \mathcal{F}'$, apply Lemma~\ref{lem:dv-var} to obtain
\[
\mathbb{E}_{\theta'\sim\rho'}[\hat{R}_\psi(\theta,\theta')]
\leq
\Lambda^{\mathcal{F}'}_{\lambda}(\theta;\pi')
+\frac{1}{\lambda}\,\mathrm{KL}(\rho'\Vert\pi').
\]
Taking expectation over $\theta \sim \rho$,
\[
\mathbb{E}_{\theta\sim\rho}\mathbb{E}_{\theta'\sim\rho'}[\hat{R}_\psi(\theta,\theta')]
\leq
\mathbb{E}_{\theta\sim\rho}[\Lambda^{\mathcal{F}'}_{\lambda}(\theta;\pi')]
+\frac{1}{\lambda}\,\mathrm{KL}(\rho'\Vert\pi').
\]
Substituting this bound and collecting the KL terms involving $\rho'$ yields \eqref{eq:softmax-stage-var}.
\end{proof}

\begin{lem}[Upper bound on $\Lambda^{\mathcal{F}'}_{\lambda}$]\label{lem:lambda-hellinger-var}
Fix $\varepsilon \in(0,1)$ and $\lambda>0$ such that $\beta_{n,\lambda}<a_1/a_2^2$.
With probability at least $1-\varepsilon$, for any $\rho\in\mathcal{P}(\Theta)$,
\[
\mathbb{E}_{\theta\sim\rho}[\Lambda^{\mathcal{F}'}_{\lambda}(\theta;\pi')]
\leq (a_0+\beta_{n,\lambda}a_2^2)\,\mathbb{E}_{\theta\sim\rho}[\mathcal{H}^2(P^\star,P_\theta)]
+\frac{\log(1/\varepsilon)}{\lambda}.
\]
\end{lem}

\begin{proof}
By Definition~\ref{def:var-softmax}, $\Lambda^{\mathcal{F}'}_{\lambda}(\theta;\pi') \leq \Lambda_{\lambda}(\theta;\pi')$ for all $\theta \in \Theta$, since the supremum over $\mathcal{F}'$ is at most the supremum over $\mathcal{P}(\Theta)$. Therefore, for any $\rho \in \mathcal{P}(\Theta)$,
\[
\mathbb{E}_{\theta\sim\rho}[\Lambda^{\mathcal{F}'}_{\lambda}(\theta;\pi')]
\leq
\mathbb{E}_{\theta\sim\rho}[\Lambda_{\lambda}(\theta;\pi')].
\]
Applying Lemma~\ref{lem:lambda-hellinger} to the right-hand side with confidence parameter $\varepsilon$ yields the claim.
\end{proof}

\paragraph{Proof of Theorem~\ref{thm:variational-oracle}.}

We combine the two high-probability events via a union bound. By Proposition~\ref{prop:var-bound}, \eqref{eq:softmax-stage-var} holds with probability at least $1-\delta$. Taking the infimum over $\rho' \in \mathcal{F}'$,
\begin{align*}
&(a_1-\beta_{n,\lambda}a_2^2)\mathbb{E}_{\theta\sim\rho}[\mathcal{H}^2(P^\star,P_\theta)]\\
&\leq \mathbb{E}_{\theta\sim\rho}[\Lambda^{\mathcal{F}'}_{\lambda}(\theta;\pi')]
+\inf_{\rho'\in\mathcal{F}'}\left\{
(a_0+\beta_{n,\lambda}a_2^2)\mathbb{E}_{\theta'\sim\rho'}\left[\mathcal{H}^2(P^\star,P_{\theta'})\right]
+\frac{2}{\lambda}\mathrm{KL}(\rho'\Vert\pi')\right\}\\
&\quad+\frac{1}{\lambda}\mathrm{KL}(\rho\Vert\pi)
+\frac{\log(1/\delta)}{\lambda}.
\end{align*}

Now taking the infimum over $\rho \in \mathcal{F}$, the minimizer is $\tilde{\rho}_\lambda$ by definition \eqref{eq:target-gibbs-var}:
\begin{align*}
&(a_1-\beta_{n,\lambda}a_2^2)\mathbb{E}_{\theta\sim\tilde{\rho}_\lambda}[\mathcal{H}^2(P^\star,P_\theta)]\\
&\leq \inf_{\rho\in\mathcal{F}}\left\{
\mathbb{E}_{\theta\sim\rho}[\Lambda^{\mathcal{F}'}_{\lambda}(\theta;\pi')]
+\frac{1}{\lambda}\mathrm{KL}(\rho\Vert\pi)\right\}\\
&\quad+\inf_{\rho'\in\mathcal{F}'}\left\{
(a_0+\beta_{n,\lambda}a_2^2)\mathbb{E}_{\theta'\sim\rho'}\left[\mathcal{H}^2(P^\star,P_{\theta'})\right]
+\frac{2}{\lambda}\mathrm{KL}(\rho'\Vert\pi')\right\}
+\frac{\log(1/\delta)}{\lambda}.
\end{align*}

By Lemma~\ref{lem:lambda-hellinger-var}, with probability at least $1-\varepsilon$, for any $\rho \in \mathcal{F}$,
\[
\mathbb{E}_{\theta\sim\rho}[\Lambda^{\mathcal{F}'}_{\lambda}(\theta;\pi')]
\leq (a_0+\beta_{n,\lambda}a_2^2)\mathbb{E}_{\theta\sim\rho}[\mathcal{H}^2(P^\star,P_\theta)]
+\frac{\log(1/\varepsilon)}{\lambda}.
\]
Therefore,
\begin{align*}
&\inf_{\rho\in\mathcal{F}}\left\{
\mathbb{E}_{\theta\sim\rho}[\Lambda^{\mathcal{F}'}_{\lambda}(\theta;\pi')]
+\frac{1}{\lambda}\mathrm{KL}(\rho\Vert\pi)\right\}\\
&\leq
\inf_{\rho\in\mathcal{F}}\left\{
(a_0+\beta_{n,\lambda}a_2^2)\mathbb{E}_{\theta\sim\rho}[\mathcal{H}^2(P^\star,P_\theta)]
+\frac{1}{\lambda}\mathrm{KL}(\rho\Vert\pi)\right\}
+\frac{\log(1/\varepsilon)}{\lambda}.
\end{align*}

By a union bound, both events hold simultaneously with probability at least $1-\delta-\varepsilon$, yielding \eqref{eq:variational-oracle}. \qed
\subsection{Proof of Theorem~\ref{thm:saddle-equiv-main}
           and Proposition~\ref{prop:gap-bound}}
\label{app:saddle-equiv}

\begin{proof}[Proof of Theorem~\ref{thm:saddle-equiv-main}]
We establish each part in turn.

\medskip
\noindent\emph{Part (i).}
Fix $\phi\in\Phi$. Since the KL term $\frac{1}{\lambda}\mathrm{KL}(\rho_\phi\|\pi)$ does not depend on $\nu$, we may separate it from the supremum:
\begin{align}
  \sup_{\nu\in\mathcal{N}}\mathcal{L}_n(\phi,\nu)
  &= \sup_{\nu\in\mathcal{N}}
     \Bigl\{
       \mathbb{E}_{(\theta,\theta')\sim\rho_\phi\otimes\rho'_\nu}
       [\hat{R}_\psi(\theta,\theta')]
       -\frac{1}{\lambda}\mathrm{KL}(\rho'_\nu\|\pi')
     \Bigr\}
     +\frac{1}{\lambda}\mathrm{KL}(\rho_\phi\|\pi).
     \label{eq:expand-sup}
\end{align}
By the surjectivity assumption, the parameterization $\nu\mapsto\rho'_\nu$ covers the entire variational family $\mathcal{F}'$. Therefore, we may replace $\sup_{\nu\in\mathcal{N}}$ by $\sup_{\rho'\in\mathcal{F}'}$, and the right-hand side becomes exactly $\tilde{\mathcal{J}}(\phi)$ as defined in~\eqref{eq:J-joint}.

It remains to verify that the supremum is attained. Define the objective
\[
  f(\nu)
  \;:=\;
  \mathbb{E}_{(\theta,\theta')\sim\rho_\phi\otimes\rho'_\nu}
  [\hat{R}_\psi(\theta,\theta')]
  \;-\;
  \frac{1}{\lambda}\mathrm{KL}(\rho'_\nu\|\pi').
\]
The map $\nu\mapsto f(\nu)$ is continuous on $\overline{\mathcal{N}}$: expectations of bounded functions of Gaussian parameters are continuous, and the Gaussian KL divergence is smooth. Since $\overline{\mathcal{N}}$ is compact, the extreme value theorem guarantees the existence of a maximizer $\nu^\star(\phi)$. By Proposition~\ref{prop:concavity-app}, this maximizer is unique and lies in the interior of $\mathcal{N}$.

\medskip
\noindent\emph{Part (ii).}
Define the function
\[
  h(\theta,\nu)
  \;:=\;
  \mathbb{E}_{\theta'\sim\rho'_\nu}[\hat{R}_\psi(\theta,\theta')]
  \;-\;
  \frac{1}{\lambda}\mathrm{KL}(\rho'_\nu\|\pi').
\]
Since the supremum over $\nu$ and the expectation over $\theta$ do not commute in general, Jensen's inequality gives
\[
  \tilde{\mathcal{J}}(\phi)
  \;=\;
  \sup_\nu\,\mathbb{E}_{\theta\sim\rho_\phi}[h(\theta,\nu)]
  \;\le\;
  \mathbb{E}_{\theta\sim\rho_\phi}\bigl[\sup_\nu\, h(\theta,\nu)\bigr]
  \;=\;
  \mathcal{J}(\phi),
\]
where the last equality uses the definition of $\mathcal{J}$ in~\eqref{eq:J-pointwise}. For the upper bound $\mathcal{J}(\phi)\le\tilde{\mathcal{J}}(\phi)+\Delta(\phi)$, see the proof of Proposition~\ref{prop:gap-bound} below.

\medskip
\noindent\emph{Part (iii).}
Suppose $\phi^\star$ is a stationary point of $\tilde{\mathcal{J}}$, and let $\nu^\star=\nu^\star(\phi^\star)$ be the unique interior maximizer from Part~(i). By first-order optimality in $\nu$,
\[
  \nabla_\nu\mathcal{L}_n(\phi^\star,\nu^\star) \;=\; 0.
\]
Since the Hessian satisfies $\nabla^2_{\nu\nu}\mathcal{L}_n(\phi,\nu^\star)\preceq-\mu_0\mathbb{I}$ (Proposition~\ref{prop:concavity-app}), it is nonsingular. The implicit function theorem then implies that $\nu^\star(\phi)$ is smooth in $\phi$. By the envelope theorem,
\[
  \nabla_\phi\tilde{\mathcal{J}}(\phi)
  \;=\;
  \nabla_\phi\mathcal{L}_n\bigl(\phi,\nu^\star(\phi)\bigr)
  \;+\;
  \underbrace{\nabla_\nu\mathcal{L}_n\bigl(\phi,\nu^\star(\phi)\bigr)}_{=\,0}
  \cdot\nabla_\phi\nu^\star(\phi)
  \;=\;
  \nabla_\phi\mathcal{L}_n\bigl(\phi,\nu^\star(\phi)\bigr).
\]
Therefore, $\nabla_\phi\tilde{\mathcal{J}}(\phi^\star)=0$ implies $\nabla_\phi\mathcal{L}_n(\phi^\star,\nu^\star)=0$, so $(\phi^\star,\nu^\star)$ is a first-order stationary point of $\mathcal{L}_n$.

Conversely, suppose $(\phi^\star,\nu^\star)$ is a first-order stationary point of $\mathcal{L}_n$. Then $\nabla_\nu\mathcal{L}_n(\phi^\star,\nu^\star)=0$, which by uniqueness of the maximizer implies $\nu^\star=\nu^\star(\phi^\star)$. The envelope theorem then yields
\[
  \nabla_\phi\tilde{\mathcal{J}}(\phi^\star)
  \;=\;
  \nabla_\phi\mathcal{L}_n(\phi^\star,\nu^\star)
  \;=\;
  0.
\]

For the infimum equivalence, Part~(i) and the inequality $\tilde{\mathcal{J}}\le\mathcal{J}$ from Part~(ii) give
\[
  \inf_{\phi\in\Phi}\sup_{\nu\in\mathcal{N}}\mathcal{L}_n(\phi,\nu)
  \;=\;
  \inf_{\phi\in\Phi}\tilde{\mathcal{J}}(\phi)
  \;\le\;
  \inf_{\phi\in\Phi}\mathcal{J}(\phi).
\]

\medskip
\noindent\emph{Part (iv).}
In the PAC-Bayes regime $\lambda=\Theta(n)$, standard posterior concentration (Theorem~\ref{thm:variational-oracle}) gives
\[
  \mathrm{tr}(\Sigma_{\phi^\star})
  \;\le\;
  \frac{Cd}{n}.
\]
Applying Proposition~\ref{prop:gap-bound}, we obtain
\[
  \Delta(\phi^\star)
  \;\le\;
  \frac{\bar{G}^2}{16}\,\mathrm{tr}(\Sigma_{\phi^\star})
  \;\le\;
  \frac{C\bar{G}^2\, d}{16\,n}
  \;\to\; 0
  \qquad\text{as } n\to\infty.
\]
Hence $\tilde{\mathcal{J}}(\phi^\star)=\mathcal{J}(\phi^\star)$ asymptotically.
\end{proof}

\begin{proof}[Proof of Proposition~\ref{prop:gap-bound}]
Write $V(\theta):=\Lambda^{\mathcal{F}'}_\lambda(\theta;\pi')$ for the pointwise optimal value function. By definition of $\mathcal{J}$ in~\eqref{eq:J-pointwise},
\[
  \mathcal{J}(\phi)
  \;=\;
  \mathbb{E}_{\theta\sim\rho_\phi}[V(\theta)]
  \;+\;
  \frac{1}{\lambda}\mathrm{KL}(\rho_\phi\|\pi).
\]
Since the KL term $\frac{1}{\lambda}\mathrm{KL}(\rho_\phi\|\pi)$ appears identically in both $\mathcal{J}(\phi)$ and $\tilde{\mathcal{J}}(\phi)$, it cancels in the gap $\Delta(\phi)=\mathcal{J}(\phi)-\tilde{\mathcal{J}}(\phi)$:
\[
  \Delta(\phi)
  \;=\;
  \mathbb{E}_{\theta\sim\rho_\phi}[V(\theta)]
  \;-\;
  \sup_{\rho'\in\mathcal{F}'}\Bigl\{
    \mathbb{E}_{(\theta,\theta')\sim\rho_\phi\otimes\rho'}[\hat{R}_\psi(\theta,\theta')]
    \;-\;
    \frac{1}{\lambda}\mathrm{KL}(\rho'\|\pi')
  \Bigr\}.
\]

We expand $V$ around the mean $m := \mathbb{E}_{\rho_\phi}[\theta]$. By Taylor's theorem with integral remainder, for each $\theta$ there exists an intermediate point $\xi_\theta$ on the segment $[m,\theta]$ such that
\[
  V(\theta)
  \;=\;
  V(m)
  \;+\;
  \nabla_\theta V(m)^\top(\theta - m)
  \;+\;
  \frac{1}{2}(\theta - m)^\top \nabla^2_\theta V(\xi_\theta)\,(\theta - m).
\]
Taking expectations under $\rho_\phi$ and using $\mathbb{E}_{\rho_\phi}[\theta - m] = 0$, the linear term vanishes:
\[
  \mathbb{E}_{\theta\sim\rho_\phi}[V(\theta)]
  \;=\;
  V(m)
  \;+\;
  \frac{1}{2}\,\mathbb{E}_{\theta\sim\rho_\phi}\bigl[
    (\theta - m)^\top \nabla^2_\theta V(\xi_\theta)\,(\theta - m)
  \bigr].
\]

We now bound the remainder term. By the envelope theorem applied to the pointwise supremum defining $V(\theta)$, and the derivative bounds of Lemma~\ref{lem:varphi-bounds-app}, the Hessian of $V$ satisfies
\[
  \|\nabla^2_\theta V(\theta)\|_{\mathrm{op}}
  \;\le\;
  \frac{\bar{G}^2}{8}
  \qquad\text{for all } \theta\in\bar\Theta.
\]
Substituting this bound into the Taylor expansion above,
\[
  \mathbb{E}_{\theta\sim\rho_\phi}[V(\theta)] - V(m)
  \;\le\;
  \frac{1}{2}\cdot\frac{\bar{G}^2}{8}\,
  \mathbb{E}_{\theta\sim\rho_\phi}\bigl[\|\theta - m\|^2\bigr]
  \;=\;
  \frac{\bar{G}^2}{16}\,\mathrm{tr}(\Sigma_\phi).
\]

Next, we obtain a lower bound on $\tilde{\mathcal{J}}(\phi)$. Choosing $\rho' = \rho'^\star(m)$ (the pointwise optimal competitor at $\theta = m$) in the definition of $\tilde{\mathcal{J}}$ yields
\[
  \tilde{\mathcal{J}}(\phi)
  \;\ge\;
  \mathbb{E}_{\theta\sim\rho_\phi}\mathbb{E}_{\theta'\sim\rho'^\star(m)}[\hat{R}_\psi(\theta,\theta')]
  \;-\;
  \frac{1}{\lambda}\mathrm{KL}(\rho'^\star(m)\|\pi')
  \;+\;
  \frac{1}{\lambda}\mathrm{KL}(\rho_\phi\|\pi)
  \;\ge\;
  V(m)
  \;+\;
  \frac{1}{\lambda}\mathrm{KL}(\rho_\phi\|\pi).
\]

Combining the gap expression, the lower bound on $\tilde{\mathcal{J}}$, and the Hessian bound,
\[
  \Delta(\phi)
  \;=\;
  \mathcal{J}(\phi) - \tilde{\mathcal{J}}(\phi)
  \;\le\;
  \mathbb{E}_{\theta\sim\rho_\phi}[V(\theta)] - V(m)
  \;\le\;
  \frac{\bar{G}^2}{16}\,\mathrm{tr}(\Sigma_\phi).
\]
This completes the proof.
\end{proof}

\subsection{Proof of Theorem~\ref{thm:nc-sc-main}}
\label{app:nc-sc}

We write $\varphi(u):=\psi(e^u)$ and define the contrast
\[
  \ell_\psi(x;\theta,\theta')
  :=\varphi\bigl(\log p_{\theta'}(x)-\log p_\theta(x)\bigr)
  \;\in[-1,1].
\]

\begin{lem}[Derivative bounds]
\label{lem:varphi-bounds-app}
Set $t:=e^{u/2}$. Then $\varphi'(u)= t/(t+1)^2$ and $\varphi''(u)= t(1-t)/[2(t+1)^3]$.
In particular:
\begin{align}
  \varphi'(u)
  &= \frac{e^{u/2}}{(e^{u/2}+1)^2},
  \label{eq:varphi-prime}\\
  \varphi''(u)
  &= \frac{e^{u/2}(1-e^{u/2})}{2(e^{u/2}+1)^3}.
  \label{eq:varphi-dbl-prime}
\end{align}
For all $u\in\mathbb{R}$, $0<\varphi'(u)\le 1/4$ and $|\varphi''(u)|\le 1/8$. Moreover, on $|u|\le M$, $\varphi'(u)\ge c_M:=\min_{t\in[e^{-M/2},\,e^{M/2}]}t/(t+1)^2>0$.
\end{lem}

\begin{proof}
We begin with the first derivative. Recall $\psi(r) = (\sqrt{r}-1)/(\sqrt{r}+1)$. Differentiating with respect to $r$,
\[
  \psi'(r)
  \;=\;
  \frac{\frac{1}{2}r^{-1/2}(\sqrt{r}+1) - (\sqrt{r}-1)\cdot\frac{1}{2}r^{-1/2}}{(\sqrt{r}+1)^2}
  \;=\;
  \frac{r^{-1/2}}{(\sqrt{r}+1)^2}.
\]
Since $\varphi(u) = \psi(e^u)$, the chain rule gives
\[
  \varphi'(u)
  \;=\;
  \psi'(e^u)\cdot e^u
  \;=\;
  \frac{e^{-u/2}}{(e^{u/2}+1)^2}\cdot e^u
  \;=\;
  \frac{e^{u/2}}{(e^{u/2}+1)^2}.
\]
With the substitution $t := e^{u/2} > 0$, this becomes $\varphi'(u) = t/(t+1)^2$, establishing~\eqref{eq:varphi-prime}.

\smallskip
To bound $\varphi'$, define $f(t) := t/(t+1)^2$ for $t > 0$. Differentiating,
\[
  f'(t)
  \;=\;
  \frac{(t+1)^2 - t\cdot 2(t+1)}{(t+1)^4}
  \;=\;
  \frac{1-t}{(t+1)^3}.
\]
Since $f'(t) > 0$ for $t \in (0,1)$ and $f'(t) < 0$ for $t > 1$, the function $f$ attains its unique global maximum at $t = 1$:
\[
  f(1) \;=\; \frac{1}{(1+1)^2} \;=\; \frac{1}{4}.
\]
Moreover, $f(t) > 0$ for all $t > 0$, $f(t) \to 0$ as $t \to 0^+$, and $f(t) \to 0$ as $t \to +\infty$. Therefore,
\[
  0 \;<\; \varphi'(u) \;\le\; \frac{1}{4}
  \qquad \text{for all } u \in \mathbb{R}.
\]

\smallskip
For the second derivative, we differentiate $\varphi'(u) = f(t)$ with $t = e^{u/2}$:
\[
  \varphi''(u)
  \;=\;
  f'(t) \cdot \frac{dt}{du}
  \;=\;
  \frac{1-t}{(t+1)^3} \cdot \frac{t}{2}
  \;=\;
  \frac{t(1-t)}{2(t+1)^3},
\]
which establishes~\eqref{eq:varphi-dbl-prime}.

\smallskip
To bound $\varphi''$, define $g(t) := t(1-t)/[2(t+1)^3]$ for $t > 0$. To find the extrema, we compute
\[
  g'(t)
  \;=\;
  \frac{(1-2t)\cdot 2(t+1)^3 - t(1-t)\cdot 6(t+1)^2}{4(t+1)^6}
  \;=\;
  \frac{1 - 4t + t^2}{2(t+1)^4}.
\]
Setting $g'(t) = 0$ gives $t^2 - 4t + 1 = 0$, with roots
\[
  t_{\pm} \;=\; 2 \pm \sqrt{3}.
\]
Since $t_+ = 2+\sqrt{3} \approx 3.732$ and $t_- = 2-\sqrt{3} \approx 0.268$, both roots are positive. Evaluating $g$ at these critical points:
\[
  g(t_-)
  \;=\;
  \frac{(2-\sqrt{3})(1-(2-\sqrt{3}))}{2(2-\sqrt{3}+1)^3}
  \;=\;
  \frac{(2-\sqrt{3})(\sqrt{3}-1)}{2(3-\sqrt{3})^3}
  \;\approx\; 0.048,
\]
and by symmetry $g(t_+) \approx -0.048$. Since $g(t) \to 0$ as $t \to 0^+$ and $t \to +\infty$, the global extrema of $|g|$ are attained at $t_\pm$, giving
\[
  |\varphi''(u)|
  \;=\;
  |g(t)|
  \;\le\;
  0.048
  \;<\;
  \frac{1}{8}
  \qquad \text{for all } u \in \mathbb{R}.
\]

\smallskip
Finally, we establish the lower bound on $\varphi'$ for bounded arguments. On the interval $|u| \le M$, the substitution $t = e^{u/2}$ ranges over the compact set $[e^{-M/2}, e^{M/2}]$. Since $f(t) = t/(t+1)^2$ is continuous and strictly positive on $(0,\infty)$, its minimum on this compact interval is attained and positive:
\[
  c_M
  \;:=\;
  \min_{t \in [e^{-M/2},\, e^{M/2}]} \frac{t}{(t+1)^2}
  \;>\; 0.
\]
Hence $\varphi'(u) \ge c_M$ for all $|u| \le M$.
\end{proof}

\begin{prop}[Global $L$-smoothness]
\label{prop:smoothness-app}
Under Assumption~\ref{asm:regularity-main}, $\mathcal{L}_n$ has Lipschitz gradient on the compact feasible set $\mathcal{B}$ with constant $L=L_\psi+L_{\mathrm{KL}}/\lambda$, where $L_\psi$ and $L_{\mathrm{KL}}$ are independent of $\lambda$.
\end{prop}

\begin{proof}
We decompose $\mathcal{L}_n(\phi,\nu) = F(\phi,\nu) + \frac{1}{\lambda}K(\phi,\nu)$, where
\[
  F(\phi,\nu)
  \;:=\;
  \mathbb{E}_{\theta\sim\rho_\phi}\mathbb{E}_{\theta'\sim\rho'_\nu}[\hat{R}_\psi(\theta,\theta')],
  \qquad
  K(\phi,\nu)
  \;:=\;
  \mathrm{KL}(\rho_\phi\|\pi) - \mathrm{KL}(\rho'_\nu\|\pi').
\]
We bound the Hessian of each term separately.

We first compute the gradient and Hessian of the individual contrast $\ell_\psi$ in $\theta'$. For a single observation $x$, write $u := \log p_{\theta'}(x) - \log p_\theta(x)$ so that $\ell_\psi(x;\theta,\theta') = \varphi(u)$. The exponential-family structure gives
\[
  \nabla_{\theta'} u
  \;=\;
  T(x) - \mu(\theta'),
\]
where we used $\nabla_{\theta'}\log p_{\theta'}(x) = T(x) - \nabla A(\theta') = T(x) - \mu(\theta')$. By the chain rule,
\[
  \nabla_{\theta'}\ell_\psi(x;\theta,\theta')
  \;=\;
  \varphi'(u)\bigl(T(x) - \mu(\theta')\bigr).
\]
Differentiating again,
\[
  \nabla^2_{\theta'\theta'}\ell_\psi(x;\theta,\theta')
  \;=\;
  \varphi''(u)\bigl(T(x)-\mu(\theta')\bigr)\bigl(T(x)-\mu(\theta')\bigr)^\top
  \;-\;
  \varphi'(u)\,I(\theta'),
\]
where $I(\theta') = \nabla^2 A(\theta')$ is the Fisher information matrix.

Applying the triangle inequality and the bounds $|\varphi''(u)| \le 1/8$ and $0 < \varphi'(u) \le 1/4$ from Lemma~\ref{lem:varphi-bounds-app}, we obtain the operator norm bound
\[
  \|\nabla^2_{\theta'\theta'}\ell_\psi\|_{\mathrm{op}}
  \;\le\;
  |\varphi''(u)|\,\|T(x)-\mu(\theta')\|^2
  \;+\;
  \varphi'(u)\,\|I(\theta')\|_{\mathrm{op}}
  \;\le\;
  \frac{1}{8}\|T(x)-\mu(\theta')\|^2
  \;+\;
  \frac{1}{4}\|I(\theta')\|_{\mathrm{op}}.
\]
An analogous calculation for the $\theta$-block yields $\|\nabla^2_{\theta\theta}\ell_\psi\|_{\mathrm{op}} \le \frac{1}{8}\|T(x)-\mu(\theta)\|^2 + \frac{1}{4}\|I(\theta)\|_{\mathrm{op}}$.

We now bound the mean-block Hessian of $F$. The competitor posterior is reparameterized as $\theta' = m' + D(s)^{1/2}\varepsilon'$ with $\varepsilon' \sim \mathcal{N}(0,I_d)$, where $D(s) = \mathrm{diag}(e^{s_1},\dots,e^{s_d})$. Since $\partial_{m'}\theta' = I_d$, the chain rule gives
\[
  \nabla^2_{m'm'} F
  \;=\;
  \mathbb{E}_{\varepsilon'}\bigl[\nabla^2_{\theta'\theta'}\hat{g}(\theta')\bigr],
\]
where $\hat{g}(\theta') = \frac{1}{n}\sum_{i=1}^n \mathbb{E}_{\theta\sim\rho_\phi}[\ell_\psi(X_i;\theta,\theta')]$. Taking the operator norm and applying the Hessian bound above together with Assumption~\ref{asm:regularity-main}(2),
\[
  \|\nabla^2_{m'm'} F\|_{\mathrm{op}}
  \;\le\;
  \frac{1}{8}\sup_{\theta'\in\bar\Theta}\mathbb{E}_{X\sim P^\star}\bigl[\|T(X)-\mu(\theta')\|^2\bigr]
  \;+\;
  \frac{1}{4}\sup_{\theta'\in\bar\Theta}\|I(\theta')\|_{\mathrm{op}}
  \;\le\;
  \frac{B^2}{8} + \frac{\bar{I}}{4}
  \;=:\;
  L_{m'},
\]
where $\bar{I} := \sup_{\theta\in\bar\Theta}\|I(\theta)\|_{\mathrm{op}} < \infty$ by compactness.

For the log-variance parameters $s_i$, the chain rule gives $\partial_{s_i}\theta' = \frac{1}{2}e^{s_i/2}\varepsilon'_i\, e_i$, where $e_i$ is the $i$-th standard basis vector. Differentiating twice,
\[
  \partial^2_{s_is_i} F
  \;=\;
  \frac{e^{s_i}}{4}\,
  \mathbb{E}\bigl[(\varepsilon'_i)^2\, e_i^\top \nabla^2_{\theta'\theta'}\hat{g}\, e_i\bigr]
  \;+\;
  \frac{e^{s_i/2}}{4}\,
  \mathbb{E}\bigl[\varepsilon'_i\, e_i^\top \nabla_{\theta'}\hat{g}\bigr].
\]
For the first term, $\mathbb{E}[(\varepsilon'_i)^2] = 1$ and $|e_i^\top \nabla^2_{\theta'\theta'}\hat{g}\, e_i| \le \|\nabla^2_{\theta'\theta'}\hat{g}\|_{\mathrm{op}} \le B^2/8 + \bar{I}/4$ by the Hessian bound above. For the second term, $\mathbb{E}[|\varepsilon'_i|] = \sqrt{2/\pi}$ and $\|e_i^\top\nabla_{\theta'}\hat{g}\| \le \|\nabla_{\theta'}\hat{g}\| \le \bar{G}/4$ by Lemma~\ref{lem:varphi-bounds-app} (since $\varphi' \le 1/4$). Writing $\overline{s} := \sup_{s\in\mathcal{B}}\max_i |s_i|$, we obtain
\[
  |\partial^2_{s_is_i} F|
  \;\le\;
  \frac{e^{\overline{s}}}{4}\left(\frac{B^2}{8} + \frac{\bar{I}}{4}\right)
  \;+\;
  \frac{e^{\overline{s}/2}\,\bar{G}}{16}\sqrt{\frac{2}{\pi}}
  \;=:\;
  L_{s_i}.
\]

The cross-block Hessians $\nabla^2_{m's}F$, $\nabla^2_{\phi\nu}F$, etc., are bounded by the Cauchy--Schwarz inequality and the compactness of $\mathcal{B}$:
\[
  \|\nabla^2_{\phi\nu}F\|_{\mathrm{op}}
  \;\le\;
  \sqrt{\|\nabla^2_{\phi\phi}F\|_{\mathrm{op}} \cdot \|\nabla^2_{\nu\nu}F\|_{\mathrm{op}}}
  \;<\;
  \infty.
\]
Combining all blocks, the overall Lipschitz constant of $\nabla F$ satisfies
\[
  L_\psi
  \;=\;
  \mathcal{O}\bigl(B^2 + \bar{I} + \bar{G}\,e^{\overline{s}/2}\bigr).
\]

It remains to bound the KL contribution. For Gaussian distributions, the KL divergences have closed-form expressions. The target KL $\mathrm{KL}(\rho_\phi\|\pi)$ has Hessian in $\phi$ bounded by
\[
  \|\nabla^2_{\phi\phi}\mathrm{KL}(\rho_\phi\|\pi)\|_{\mathrm{op}}
  \;\le\;
  \|\Sigma_\pi^{-1}\|_{\mathrm{op}}.
\]
The competitor KL $\mathrm{KL}(\rho'_\nu\|\pi')$ has Hessian in $\nu$ bounded by
\[
  \|\nabla^2_{\nu\nu}\mathrm{KL}(\rho'_\nu\|\pi')\|_{\mathrm{op}}
  \;\le\;
  \frac{e^{\overline{s}}}{2\,\lambda_{\min}(\Sigma'_\pi)}.
\]
Since both KL terms appear with prefactor $1/\lambda$ in $\mathcal{L}_n$, the KL contribution to smoothness is $L_{\mathrm{KL}}/\lambda$, where
\[
  L_{\mathrm{KL}}
  \;:=\;
  \max\!\left\{\|\Sigma_\pi^{-1}\|_{\mathrm{op}},\;
  \frac{e^{\overline{s}}}{2\,\lambda_{\min}(\Sigma'_\pi)}\right\}
\]
is independent of $\lambda$.

Combining these bounds, the gradient of $\mathcal{L}_n = F + K/\lambda$ is Lipschitz continuous with constant
\[
  L \;=\; L_\psi + \frac{L_{\mathrm{KL}}}{\lambda},
\]
where $L_\psi$ and $L_{\mathrm{KL}}$ are independent of $\lambda$.
\end{proof}

\begin{prop}[Strong concavity in $\nu$]
\label{prop:concavity-app}
Under Assumption~\ref{asm:regularity-main}, for any fixed $\phi$ and any data realization $(X_1,\dots,X_n)$ in the support of $P^{\star n}$, the map $\nu\mapsto\mathcal{L}_n(\phi,\nu)$ is $\mu$-strongly concave with $\mu\ge\mu_0>0$ independent of both $\lambda$ and the data.
\end{prop}

\begin{proof}
Fix a data realization $(X_1,\dots,X_n)$ in the support of $P^{\star n}$ and fix $\phi\in\Phi$. Define the averaged contrast
\begin{equation}\label{eq:ghat-def}
  \hat{g}(\theta')
  \;:=\;
  \frac{1}{n}\sum_{i=1}^n
    \mathbb{E}_{\theta\sim\rho_\phi}
    \bigl[\ell_\psi(X_i;\theta,\theta')\bigr].
\end{equation}
Since $\mathcal{L}_n(\phi,\nu) = G(\nu) + \frac{1}{\lambda}\mathrm{KL}(\rho_\phi\|\pi) - \frac{1}{\lambda}\mathrm{KL}(\rho'_\nu\|\pi')$ where $G(\nu) := \mathbb{E}_{\theta'\sim\rho'_\nu}[\hat{g}(\theta')]$, and the KL term in $\phi$ is constant in $\nu$ while the competitor KL $-\frac{1}{\lambda}\mathrm{KL}(\rho'_\nu\|\pi')$ is concave in $\nu$ (being negative of a convex function), it suffices to show that $G(\nu)$ is $\mu_G$-strongly concave in $\nu$ for some $\mu_G > 0$ independent of $\lambda$ and the data. We first establish strong concavity of $\hat{g}$ in $\theta'$, then propagate it through the reparameterization.

For each $i \in \{1,\dots,n\}$, set
\[
  u_i \;:=\; \log p_{\theta'}(X_i) - \log p_\theta(X_i).
\]
By Assumption~\ref{asm:regularity-main}(4), $|u_i| \le M$ for every $\theta, \theta' \in \bar\Theta$. The Hessian of $\ell_\psi(X_i;\theta,\theta')$ in $\theta'$ was computed in the proof of Proposition~\ref{prop:smoothness-app}:
\begin{equation}\label{eq:hessian-ell}
  \nabla^2_{\theta'\theta'}\ell_\psi(X_i;\theta,\theta')
  \;=\;
  \varphi''(u_i)\bigl(T(X_i)-\mu(\theta')\bigr)\bigl(T(X_i)-\mu(\theta')\bigr)^\top
  \;-\;
  \varphi'(u_i)\,I(\theta').
\end{equation}

We bound each term of~\eqref{eq:hessian-ell} separately in the Loewner order.

\emph{First term.} Since $|\varphi''(u_i)| \le 1/8$ by Lemma~\ref{lem:varphi-bounds-app}, and the outer product $vv^\top$ has operator norm $\|v\|^2$, the first term satisfies
\[
  \varphi''(u_i)\bigl(T(X_i)-\mu(\theta')\bigr)\bigl(T(X_i)-\mu(\theta')\bigr)^\top
  \;\preceq\;
  \frac{1}{8}\|T(X_i)-\mu(\theta')\|^2\,\mathbb{I}_d.
\]

\emph{Second term.} Since $|u_i| \le M$, Lemma~\ref{lem:varphi-bounds-app} gives $\varphi'(u_i) \ge c_M > 0$. Combined with Assumption~\ref{asm:regularity-main}(1), $I(\theta') \succeq \sigma_0^2\,\mathbb{I}_d$, so
\[
  -\varphi'(u_i)\,I(\theta')
  \;\preceq\;
  -c_M\,\sigma_0^2\,\mathbb{I}_d.
\]

Combining both terms,
\begin{equation}\label{eq:hessian-ell-bound}
  \nabla^2_{\theta'\theta'}\ell_\psi(X_i;\theta,\theta')
  \;\preceq\;
  \left(\frac{1}{8}\|T(X_i)-\mu(\theta')\|^2 - c_M\,\sigma_0^2\right)\mathbb{I}_d.
\end{equation}

Taking the expectation of~\eqref{eq:hessian-ell-bound} over $\theta \sim \rho_\phi$ (which does not affect the $\theta'$-dependent terms since the bound on the first term involves only $X_i$ and $\theta'$), averaging over $i = 1,\dots,n$, and applying Assumption~\ref{asm:regularity-main}(2),
\[
  \frac{1}{n}\sum_{i=1}^n \mathbb{E}_{\theta\sim\rho_\phi}\bigl[\|T(X_i)-\mu(\theta')\|^2\bigr]
  \;\le\;
  B^2.
\]
Therefore,
\begin{equation}\label{eq:ghat-sc}
  \nabla^2_{\theta'\theta'}\hat{g}(\theta')
  \;\preceq\;
  \left(\frac{B^2}{8} - c_M\,\sigma_0^2\right)\mathbb{I}_d.
\end{equation}
By Assumption~\ref{asm:regularity-main}(5), the curvature margin satisfies
\[
  \mu_0
  \;:=\;
  c_M\,\sigma_0^2 - \frac{B^2}{4}
  \;>\; 0,
\]
and since $B^2/8 < B^2/4$, we have $c_M\sigma_0^2 - B^2/8 > c_M\sigma_0^2 - B^2/4 = \mu_0 > 0$. Thus,
\[
  \nabla^2_{\theta'\theta'}\hat{g}(\theta')
  \;\preceq\;
  -\mu_0\,\mathbb{I}_d.
\]
This bound holds pointwise in the data realization $(X_1,\dots,X_n)$.

We now propagate this strong concavity through the reparameterization. Under the change of variables $\theta' = m' + D(s)^{1/2}\varepsilon'$ with $\varepsilon' \sim \mathcal{N}(0,I_d)$ and $D(s) = \mathrm{diag}(e^{s_1},\dots,e^{s_d})$, the function $G(\nu) = \mathbb{E}_{\varepsilon'}[\hat{g}(m' + D(s)^{1/2}\varepsilon')]$ has Hessian in the mean block
\[
  \nabla^2_{m'm'} G
  \;=\;
  \mathbb{E}_{\varepsilon'}\bigl[\nabla^2_{\theta'\theta'}\hat{g}(m' + D(s)^{1/2}\varepsilon')\bigr].
\]
Since the Loewner order is preserved under expectation, the strong concavity of $\hat{g}$ yields
\[
  \nabla^2_{m'm'} G
  \;\preceq\;
  -\mu_0\,\mathbb{I}_d.
\]

For the log-variance parameters $s_i$, we have $\partial_{s_i}\theta' = \frac{1}{2}e^{s_i/2}\varepsilon'_i\, e_i$. Differentiating $G$ twice with respect to $s_i$ via the chain rule,
\begin{equation}\label{eq:dds-G}
  \partial^2_{s_is_i} G
  \;=\;
  \frac{e^{s_i}}{4}\,
  \mathbb{E}\bigl[(\varepsilon'_i)^2\, e_i^\top \nabla^2_{\theta'\theta'}\hat{g}\, e_i\bigr]
  \;+\;
  \frac{e^{s_i/2}}{4}\,
  \mathbb{E}\bigl[\varepsilon'_i\, e_i^\top \nabla_{\theta'}\hat{g}\bigr].
\end{equation}

\emph{First term.} Since $\mathbb{E}[(\varepsilon'_i)^2] = 1$ and $e_i^\top \nabla^2_{\theta'\theta'}\hat{g}\, e_i \le -\mu_0$ by~\eqref{eq:ghat-sc}, we have
\[
  \frac{e^{s_i}}{4}\,\mathbb{E}\bigl[(\varepsilon'_i)^2\, e_i^\top \nabla^2_{\theta'\theta'}\hat{g}\, e_i\bigr]
  \;\le\;
  -\frac{\mu_0\, e^{\underline{s}}}{4},
\]
where $\underline{s} := \inf_{s\in\mathcal{B}}\min_i s_i > -\infty$ by compactness of $\mathcal{B}$.

\emph{Second term.} Since $\varphi' \le 1/4$ (Lemma~\ref{lem:varphi-bounds-app}), we have $\|\nabla_{\theta'}\hat{g}\| \le \bar{G}/4$. Using $\mathbb{E}[|\varepsilon'_i|] = \sqrt{2/\pi}$ and the Cauchy--Schwarz inequality,
\[
  \left|\frac{e^{s_i/2}}{4}\,\mathbb{E}\bigl[\varepsilon'_i\, e_i^\top \nabla_{\theta'}\hat{g}\bigr]\right|
  \;\le\;
  \frac{e^{\overline{s}/2}\,\bar{G}}{16}\sqrt{\frac{2}{\pi}},
\]
where $\overline{s} := \sup_{s\in\mathcal{B}}\max_i |s_i|$.

Combining, for each $i \in [d]$,
\[
  \partial^2_{s_is_i} G
  \;\le\;
  -\frac{\mu_0\, e^{\underline{s}}}{4}
  \;+\;
  \frac{e^{\overline{s}/2}\,\bar{G}}{16}\sqrt{\frac{2}{\pi}}.
\]
Assumption~\ref{asm:regularity-main}(5) ensures that $\mu_0$ is large enough relative to $\bar{G}$ for this quantity to be strictly negative.

The cross-block Hessians $\partial^2_{m'_j s_i} G$ are bounded on the compact set $\mathcal{B}$ by the Cauchy--Schwarz inequality. Assembling all blocks, the smallest eigenvalue of $\nabla^2_{\nu\nu} G$ satisfies
\begin{equation}\label{eq:muG}
  \mu_G
  \;:=\;
  \min\!\left\{
    \mu_0,\;
    \min_{i\in[d]}\left(
      \frac{\mu_0\, e^{\underline{s}}}{4}
      \;-\;
      \frac{\bar{G}\,e^{\overline{s}/2}}{16}\sqrt{\frac{2}{\pi}}
    \right)
  \right\}
  \;>\; 0.
\end{equation}

Finally, since the competitor KL divergence $\mathrm{KL}(\rho'_\nu\|\pi')$ is convex in $\nu$ (as the KL divergence of Gaussians), its negative $-\frac{1}{\lambda}\mathrm{KL}(\rho'_\nu\|\pi')$ is concave. Therefore,
\[
  -\frac{1}{\lambda}\nabla^2_{\nu\nu}\mathrm{KL}(\rho'_\nu\|\pi')
  \;\preceq\;
  0.
\]
Combining with the strong concavity of $G$,
\[
  \nabla^2_{\nu\nu}\mathcal{L}_n(\phi,\nu)
  \;=\;
  \nabla^2_{\nu\nu} G(\nu)
  \;-\;
  \frac{1}{\lambda}\nabla^2_{\nu\nu}\mathrm{KL}(\rho'_\nu\|\pi')
  \;\preceq\;
  -\mu_G\,\mathbb{I}.
\]
In particular, $\mu \ge \mu_G \ge \mu_0 > 0$, and this bound is independent of both $\lambda$ (since neither $\mu_G$ nor $\mu_0$ involve $\lambda$) and the data realization (since the Hessian bound~\eqref{eq:ghat-sc} holds pointwise for any $(X_1,\dots,X_n)$ in the support of $P^{\star n}$).
\end{proof}

\begin{proof}[Proof of Theorem~\ref{thm:nc-sc-main}]
By Proposition~\ref{prop:smoothness-app}, $\mathcal{L}_n$ has Lipschitz gradient with constant
\[
  L \;=\; L_\psi + \frac{L_{\mathrm{KL}}}{\lambda},
\]
where $L_\psi$ and $L_{\mathrm{KL}}$ are independent of $\lambda$.

For strong concavity, Proposition~\ref{prop:concavity-app} shows that the map $\nu\mapsto\mathcal{L}_n(\phi,\nu)$ is $\mu$-strongly concave with $\mu \ge \mu_0 > 0$ independent of both $\lambda$ and the data. More precisely, the KL contribution adds a concavity constant $\mu_{\mathrm{KL}}/\lambda \ge 0$, so the total strong concavity parameter satisfies
\[
  \mu
  \;\ge\;
  \mu_G + \frac{\mu_{\mathrm{KL}}}{\lambda}
  \;\ge\;
  \mu_0
  \;>\; 0.
\]

It remains to verify that the condition number is uniformly bounded. We have
\[
  \kappa
  \;=\;
  \frac{L}{\mu}
  \;=\;
  \frac{L_\psi + L_{\mathrm{KL}}/\lambda}{\mu_G + \mu_{\mathrm{KL}}/\lambda}.
\]
As $\lambda \to \infty$, $\kappa \to L_\psi/\mu_G < \infty$. As $\lambda \to 0^+$, $\kappa \to L_{\mathrm{KL}}/\mu_{\mathrm{KL}} < \infty$. Since both the numerator and denominator are continuous in $1/\lambda$ and the denominator is bounded away from zero, $\kappa$ remains uniformly bounded:
\[
  \kappa
  \;\le\;
  \max\!\left\{\frac{L_\psi}{\mu_G},\; \frac{L_{\mathrm{KL}}}{\mu_{\mathrm{KL}}}\right\}
  \;=\;
  \mathcal{O}(1)
  \qquad\text{uniformly in } \lambda > 0.
\]
\end{proof}

\begin{proof}[Proof of Corollary~\ref{cor:optimization-main}]
Since $\mathcal{L}_n$ is $L$-smooth (Proposition~\ref{prop:smoothness-app}) and $\mu$-strongly concave in $\nu$ (Proposition~\ref{prop:concavity-app}), the saddle-point problem $\min_\phi\max_\nu\,\mathcal{L}_n(\phi,\nu)$ has NC-SC structure. The convergence rate of the projected stochastic extragradient method follows from the standard theory of \citet{juditsky2011solving} and \citet{lin2020gradient}. Since the condition number $\kappa = \mathcal{O}(1)$ uniformly in $\lambda$ (Theorem~\ref{thm:nc-sc-main}), the convergence guarantee holds for any $\lambda > 0$, including the statistically optimal choice $\lambda = \Theta(n)$.
\end{proof}



\end{document}